\documentclass{article}

     \PassOptionsToPackage{numbers, compress}{natbib}

     \usepackage[final]{neurips_2021}

\usepackage[utf8]{inputenc} 
\usepackage[T1]{fontenc}    
\usepackage{hyperref}       
\hypersetup{
	colorlinks   = true, 
	urlcolor     = black, 
	linkcolor    = blue, 
	citecolor   = blue 
}
\usepackage{url}            
\usepackage{booktabs}       
\usepackage{amsfonts}       
\usepackage{nicefrac}       
\usepackage{microtype}      
\usepackage{xcolor}         

\usepackage{amsthm}
\newtheorem{theorem}{Theorem}
\newtheorem{lemma}{Lemma}

\newtheorem{definition}{Definition}
\newtheorem{assumption}{Assumption}

\usepackage{algorithm}
\usepackage{algorithmic}
\usepackage{ifthen}
\usepackage{thm-restate}
\usepackage{tikz}
\usetikzlibrary{positioning,chains,fit,shapes,calc}
\usepackage{eqparbox,array}

\usepackage{amsmath}
\usepackage{amssymb}
\usepackage{amsfonts}
\usepackage{enumerate}
\usepackage{flushend}
\usepackage{mathrsfs}
\usepackage{subfigure}
\usepackage{booktabs}
\usepackage{makecell}
\usepackage{setspace}
\usepackage{xcolor}
\usepackage{wrapfig}
\usepackage{bm}
\usepackage{multicol}

\newcommand{\R}{\mathbb{R}}
\newcommand{\N}{\mathbb{N}}

\newcommand{\bv}{\boldsymbol{v}}
\newcommand{\bw}{\boldsymbol{w}}

\newcommand{\cA}{\mathcal{A}}

\newcommand{\cE}{\mathcal{E}}
\newcommand{\cF}{\mathcal{F}}
\newcommand{\cG}{\mathcal{G}}
\newcommand{\cH}{\mathcal{H}}
\newcommand{\cI}{\mathcal{I}}

\newcommand{\cM}{\mathcal{M}}
\newcommand{\cN}{\mathcal{N}}

\newcommand{\cP}{\mathcal{P}}

\newcommand{\cS}{\mathcal{S}}

\newcommand{\argmax}{\operatornamewithlimits{argmax}}
\newcommand{\argmin}{\operatornamewithlimits{argmin}}
\mathchardef\mhyphen="2D
\newcommand{\ex}{\mathbb{E}}

\newcommand{\sbr}[1]{\left( #1 \right)}
\newcommand{\mbr}[1]{\left[ #1 \right]}
\newcommand{\lbr}[1]{\left\{ #1 \right\}}

\newcommand{\opt}{\mathtt{OPT}}
\newcommand{\oracle}{\mathtt{MaxOracle}}
\newcommand{\existoracle}{\mathtt{ExistOracle}}

\newcommand{\aroracle}{\mathtt{AR\mbox{-}Oracle}}
\newcommand{\findewinset}{\mathtt{BottleneckSearch}}
\newcommand{\rad}{\textup{rad}}
\newcommand{\uniformfc}{\mathtt{UniformFC}}
\newcommand{\uniformfb}{\mathtt{UniformFB}}
\newcommand{\nonlucb}{\mathtt{GenLUCB}}

\newcommand{\algbottleneck}{\mathtt{BLUCB}}
\newcommand{\algbottleneckpac}{\mathtt{BLUCB\mbox{-}PAC}}
\newcommand{\algbottleneckexplore}{\mathtt{BLUCB\mbox{-}Explore}}
\newcommand{\algbottleneckverify}{\mathtt{BLUCB\mbox{-}Verify}}
\newcommand{\algbottleneckparallel}{\mathtt{BLUCB\mbox{-}Parallel}}
\newcommand{\algbottleneckparallelpac}{\mathtt{BLUCB\mbox{-}Parallel\mbox{-}PAC}}
\newcommand{\bsar}{\mathtt{BSAR}}
\newcommand{\wmin}{\mathtt{MinW}}
\newcommand{\emin}{\mathtt{MinE}}
\newcommand{\kl}{\textup{KL}}
\newcommand{\fc}{\textup{C}}
\newcommand{\fb}{\textup{B}}
\newcommand{\fg}{\textup{G}}
\newcommand{\superset}{\cS}
\newcommand{\sub}{\textup{sub}}
\newcommand{\exclude}{\textup{ex}}

\newcommand{\compilehidecomments}{true}
\ifthenelse{ \equal{\compilehidecomments}{true} }{%
	\newcommand{\wei}[1]{}
	\newcommand{\yuko}[1]{}
	\newcommand{\yihan}[1]{}
}{
	\newcommand{\wei}[1]{{\color{red}  [\text{Wei:} #1]}}
	\newcommand{\yuko}[1]{{\color{blue} [\text{Yuko:} #1]}}
	\newcommand{\yihan}[1]{{\color{teal} [\text{Yihan:} #1]}}
}

\newcommand{\compilefullversion}{true} 
\ifthenelse{\equal{\compilefullversion}{false}}{%
	\newcommand{\OnlyInFull}[1]{}
	\newcommand{\OnlyInShort}[1]{#1}
}{%
	\newcommand{\OnlyInFull}[1]{#1}%
	\newcommand{\OnlyInShort}[1]{}%
}%

\allowdisplaybreaks
\title{Combinatorial Pure Exploration with Bottleneck Reward Function}

\author{%
    Yihan Du\\
  IIIS, Tsinghua University\\
  Beijing, China\\
  \texttt{duyh18@mails.tsinghua.edu.cn} \\
   \And
   Yuko Kuroki \\
   The University of Tokyo / RIKEN \\
   Tokyo, Japan \\
   \texttt{yukok@is.s.u-tokyo.ac.jp} \\
   \AND
   Wei Chen \\
   Microsoft Research \\
   Beijing, China \\
   \texttt{weic@microsoft.com} \\
}

\begin{document}

\maketitle

\begin{abstract}
In this paper, we study the Combinatorial Pure Exploration problem with the Bottleneck reward function (CPE-B) under the fixed-confidence (FC) and fixed-budget (FB) settings.
In CPE-B, given a set of base arms and a collection of subsets of base arms (super arms) following a certain combinatorial constraint, a learner sequentially plays a base arm and observes its random reward, with the objective of finding the optimal super arm with the maximum bottleneck value, defined as the minimum expected reward of the base arms contained in the super arm.
CPE-B captures a variety of practical scenarios such as network routing in communication networks, and its \emph{unique challenges} fall on how to utilize the bottleneck property to save samples and achieve the statistical optimality. None of the existing CPE studies (most of them assume linear rewards) can be adapted to solve such challenges, and thus we develop brand-new techniques to handle them.
For the FC setting, we propose novel algorithms with optimal sample complexity for a broad family of instances and establish a matching lower bound to demonstrate the optimality (within a logarithmic factor).
For the FB setting, we design an algorithm which achieves the state-of-the-art error probability guarantee and is the first to run efficiently on fixed-budget path instances, compared to existing CPE algorithms. 
Our experimental results on the top-$k$, path and matching instances validate the empirical superiority of the proposed algorithms over their baselines.
\end{abstract}

\section{Introduction}
The Multi-Armed Bandit (MAB) problem~\cite{lai_robbins1985,thompson1933,UCB_auer2002,agrawal2012analysis} is a classic model to solve the exploration-exploitation trade-off in online decision making.
Pure exploration~\cite{Audibert2010,kalyanakrishnan2012,bubeck2013_SAR,APT2016} is an important variant of the MAB problem, which aims to identify the best arm under a given confidence or a given sample budget.
There are various works studying pure exploration, such as top-$k$ arm identification~\cite{Even2006,kalyanakrishnan2012,bubeck2013_SAR, Multiple-Arm_Identification_Kuroki2020}, top-$k$ arm under matriod constraints~\cite{matroid_ChenLJ2016} and multi-bandit best arm identification~\cite{Multi_Bandit,bubeck2013_SAR}. 

The Combinatorial Pure Exploration (CPE) framework, firstly proposed by Chen et al.~\cite{chen2014cpe}, encompasses a rich class of pure exploration problems~\cite{Audibert2010,kalyanakrishnan2012,matroid_ChenLJ2016}. 
In CPE, there are a set of base arms, each associated with an unknown reward distribution. A subset of base arms is called a super arm, which follows a certain combinatorial structure.
At each timestep, a learner plays a base arm and observes a random reward sampled from its distribution, with the objective to identify the optimal super arm with the maximum expected reward.
While Chen et al.~\cite{chen2014cpe} provide this general CPE framework, 
their algorithms and analytical techniques only work under the linear reward function and cannot be applied to other nonlinear reward cases.\footnote{The algorithmic designs and analytical tools (e.g., symmetric difference and exchange set) in \cite{chen2014cpe} all rely on the linear property and cannot be applied to nonlinear reward cases, e.g, the bottleneck reward problem.}

However, in many real-world scenarios, the expected reward function is not necessarily linear.
One of the common and important cases is the \emph{bottleneck reward} function, i.e., the expected reward of a super arm is the minimum expected reward of the base arms contained in it.
%
For example, in communication networks~\cite{communication_networks}, the transmission speed of a path is usually determined by the link with the lowest rate, and a learner samples the links in order to find the optimal transmission path which maximizes its bottleneck link rate. 
In traffic scheduling~\cite{urban_traffic}, a scheduling system collects the information of road segments in order to plan an efficient route which optimizes its most congested (bottleneck) road segment.
In neural architecture search~\cite{network_architecture_search}, the overall efficiency of a network architecture is usually constrained by its worst module, and an agent samples the available modules with the objective to identify the best network architecture in combinatorial search space.

In this paper, we study the Combinatorial Pure Exploration with the Bottleneck reward function (CPE-B) which aims to identify the optimal super arm with the maximum bottleneck value by querying the base arm rewards, where the bottleneck value of a super arm is defined as the minimum expected reward of its containing base arms.
We consider two popular settings in pure exploration, i.e, \emph{fixed-confidence (FC)}, where given confidence parameter $\delta$, the learner aims to identify the optimal super arm with probability $1-\delta$ and minimize the number of used samples (sample complexity), and \emph{fixed-budget (FB)}, where the learner needs to use a given sample budget to find the optimal super arm and minimize the error probability. 


\begin{wrapfigure}[12]{r}{0.4\textwidth}
	\centering    
	\vspace*{-2em}
	\includegraphics[width=0.39\textwidth]{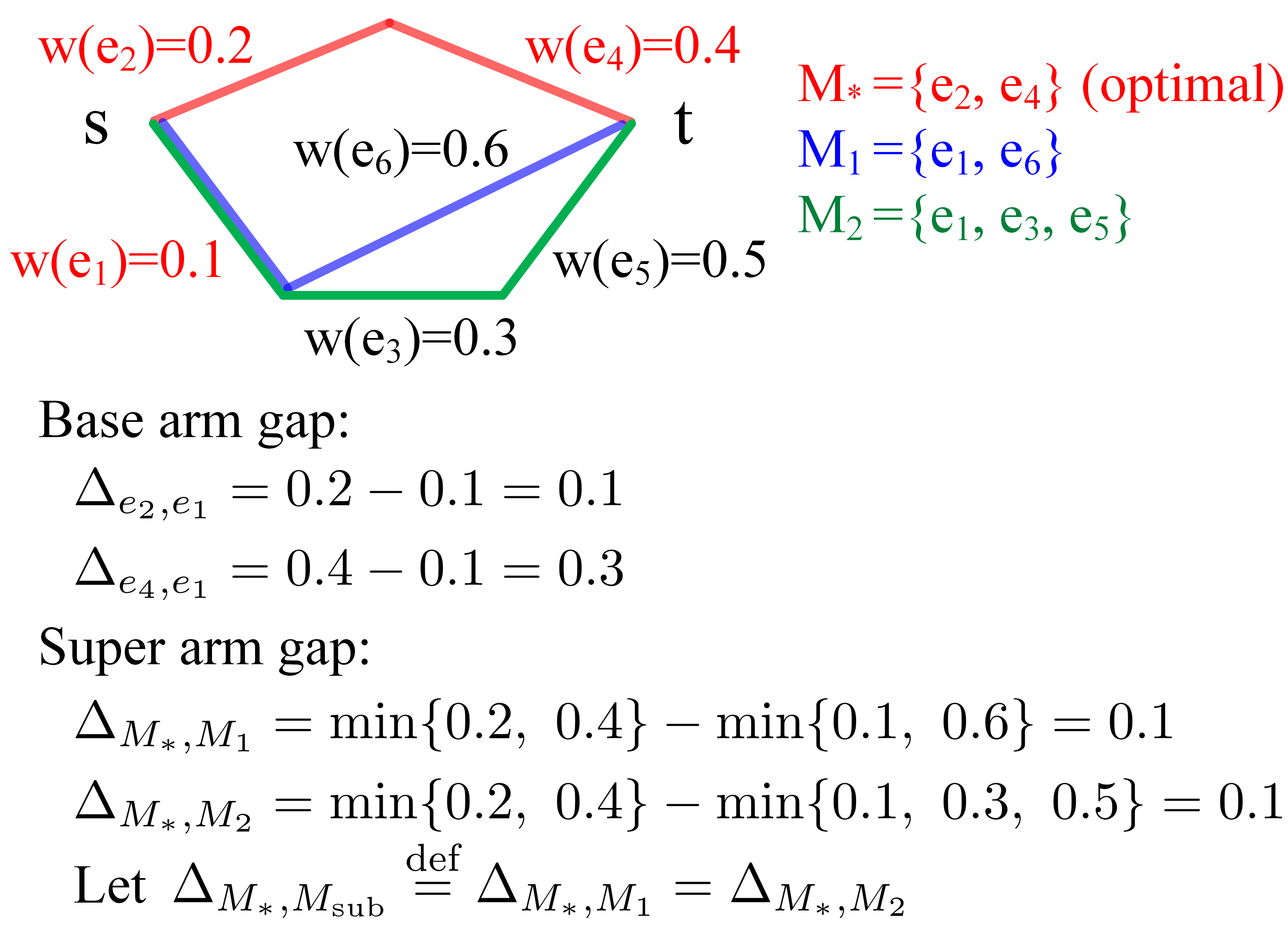}  
	\caption{Illustrating example.}     
	\label{fig:example}     
\end{wrapfigure}
\textbf{Challenges of CPE-B.}
Compared to prior CPE works~\cite{chen2014cpe,ChenLJ_Nearly_Optimal_Sampling17,huang_CPE_CS2018}, our CPE-B aims at utilizing the bottleneck property to save samples and achieve the statistical optimality.
It faces with two \emph{unique challenges}, i.e., how to (i) achieve the tight \emph{base-arm-gap dependent} sample complexity and (ii) avoid the dependence on \emph{unnecessary base arms} in the results, while running in polynomial time.
We use a simple example in Figure~\ref{fig:example} to illustrate our challenges. 
In Figure~\ref{fig:example}, there are six edges (base arms) and three $s$-$t$ paths (super arms), and the base arm reward $w(e_i)$, base arm gap $\Delta_{e_i, e_j}$ and super arm gap $\Delta_{M_*, M_{\sub}}$ are as shown in the figure. 
In order to identify the optimal path, all we need is to pull $e_1,e_2,e_4$ to determine that $e_1$ is worse than $e_2$ and $e_4$, and $e_3,e_5,e_6$ are useless for revealing the sub-optimality of $M_1$ and $M_2$. 
In this case, the optimal sample complexity should be $O( (\frac{2}{\Delta_{e_2,e_1}^2}+\frac{1}{\Delta_{e_4,e_1}^2}) \ln \delta^{-1})$, which depends on the tight base arm gaps and only includes the critical base arms ($e_1,e_2,e_4$). 
However, if one naively adapts existing CPE algorithms~\cite{chen2014cpe,CPE_DB_2020,CPE_BL_PL2020} to work with bottleneck reward function, an inferior sample complexity of $O(\sum_{e_i, i \in [6]}\frac{1}{\Delta_{M_*,M_{\sub}}^2} \ln \delta^{-1})$ is incurred, which depends on the loose super arm gaps and 
contains a summation over all base arms (including the unnecessary $e_3,e_5,e_6$).
Hence, our challenge falls on how to achieve such efficient sampling in an online environment, where we do not know which are critical base arms $e_1, e_2, e_4$ but want to gather just enough information to identify the optimal super arm. 
We remark that, none of existing CPE studies can be applied to solve the unique challenges of CPE-B, and thus we develop brand-new techniques to handle them and attain the optimal results (up to a logarithmic factor).

\textbf{Contributions.}
For CPE-B in the FC setting, (i) we first develop a novel algorithm $\algbottleneck$, which employs a bottleneck-adaptive sample strategy and achieves the tight base-arm-gap dependent sample complexity. (ii) We further propose an improved algorithm $\algbottleneckparallel$ in high confidence regime, which adopts
an efficient ``bottleneck-searching'' offline procedure and a novel ``check-near-bottleneck'' stopping condition. 
The sample complexity of $\algbottleneckparallel$ drops the dependence on unnecessary base arms and achieves the optimality (within a logarithmic factor) under small enough $\delta$. 
(iii) A matching sample complexity lower bound for the FC setting is also provided, which demonstrates the optimality of our algorithms.
For the FB setting, (iv) we propose a novel algorithm $\bsar$ with a special acceptance scheme for the bottleneck identification task. $\bsar$ achieves the state-of-the-art error probability and is the first to run efficiently on fixed-budget path instances, compared to existing CPE algorithms.
All our proposed algorithms run in \emph{polynomial time}.\footnote{Here ``polynomial time'' refers to polynomial time in the number of base arms $n$ (which is equal to the number of edges $E$ in our considered instances such as $s$-$t$ paths, matchings and spanning trees).\yihan{Added the clarification for ``polynomial time''.}}
The experimental results demonstrate that our algorithms significantly outperform the baselines.
\OnlyInFull{Due to space limit, we defer all the proofs to Appendix.}
\OnlyInShort{Due to space limit, we defer all the proofs to the supplementary material.}

\subsection{Related Work}
In the following we briefly review the related work in the CPE literature.
Chen et al.~\cite{chen2014cpe} firstly propose the CPE model and only consider the linear reward function
(CPE-L), and their results for CPE-L are further improved by~\cite{gabillon2016improved,ChenLJ_Nearly_Optimal_Sampling17}.
Huang et al.~\cite{huang_CPE_CS2018} investigate the continuous and separable reward functions (CPE-CS), but their algorithm only runs efficiently on simple cardinality constraint instances.
All these works consider directly sampling base arms and getting their feedback.
There are also several CPE studies which consider other forms of sampling and feedback.
Chen et al.~\cite{CPE_DB_2020} propose the CPE for dueling bandit setting, where at each timestep the learner pulls a duel between two base arms and observes their comparison outcome.
Kuroki et al.~\cite{Online_Dense_Subgraph_Kuroki2020} study an online densest subgraph problem, where the decision is a subgraph and the feedback is the reward sum of the edges in the chosen subgraph (i.e., full-bandit feedback). 
Du et al.~\cite{CPE_BL_PL2020} investigate CPE with the full-bandit or partial linear feedback. 
All of the above studies consider the pure exploration setting, while in combinatorial bandits there are other works~\cite{chen2013combinatorial_ICML,combes2015combinatorial,chen2016combinatorial_JMLR} studying the regret minimization setting (CMAB). In CMAB, the learner plays a super arm and observes the rewards from all base arms contained in it, with goal of minimizing the regret, which is significantly different from our setting.
Note that none of the above studies covers our CPE-B problem or can be adapted to solve the unique challenges of CPE-B, and thus CPE-B demands a new investigation.


\section{Problem Formulation} \label{sec:formulation}

In this section, we give the formal formulation of CPE-B.
%
In this problem, a learner is given $n$ base arms numbered by $1, 2, \dots, n$. Each base arm $e \in [n]$ is associated with an \emph{unknown} reward distribution with the mean of $w(e)$ and 
an $R$-sub-Gaussian tail,
which is a standard assumption in bandits~\cite{improved_linear_bandit2011,chen2014cpe,APT2016,Tao2018}. Let $\bw=(w(1), \dots, w(n))^\top$ be the expected reward vector of base arms. 
The learner is also given a decision class $\cM \subseteq 2^{[n]}$, which is a collection of super arms (subsets of base arms) and generated from a certain combinatorial structure, such as $s$-$t$ paths, maximum cardinality matchings, and spanning trees. 
For each super arm $M \in \cM$, we define its expected reward (also called \emph{bottleneck value}) as $\wmin(M, \bw)=\min_{e \in M} w(e)$,\footnote{In general, the second input of function $\wmin$ can be any vector: for any $M \in \cM$ and $\bv \in \R^n$, $\wmin(M, \bv)=\min_{e \in M} v(e)$.} i.e., the minimum expected reward of its constituent base arms, which is so called \emph{bottleneck reward function}.
Let $M_*=\argmax_{M \in \cM} \wmin(M, \bw)$ be the optimal super arm with the maximum  bottleneck value, and  $\opt=\wmin(M_*, \bw)$ be the optimal value.
Following the pure exploration literature~\cite{Even2006,chen2014cpe,ChenLJ_Nearly_Optimal_Sampling17,CPE_DB_2020}, we assume that $M_*$ is unique, and this assumption can be removed in our extension to the PAC learning setting \OnlyInFull{(see Section~\ref{apx:PAC}).}\OnlyInShort{(see the supplementary material).}

At each timestep, the learner plays (or samples) a base arm $p_t \in [n]$ and observes a random reward sampled from its reward distribution, where the sample is independent among different timestep $t$. 
The learner's objective is to identify the optimal super arm $M_*$ from $\cM$.

For this identification task, we study two common metrics in pure exploration~\cite{kalyanakrishnan2012,bubeck2013_SAR,APT2016,ChenLJ_Nearly_Optimal_Sampling17},
i.e., fixed-confidence (FC) and fixed-budget (FB) settings.
In the FC setting, given a confidence parameter $\delta \in (0,1)$, the learner needs to identify $M_*$ with probability at least $1-\delta$ and minimize the \emph{sample complexity}, i.e., the number of samples used.
%
%
In the FB setting, the learner
is given a fixed sample budget $T$, and needs to identify $M_*$ within $T$ samples and minimize the \emph{error probability}, i.e., the probability of returning a wrong answer.


\section{Algorithms for the Fixed-Confidence Setting} 
\label{sec:fixed_confidence}
In this section, we first propose a simple algorithm $\algbottleneck$ for the FC setting, which adopts a novel bottleneck-adaptive sample strategy to obtain the tight base-arm-gap dependent sample complexity. 
We further develop an improvement $\algbottleneckparallel$ in high confidence regime, 
whose sample complexity drops the dependence on unnecessary base arms for small enough $\delta$.
Both algorithms achieve the optimal sample complexity for a family of instances (within a logarithmic factor).
%

\subsection{Algorithm $\algbottleneck$ with Base-arm-gap Dependent Results} \label{sec:fc_algorithm}

\begin{algorithm}[t!]
	\caption{$\algbottleneck$, algorithm for CPE-B in the FC setting} \label{alg:bottleneck}
	\begin{multicols}{2}
		\begin{algorithmic}[1]
			\STATE \textbf{Input:} $\cM$, $\delta \in (0,1)$ and $\oracle$.
			\STATE Initialize: play each $e \in [n]$ once, and update empirical means $\hat{\bw}_{n+1}$ and $T_{n+1}$\\
			\FOR{$t=n+1, n+2, \dots$}
			\STATE $\rad_t(e) \leftarrow  \sqrt{2 \ln (\frac{4nt^3}{\delta})/ T_t(e)},\  \forall e \in [n]$\;
			\vspace*{-0.8em}
			\STATE $\underline{w}_t(e)\leftarrow \hat{w}_t(e)-\rad_t(e),\  \forall e \in [n]$\;
			\STATE $\bar{w}_t(e)\leftarrow \hat{w}_t(e)+\rad_t(e),\  \forall e \in [n]$\;
			\STATE $M_t \leftarrow \oracle(\cM, \underline{\bw}_t)$\;
			\label{line:blucb_first_oracle}
			\STATE $\tilde{M}_t \leftarrow \oracle(\cM \setminus \superset(M_t), \bar{\bw}_t)$\;
			\label{line:blucb_second_oracle}
			\IF{ $\wmin(M_t, \underline{\bw}_t) \geq \wmin(\tilde{M}_t, \bar{\bw}_t)$ } 
			\label{line:blucb_stop}
			\STATE \textbf{return} $M_t$\; 
			\ENDIF
			\STATE $c_t \leftarrow \argmin_{e \in M_t} \underline{w}_t(e)$\; \label{line:blucb_c_t}
			\STATE $d_t \leftarrow \argmin_{e \in \tilde{M}_t} \underline{w}_t(e)$\; 
			\STATE $p_t \leftarrow \argmax_{e \in \{c_t, d_t\}} \rad_t(e) $\; \label{line:blucb_p_t}
			\STATE Play $p_t$, and observe the reward\;
			\STATE Update empirical means $\hat{w}_{t+1}(p_t)$\;
			\STATE Update the number of samples $T_{t+1}(p_t)$\;
			\ENDFOR
		\end{algorithmic}
	\end{multicols}
	\vspace*{-1em}
\end{algorithm}

Algorithm~\ref{alg:bottleneck} illustrates the proposed algorithm $\algbottleneck$ for CPE-B in the FC setting.
Here $\superset(M_t)$ denotes the set of all supersets of super arm $M_t$ (Line~\ref{line:blucb_second_oracle}).
Since the bottleneck reward function is monotonically decreasing, for any $M' \in \superset(M_t)$, we have $\wmin(M',\bw) \leq \wmin(M_t,\bw)$.
Hence, to verify the optimality of $M_t$, we only need to compare $M_t$ against super arms in $\cM \setminus \superset(M_t)$, and this property will also be used in the later algorithm $\algbottleneckparallel$.
%

$\algbottleneck$ is allowed to access an efficient \emph{bottleneck maximization oracle} $\oracle(\cF, \bv)$, 
which returns an optimal super arm from $\cF$ with respect to $\bv$, i.e., $\oracle(\cF, \bv) \in \argmax_{M \in \cF} \wmin(M, \bv)$. For $\cF=\cM$ (Line~\ref{line:blucb_first_oracle}), such an efficient oracle exists for many decision classes, such as the bottleneck shortest path~\cite{bottleneck_shortest_path2006}, bottleneck bipartite matching~\cite{bottleneck_bipartite_matching1994} and minimum bottleneck spanning tree~\cite{bottleneck_spanning_tree1978} algorithms.
For $\cF=\cM \setminus \superset(M_t)$ (Line~\ref{line:blucb_second_oracle}), 
we can also efficiently find the best super arm (excluding the supersets of $M_t$) by repeatedly removing each base arm in $M_t$ and calling the basic maximization oracle, and then selecting the one with the maximum bottleneck value.



We describe the procedure of $\algbottleneck$ as follows: 
at each timestep $t$, we calculate the lower and upper confidence bounds of base arm rewards, denoted by $\underline{\bw}_t$ and $\bar{\bw}_t$, respectively. 
Then, we call $\oracle$ to find the super arm $M_t$ with the maximum pessimistic bottleneck value from $\cM$ using $\underline{\bw}_t$ (Line~\ref{line:blucb_first_oracle}), and the super arm $\tilde{M}_t$ with the maximum optimistic bottleneck value from $\cM \setminus \superset(M_t)$ using $\bar{\bw}_t$ (Line~\ref{line:blucb_second_oracle}). 
$M_t$ and $\tilde{M}_t$ are two critical super arms that determine when the algorithm should stop or not.
If the pessimistic bottleneck value of $M_t$ is higher than the optimistic bottleneck value of $\tilde{M}_t$ (Line~\ref{line:blucb_stop}), we can determine that $M_t$ has the higher bottleneck value than any other super arm with high confidence, and then the algorithm can stop and output $M_t$. 
Otherwise, we 
select two base arms $c_t$ and $d_t$ with the minimum lower reward confidence bounds in $M_t$ and $\tilde{M}_t$  respectively, and play the one with the larger confidence radius~(Lines~\ref{line:blucb_c_t}-\ref{line:blucb_p_t}).

\textbf{Bottleneck-adaptive sample strategy.}
The ``select-minimum'' sample strategy in Lines~\ref{line:blucb_c_t}-\ref{line:blucb_p_t} comes from an \emph{insight} for the bottleneck problem:
to determine that $M_t$ has a higher bottleneck value than $\tilde{M}_t$, it suffices to find a base arm from $\tilde{M}_t$ which is worse than any base arm (the bottleneck base arm) in $M_t$.
To achieve this, base arms $c_t$ and $d_t$, which have the most potential to be the bottlenecks of $M_t$ and $\tilde{M}_t$, are the most necessary ones to be sampled.
This bottleneck-adaptive sample strategy is crucial for $\algbottleneck$ to achieve the tight base-arm-gap dependent sample complexity.
In contrast, the sample strategy of prior CPE algorithms \cite{chen2014cpe,CPE_DB_2020,CPE_BL_PL2020} treats all  base arms in critical super arms ($M_t$ and $\tilde{M}_t$) equally and does a uniform choice. If one naively adapts those algorithms with the current reward function $\wmin(M, \bw)$, a loose super-arm-gap dependent sample complexity is incurred.

To formally state the sample complexity of $\algbottleneck$, we introduce some notation and gap definition.
Let $N = \{e \mid e \notin M_*, w(e) < \opt \}$ and $\tilde{N} = \{e \mid e \notin M_*, w(e) \geq \opt \}$, which stand for the necessary and \emph{unnecessary} base arms contained in the sub-optimal super arms, respectively. 
We define the reward gap for the FC setting as
\begin{definition}[Fixed-confidence Gap] \label{def:fc_gap}
	\begin{equation}
	\Delta^\fc_e \!=\!
	\left\{
	\begin{array}{lr} 
	w(e)-\max_{M \neq M_*} \wmin(M, \bw),  
	\  \textup{if } e \in M_* , &\textup{\qquad(a)}
	\\
	w(e)-\max_{M \in \cM: e\in M} \wmin(M, \bw),
	\  \textup{if } e \in \tilde{N}, &\textup{\qquad(b)}
	\\
	\opt-\max_{M \in \cM: e\in M} \wmin(M, \bw),
	\  \textup{if } e \in N. 
	\end{array}
	\right. \nonumber
	\end{equation}
\end{definition}



Now we present the sample complexity upper bound of $\algbottleneck$.

\begin{theorem}[Fixed-confidence Upper Bound] 
	\label{thm:bottleneck_baseline_ub}
	With probability at least $1-\delta$, algorithm $\algbottleneck$ (Algorithm~\ref{alg:bottleneck}) for CPE-B in the FC setting returns the optimal super arm with sample complexity 
	$$
	O \sbr{ \sum_{e \in [n]} \frac{ R^2 }{(\Delta^\fc_e)^2} \ln \sbr{  \sum_{e \in [n]} \frac{ R^2 n  }{(\Delta^\fc_e)^2 \delta} }  } .
	$$
\end{theorem}

\textbf{Base-arm-gap dependent sample complexity.}
Owing to the bottleneck-adaptive sample strategy, the reward gap $\Delta^\fc_e$ (Definition~\ref{def:fc_gap}(a)(b)) is just defined as the difference between some critical bottleneck value and $w(e)$ itself, instead of the bottleneck gap between two super arms, and thus
our result depends on the tight base-arm-level (instead of super-arm-level) gaps.
For example, in Figure~\ref{fig:example}, $\algbottleneck$ only spends $\tilde{O}( (\frac{2}{\Delta_{e_2,e_1}^2}+\sum_{i=3,4,5,6}\frac{1}{\Delta_{e_i,e_1}^2} ) \ln \delta^{-1})$ samples, while a naive adaptation of prior CPE algorithms~\cite{chen2014cpe,CPE_DB_2020,CPE_BL_PL2020} with the bottleneck reward function will cause a loose super-arm-gap dependent result $\tilde{O}(\sum_{e_i, i \in [6]}\frac{1}{\Delta_{M_*,M_{\sub}}^2} \ln \delta^{-1})$.
Regarding the optimality, Theorem~\ref{thm:bottleneck_baseline_ub} matches the lower bound (presented in Section~\ref{sec:lower_bound_fixed_confidence}) for some family of instances (up to a logarithmic factor). However, in general cases there still exists a gap on those needless base arms $\tilde{N}$ ($e_3,e_5,e_6$ in Figure~\ref{fig:example}), which are not contained in the lower bound. Next, we show how to bridge this gap.

\subsection{Remove Dependence on Unnecessary Base Arms under Small $\delta$} \label{sec:verify}

\begin{algorithm}[t]
	\caption{$\algbottleneckparallel$, an improved algorithm for the FC setting under small $\delta$ }
	\label{alg:bottleneck_parallel}
	\begin{algorithmic}[1]
		\STATE \textbf{Input:} $\delta \in (0, 0.01)$ and sub-algorithm $\algbottleneckverify$.
		\STATE For $k=0,1,\dots$, let $\algbottleneckverify_k$ be the sub-algorithm $\algbottleneckverify$ with $\delta_k=\frac{\delta}{2^{k+1}}$\;
		\FOR{$t=1,2,\dots$}
		\FOR{ \textup{each} $k = 0,1,\dots \ \textup{ such that }\  t \textup{ mod } 2^k =0$}
		\STATE Start or resume  $\algbottleneckverify_k$ with one sample, and then suspend $\algbottleneckverify_k$\; \label{line:parallel_multiple_simulation}
		\STATE \textbf{if} $\algbottleneckverify_k$ returns an answer $M_\textup{out}$, then \textbf{return} $M_\textup{out}$\;
		\ENDFOR
		\ENDFOR
	\end{algorithmic}
\end{algorithm}

\begin{algorithm}[t]
	\caption{$\algbottleneckverify$, sub-algorithm of $\algbottleneckparallel$} \label{alg:bottleneck_verify}
	\begin{multicols}{2}
		\begin{algorithmic}[1]
			\STATE \textbf{Input:} $\cM$, $\delta^{V} \!\in\! (0,0.01)$ and $\oracle$.
			\STATE $\kappa \leftarrow 0.01$\; 
			\STATE $\!\!\! \hat{M}_*,\! \hat{B}_{\sub} \!\! \leftarrow \!\! \algbottleneckexplore(\cM, \kappa, \oracle)$\;
			\label{line:blucb_verify_call_explore}
			\vspace*{-1em}
			\STATE Initialize: play each $e \in [n]$ once, and update empirical means $\hat{\bw}_{n+1}$ and $T_{n+1}$\;
			\FOR{$t=n+1, n+2, \dots$}
			\STATE $\rad_t(e) \! \leftarrow \! R  \sqrt{2 \ln (\frac{4nt^3}{\delta^{V} })/ T_t(e)}, \forall e \!\in\! [n]$\;
			\STATE $\underline{w}_t(e)\leftarrow \hat{w}_t(e)-\rad_t(e),\  \forall e \in [n]$\;
			\STATE $\bar{w}_t(e)\leftarrow \hat{w}_t(e)+\rad_t(e),\  \forall e \in [n]$\;
			\STATE $\tilde{M}_t=\oracle(\cM \setminus \superset(\hat{M}_*), \bar{\bw}_t)$\;
			\IF{$\wmin(\hat{M}_*, \underline{\bw}_t) \geq \wmin(\tilde{M}_t, \bar{\bw}_t)$}
			\label{line:bottleneck_verify_stop}
			\STATE \textbf{return} $\hat{M}_*$\; 
			\ENDIF
			\STATE $c_t \leftarrow \argmin_{e \in \hat{M}_*} \underline{w}_t(e)$\;
			\label{line:blucb_verify_sample_start}
			\STATE $F_t \leftarrow \{ e \in \hat{B}_{\sub}: \bar{w}_t(e) > \underline{w}_t(c_t) \}$\;
			\STATE $p_t \leftarrow \argmax_{e \in F_t \cup \{ c_t \}} \rad_t(e)$\;
			\label{line:blucb_verify_sample_end}
			\STATE Play $p_t$, and observe the reward\;
			\STATE Update empirical means $\hat{w}_{t+1}(p_t)$\;
			\STATE Update the number of samples $T_{t+1}(p_t)$\;
			\ENDFOR
		\end{algorithmic}
	\end{multicols}
	\vspace*{-1em}
\end{algorithm}

\begin{algorithm}[t]
	\caption{$\algbottleneckexplore$, sub-algorithm of $\algbottleneckverify$, the \emph{key algorithm}} \label{alg:bottleneck_explore}
	\begin{multicols}{2}
		\begin{algorithmic}[1]
			\STATE \textbf{Input:} $\cM$, $\kappa=0.01$ and $\oracle$.
			\STATE Initialize: play each $e \in [n]$ once, and update empirical means $\hat{\bw}_{n+1}$ and $T_{n+1}$\;
			\FOR{$t=n+1, n+2, \dots$}
			\STATE $\!\! \rad_t(e) \!\leftarrow\! R \sqrt{2 \ln (\frac{4nt^3}{ \kappa })/ T_t(e)},\  \forall e \!\in\! [n]$\;
			\STATE $\underline{w}_t(e)\leftarrow \hat{w}_t(e)-\rad_t(e),\  \forall e \in [n]$\;
			\STATE $\bar{w}_t(e)\leftarrow \hat{w}_t(e)+\rad_t(e),\  \forall e \in [n]$\;
			\STATE $M_t \leftarrow \oracle(\cM, \underline{\bw}_t)$\;
			\STATE $\hat{B}_{\sub,t} \!\!\leftarrow\!\! \findewinset(\cM, M_t, \underline{\bw}_t)$\; \label{line:blucb_explore_findewinset}
			\vspace*{-0.8em} 
			\IF{$\bar{w}_t(e) \!\leq \! \frac{1}{2} (\wmin(M_t,\underline{\bw}_t)+\underline{w}_t(e))$ for all $e \in \hat{B}_{\sub,t}$} \label{line:blucb_explore_stop_condition}
			\STATE \textbf{return} $M_t, \hat{B}_{\sub,t}$\; 
			\ENDIF
			\STATE $c_t \leftarrow \argmin_{e \in M_t} \underline{w}_t(e)$\;
			\STATE $\hat{B}'_{\sub,t} \leftarrow \{ e \in \hat{B}_{\sub,t}:$\\ $\quad \bar{w}_t(e) > \frac{1}{2} (\wmin(M_t, \underline{\bw}_t)+\underline{w}_t(e)) \}$\; 
			\STATE $p_t \leftarrow \argmax_{e \in \hat{B}'_{\sub,t} \cup \{ c_t \}} \rad_t(e)$\;
			\STATE Play $p_t$, and observe the reward\;
			\STATE Update empirical means $\hat{w}_{t+1}(p_t)$\;
			\STATE Update the number of samples $T_{t+1}(p_t)$\;
			\ENDFOR
		\end{algorithmic}
	\end{multicols}
	\vspace*{-1em}
\end{algorithm}

\textbf{Challenges of avoiding unnecessary base arms.}
Under the bottleneck reward function, in each sub-optimal super arm $M_{\textup{sub}}$, only the base arms with rewards lower than $\opt$ (base arms in $N$) can determine the relationship of bottleneck values between $M_*$ and $M_{\textup{sub}}$ (the bottleneck of $M_{\textup{sub}}$ is the most efficient choice to do this), and the others (base arms in $\tilde{N}$) are useless for revealing the sub-optimality of $M_{\textup{sub}}$. 
Hence, to determine $M_*$, all we need is to sample the base arms in $M_*$ and the \emph{bottlenecks from all sub-optimal super arms}, denoted by $B_{\sub}$, to see that each sub-optimal super arm contains at least one base arm that is worse than anyone in $M_*$.
However, before sampling, (i) we do not know which is $M_*$ that should be taken as the comparison benchmark, and in each $M_{\sub}$, which base arm is its bottleneck (included in $B_{\sub}$). Also, (ii) under combinatorial setting, how to efficiently collect $B_{\sub}$ from all sub-optimal super arms is another challenge. 


To handle these challenges, we propose algorithm $\algbottleneckparallel$ based on the explore-verify-parallel framework~\cite{verification_karnin2016,ChenLJ_Nearly_Optimal_Sampling17}.
%
%
$\algbottleneckparallel$ (Algorithm~\ref{alg:bottleneck_parallel}) simultaneously simulates multiple $\algbottleneckverify_k$ (Algorithm~\ref{alg:bottleneck_verify}) with confidence $\delta^V_k=\delta/2^{k+1}$ for $k \in \N$.  
$\algbottleneckverify_k$ first calls $\algbottleneckexplore$ (Algorithm~\ref{alg:bottleneck_explore}) to guess an optimal super arm $\hat{M}_*$ and collect a \emph{near bottleneck set} $\hat{B}_{\sub}$ with \emph{constant confidence} $\kappa$, and then uses the required confidence $\delta^V_k$ to verify the correctness of $\hat{M}_*$ by only sampling base arms in $\hat{M}_*$ and $\hat{B}_{\sub}$. 
Through parallel simulations, $\algbottleneckparallel$ guarantees the $1-\delta$ correctness.

The \emph{key component} of this framework is $\algbottleneckexplore$ (Algorithm~\ref{alg:bottleneck_explore}), which provides a hypothesized answer $\hat{M}_*$ and critical base arms $\hat{B}_{\sub}$ for verification to accelerate its identification process.
Below we first describe the procedure of $\algbottleneckexplore$, and then explicate its two \emph{innovative techniques}, i.e.  offline subroutine and stopping condition, developed to handle the challenges~(i),(ii).
$\algbottleneckexplore$ employs the subroutine $\findewinset(\cM, M_{\exclude}, \bv)$ to return the set of bottleneck base arms from all super arms in $\cM \setminus \superset(M_{\exclude})$ with respect to  weight vector $\bv$.
At each timestep, we first calculate the best super arm $M_t$ under lower reward confidence bound $\underline{\bw}_t$, and call $\findewinset$ to collect the bottlenecks $\hat{B}_{\sub,t}$ from all super arms in $\cM \setminus \superset(M_t)$ with respect to $\underline{\bw}_t$ (Line~\ref{line:blucb_explore_findewinset}). Then, we use a stopping condition (Line~\ref{line:blucb_explore_stop_condition}) to examine if $M_t$ is correct and $\hat{B}_{\sub,t}$ is close enough to $\hat{B}_{\sub}$ (with confidence $\kappa$). If so, $M_t$ and $\hat{B}_{\sub,t}$ are eligible for verification and returned; otherwise, we play a base arm from $M_t$ and $\hat{B}_{\sub,t}$, which is most necessary for achieving the stopping condition.
In the following, we explicate the two innovative techniques in $\algbottleneckexplore$.

\textbf{Efficient ``bottleneck-searching'' offline subroutine.}
$\findewinset(\cM, M_{\exclude}, \bv)$ (Line~\ref{line:blucb_explore_findewinset}) serves as an efficient offline procedure to collect bottlenecks from all super arms in given decision class $\cM \setminus \superset(M_{\exclude})$ with respect to $\bv$.
To achieve efficiency, the main idea behind $\findewinset$ is to avoid enumerating super arms in the combinatorial space, but only enumerate base arms $e \in [n]$ to check if $e$ is the bottleneck of some super arm in $\cM \setminus \superset(M_{\exclude})$. We achieve this by removing all base arms with rewards lower than $v(e)$ and examining whether there exists a feasible super arm $M$ that contains $e$ in the remaining decision class. If so, $e$ is the bottleneck of $M$ and added to the output (more procedures are designed to exclude $\superset(M_{\exclude})$).
This efficient offline subroutine solves challenge (ii) on computation complexity (\OnlyInFull{see Section~\ref{apx:verification_algorithm_details} for its pseudo-codes and details}\OnlyInShort{see the supplementary material for its pseudo-codes and details}).

\textbf{Delicate ``check-near-bottleneck'' stopping condition.}
The stopping condition (Line~\ref{line:blucb_explore_stop_condition}) aims to ensure the returned $\hat{B}_{\sub,t}=\hat{B}_{\sub}$ to satisfy the following Property~(1): for each sub-optimal super arm $M_{\sub}$, some base arm $e$ such that $w(e) \leq \frac{1}{2} (\wmin(M_*,\bw) + \wmin(M_{\sub},\bw))$ is included in $\hat{B}_{\sub}$, 
which implies that $e$ is near to the actual bottleneck of $M_{\sub}$ within $\frac{1}{2} \Delta_{M_*,M_{\sub}}$, and cannot be anyone in $\tilde{N}$.
Property~(1) is crucial for $\algbottleneckverify$ to achieve the optimal sample complexity, since it guarantees that in verification using $\hat{B}_{\sub}$ to verify $M_*$ just costs the same order of samples as using $B_{\sub}$, which matches the lower bound.
%
In the following, we explain why this stopping condition can guarantee Property (1).

If the stopping condition (Line~\ref{line:blucb_explore_stop_condition}) holds, i.e., $ \forall e \in \hat{B}_{\sub,t}, \bar{w}_t(e) \leq \frac{1}{2} (\wmin(M_t, \underline{\bw}_t)+\underline{w}_t(e))$, using the definition of $\findewinset$, we have that for any $M' \in \cM \setminus \superset(M_t)$, its bottleneck $e'$ with respect to $\underline{\bw}_t$ is included in $\hat{B}_{\sub,t}$ and satisfies that
\begin{align*}
w(e') \leq \bar{w}_t(e')  
\overset{\textup{(a)}}{\leq} \frac{1}{2} (\wmin(M_t, \underline{\bw}_t)+\wmin(M', \underline{\bw}_t)) 
\leq \frac{1}{2} (\wmin(M_t, \bw)+\wmin(M', \bw)) , \label{eq:blucb_explore_stop_condition}
\end{align*}
where inequality (a) comes from $\bar{w}_t(e') \leq \frac{1}{2} (\wmin(M_t, \underline{\bw}_t)+\underline{w}_t(e'))$ and $\underline{w}_t(e'))=\wmin(M', \underline{\bw}_t)$.
Hence, we can defer that $\wmin(M', \bw) \leq w(e') \leq \frac{1}{2} (\wmin(M_t, \bw)+\wmin(M', \bw))$ for any $M' \in \cM \setminus \superset(M_t)$, and thus $M_t=M_*$ (with confidence $\kappa$). In addition, the returned $\hat{B}_{\sub,t}$ satisfies Property~(1).
This stopping condition offers knowledge of a hypothesized optimal super arm $\hat{M}_*$ and a near bottleneck set $\hat{B}_{\sub}$ for verification, which solves the challenge (i) and enables the overall sample complexity to achieve the  optimality for small enough $\delta$. Note that these two techniques are new in the literature, which are specially designed for handling the unique challenges of CPE-B.
We formally state the sample complexity of  $\algbottleneckparallel$ in  Theorem~\ref{thm:bottleneck_verify_ub}.

\begin{theorem}[Improved Fixed-confidence Upper Bound]
	\label{thm:bottleneck_verify_ub}
	For any $\delta < 0.01$, with probability at least $1-\delta$, algorithm $\algbottleneckparallel$ (Algorithm~\ref{alg:bottleneck_parallel}) for CPE-B in the FC setting returns $M_*$ and takes the expected sample complexity 
	\begin{align*}
		O \Bigg(   \sum_{e \in M_* \cup N} \frac{ R^2 }{(\Delta^\fc_e)^2} \ln \Bigg( \frac{1}{\delta} \sum_{e \in M_* \cup N} \frac{ R^2 n  }{(\Delta^\fc_e)^2 } \Bigg)  + \sum_{e \in \tilde{N}} \frac{ R^2 }{(\Delta^\fc_e)^2} \ln \Bigg(  \sum_{e \in \tilde{N}} \frac{ R^2 n  }{(\Delta^\fc_e)^2 } \Bigg)  \Bigg) .
	\end{align*}
\end{theorem}

\vspace*{-0.5em}
\textbf{Results without dependence on $\tilde{N}$ in the dominant term.}
Let $H_V=\sum_{e \in M_* \cup N} \frac{ R^2 }{(\Delta^\fc_e)^2}$ and $H_E=\sum_{e \in [n]} \frac{ R^2 }{(\Delta^\fc_e)^2}$ denote the verification and exploration hardness, respectively. 
Compared to $\algbottleneck$ (Theorem~\ref{thm:bottleneck_baseline_ub}), the sample complexity of $\algbottleneckparallel$ removes the redundant dependence on $\tilde{N}$ in the $\ln \delta^{-1}$ term, which guarantees better performance when $\ln \delta^{-1} \geq \frac{H_E}{H_E-H_V}$, i.e., $\delta \leq \exp(-\frac{H_E}{H_E-H_V})$.
This sample complexity matches the lower bound (within a logarithmic factor) under small enough $\delta$.
For the example in Figure~\ref{fig:example}, $\algbottleneckparallel$ only requires $\tilde{O}( (\frac{2}{\Delta_{e_2,e_1}^2}+\frac{1}{\Delta_{e_4,e_1}^2} ) \ln \delta^{-1})$ samples, which are just enough efforts (optimal) for identifying $M_*$. 

The condition $\delta<0.01$ in Theorem 2 is due to that the used explore-verify-parallel framework~\cite{verification_karnin2016,ChenLJ_Nearly_Optimal_Sampling17} needs a small $\delta$ to guarantee that $\algbottleneckparallel$ can maintain the same order of sample complexity as its sub-algorithm $\algbottleneckverify_k$. Prior pure exploration works~\cite{verification_karnin2016,ChenLJ_Nearly_Optimal_Sampling17} also have such condition on $\delta$.
\yihan{Added the explanation for condition $\delta$.}

\textbf{Time Complexity.}
All our algorithms can run in polynomial time, and the running time mainly depends on the offline oracles.
For example, on $s$-$t$ path instances with $E$ edges and $V$ vertices, the used offline procedures $\oracle$ and $\findewinset$ only spend $O(E)$ and $O(E^2(E+V))$ time, respectively. \OnlyInFull{See Section~\ref{apx:time_complexity} for more time complexity analysis.}\OnlyInShort{See the supplementary material for more time complexity analysis.}

\vspace*{-0.5em}
\section{Lower Bound for the Fixed-Confidence Setting}
\label{sec:lower_bound_fixed_confidence}
\vspace*{-0.5em}

In this section, we establish a matching sample complexity lower bound for CPE-B in the FC setting. 
To formally state our results, we first define the notion of \emph{$\delta$-correct algorithm} as follows.
For any confidence parameter $\delta \in (0,1)$, we call an algorithm $\cA$ a $\delta$-correct algorithm if for the fixed-confidence CPE-B problem, $\cA$ returns the optimal super arm with probability at least $1-\delta$.

\begin{theorem}[Fixed-confidence Lower Bound]
	\label{thm:fixed_confidence_lb}
	There exists a family of instances for the fixed-confidence CPE-B problem, for which given any $\delta \in (0,0.1)$, any $\delta$-correct algorithm has the expected sample complexity 
	$$
	\Omega \Bigg(\sum_{e \in M_* \cup N} \frac{ R^2 }{  (\Delta^\fc_e)^2} \ln \left(\frac{1}{\delta}\right)  \Bigg).
	$$
\end{theorem}


This lower bound demonstrates that the sample complexity of $\algbottleneckparallel$ (Theorem~\ref{thm:bottleneck_verify_ub}) is optimal (within a logarithmic factor) under small enough $\delta$, since its $\ln \delta^{-1}$ (dominant)  term does not depend on unnecessary base arms $\tilde{N}$ either.
In addition, if we impose some constraint on the constructed instances, the sample complexity of $\algbottleneck$ (Theorem~\ref{thm:bottleneck_baseline_ub}) can also match the lower bound up to a logarithmic factor 
\OnlyInFull{(see Appendix~\ref{apx:lower_bound} for details).}\OnlyInShort{(see the supplementary material for details).}
The condition $\delta<0.1$ comes from the lower bound analysis, which ensures that the binary entropy of finding a correct or wrong answer can be lower bounded by $\ln \delta^{-1}$. Existing pure exploration works~\cite{chen2014cpe,ChenLJ_Nearly_Optimal_Sampling17} also have such condition on $\delta$ in their lower bounds. 
\yihan{Added the explanation for condition on $\delta$ in Theorem 3.}
%

Notice that, both our lower and upper bounds depend on the tight base-arm-level (instead of super-arm-level) gaps, and capture the \emph{bottleneck insight}: different base arms in one super arm play distinct roles in determining its (sub)-optimality  and impose different influences on the problem hardness.


\vspace*{-0.5em}
\section{Algorithm for the Fixed-Budget Setting}
\label{sec:fixed_budget}
\vspace*{-0.5em}

\begin{algorithm}[t!]
	\caption{$\bsar$, algorithm for CPE-B in the FB setting} \label{alg:sar_bottleneck}
	\begin{multicols}{2}
		\begin{algorithmic}[1]
			\STATE \textbf{Input:} budget $T$, $\cM$, and $\mathtt{AR-Oracle}$.
			\STATE $\tilde{\log}(n) \!\leftarrow\!\! \sum_{i=1}^{n} \! \frac{1}{i}$. $\tilde{T}_0 \!\leftarrow\! 0$. $A_1, R_1 \!\leftarrow\! \varnothing$.\;
			\FOR{$t=1, \dots, n$}
			\STATE $\tilde{T}_t \leftarrow \left \lceil \frac{T-n}{\tilde{\log}(n)(n-t+1)} \right \rceil$\;
			\STATE $U_t \leftarrow [n]\setminus(A_t \cup R_t)$
			\STATE Play each $e \in U_t$ for $\tilde{T}_t-\tilde{T}_{t-1}$ times\; \label{line:bsar_play_U_t}
			\STATE Update empirical mean $\hat{w}_t(e)$, $\forall e \!\in\! U_t$
			\STATE $\hat{w}_t(e) \leftarrow \infty$ for all $e \in A_t$\; \label{line:bsar_infinity}
			\STATE $M_t \leftarrow \aroracle(\perp, R_t, \hat{\bw}_t)$\; \label{line:bsar_find_M_t}
			\FOR{ \textup{each} $e \in U_t$}
			\IF{$e \in M_t$} \label{line:bsar_e_in_M_t_if}
			\STATE \hspace*{-0.1em}$\tilde{M}_{t,e} \!\leftarrow\! \aroracle(\perp, R_t\!\cup\!\{e\}, \hat{\bw}_t)$ 
			\label{line:bsar_e_in_M_t_compute}
			\label{line:bsar_first_aroracle}
			\ELSE
			\STATE $\tilde{M}_{t,e} \leftarrow \aroracle(e, R_t, \hat{\bw}_t)$
			\label{line:bsar_second_aroracle}
			\ENDIF 
			\STATE // $\aroracle$ returns $\perp$ if the calculated feasible set is empty
			\ENDFOR 
			\STATE \hspace*{-0.5em}$ p_t \!\! \leftarrow \!\! \argmax \limits_{e \in U_t}  \wmin(\!M_t,\! \hat{\bw}_t\!) \!-\! \wmin(\!\tilde{M}_{t,e},\! \hat{\bw}_t\!)$ 
			\label{line:bsar_p_t}
			\STATE // {$\wmin(\perp, \hat{\bw}_t)=-\infty$}
			\IF{$p_t \in M_t$}
			\label{bsar:e_in_M_t}
			\label{line:bsar_acceptarm_if}
			\STATE $A_{t+1} \leftarrow A_t \cup \{p_t\}, R_{t+1} \leftarrow R_t$ \label{line:bsar_acceptarm}  \;
			\ELSE
			\STATE $A_{t+1} \leftarrow A_t, R_{t+1} \leftarrow R_t \cup \{p_t\}$\;
			\ENDIF
			\ENDFOR
			\STATE \textbf{return} $A_{n+1}$\;
		\end{algorithmic}
	\end{multicols}
	\vspace*{-1em}
\end{algorithm}

For CPE-B in the FB setting, we design a novel algorithm $\bsar$ that adopts a special acceptance scheme for bottleneck identification.
%
We allow $\bsar$ to access an efficient accept-reject oracle $\aroracle$, which takes an accepted base arm $e$ or $\perp$, a rejected base arm set $R$ and a weight vector $\bv$ as inputs, and returns an optimal super arm from the decision class 
$\cM(e, R)=\{ M \in \cM: e \in M, R \cap M = \varnothing \}$
with respect to $\bv$, i.e., $\aroracle \in \argmax_{M \in \cM(e, R)} \wmin(M, \bw)$.
If $\cM(e, R)$ is empty, $\aroracle$ simply returns $\perp$. 
Such an efficient oracle exists for many decision classes, e.g., paths, matchings and spanning trees
\OnlyInFull{(see Appendix~\ref{apx:fixed_budget} for implementation details).}
\OnlyInShort{(see the supplementary material for implementation details).}

$\bsar$ allocates the sample budget $T$ to $n$ phases adaptively, and maintains the accepted set $A_t$, rejected set $R_t$ and undetermined set $U_t$. 
In each phase, we only sample base arms in $U_t$ and set the empirical rewards of base arms in $A_t$ to infinity (Line~\ref{line:bsar_infinity}). Then, we call $\aroracle$ to compute the empirical best super arm $M_t$.
For each $e \in U_t$, we forbid $R_t$ and constrain $e$ inside/outside the calculated super arms and find the empirical best super arm $\tilde{M}_{t,e}$ from the restricted decision class (Lines~\ref{line:bsar_first_aroracle},\ref{line:bsar_second_aroracle}).
Then, we accept or reject the base arm $p_t$ that maximizes the empirical reward gap between $M_t$ and $\tilde{M}_{t,e}$, i.e., the one that is most likely to be in or out of $M_*$ (Line~\ref{line:bsar_p_t}).

\textbf{Special acceptance scheme for bottleneck and polynomial running time.}
The acceptance scheme $\hat{w}_t(e) \leftarrow \infty$ for all $e \in A_t$ (Line~\ref{line:bsar_infinity}) is critical to the correctness and computation efficiency of $\bsar$.
Since $A_t$ and $R_t$ are not pulled in phase $t$ and their estimated rewards are not accurate enough, we need to avoid them to disturb the following calculation of empirical bottleneck values  (Lines~\ref{line:bsar_find_M_t}-\ref{line:bsar_p_t}).
By setting the empirical rewards of $A_t$ to infinity, the estimation of bottleneck values for sub-optimal super arms $M_{\sub}$ avoids the disturbance of $A_t$, because each $M_{\sub}$ has at least one base arm with reward lower than $\opt$ and this base arm will never be included in $A_t$ (conditioned on high probability events). 
As for $M_*$, its empirical bottleneck value can be raised, but this only enlarges the empirical gap between $M_*$ and $M_{\sub}$ and does not affect the correctness of the choice $p_t$ (Line~\ref{line:bsar_p_t}). Hence, this acceptance scheme guarantees the correctness of $\bsar$ in bottleneck identification task.


Compared to existing CPE-L algorithm $\mathtt{CSAR}$~\cite{chen2014cpe}, they force the whole set $A_t$ inside the calculated super arms in the oracle, i.e., replacing Lines~\ref{line:bsar_first_aroracle},\ref{line:bsar_second_aroracle} with $\aroracle( A_t, R_t \cup \{e\}, \hat{\bw}_t)$ and $\aroracle(A_t \cup \{e\}, R_t, \hat{\bw}_t)$, and deleting Line~\ref{line:bsar_infinity}.
Such acceptance strategy incurs \emph{exponential-time} complexity on $s$-$t$ path instances,\footnote{Finding a $s$-$t$ path which contains a given edge set is NP-hard. \OnlyInFull{See Appendix~\ref{apx:np_hard} for its proof.}\OnlyInShort{See the supplementary material for its proof.}} and \emph{only works} for the linear reward function, where the common part $A_t$ between two compared super arms can be canceled out.
If one naively applies their acceptance strategy to our bottleneck problem, the common part $A_t$ is possible to drag down (dominate) the empirical bottleneck values of all calculated super arms  (Lines~\ref{line:bsar_find_M_t},\ref{line:bsar_first_aroracle},\ref{line:bsar_second_aroracle}) and their empirical gaps will become all zeros (Line~\ref{line:bsar_p_t}), which destroys the correctness of the choice $p_t$ in theoretical analysis.

$\bsar$ is the first to run in \emph{polynomial time} on fixed-budget $s$-$t$ path instances among existing CPE algorithms, owing to its skillful acceptance scheme and the simplified $\aroracle$ (only work with one accepted base arm instead of $A_t$).
Specifically, for $E$ edges and $V$ vertices, the time complexity of $\aroracle$ is $O(E(E+V))$ and $\bsar$ only spends $O(E^2(E+V))$ time in decision making. 

Now we give the definitions of fixed-budget reward gap and problem hardness, and then formally state the error probability result of $\bsar$.
%
%
For $e \in M_*$, $\Delta^\fb_e=\opt -  \max_{M \in \cM: e \notin M} \wmin(M, \bw)$, and for $e \notin M_*$, $\Delta^\fb_e=\opt-\max_{M \in \cM: e\in M} \wmin(M, \bw)$.
Let $\Delta^\fb_{(1)}, \dots, \Delta^\fb_{(n)}$ be the permutation of $\Delta^\fb_{1}, \dots, \Delta^\fb_{n}$ such that $\Delta^\fb_{(1)} \leq \dots \leq \Delta^\fb_{(n)}$, and the fixed-budget problem hardness is defined as $H^{\fb} = \max_{i \in [n]} \frac{i}{(\Delta^\fb_{(i)})^2}$. Let $\tilde{\log}(n) = \sum_{i=1}^{n} \frac{1}{i}$.
%
%
\begin{theorem}[Fixed-budget Upper Bound] \label{thm:bsar}
	For any $T>n$, algorithm $\bsar$ (Algorithm~\ref{alg:sar_bottleneck}) for CPE-B in the FB setting uses at most $T$ samples and returns the optimal super arm with the error probability bounded by
	$$
	O\sbr{ n^2 \exp \sbr{ - \frac{T-n  }{ \tilde{\log}(n)  R^2 H^B } } }.
	$$
\end{theorem}

Compared to the uniform sampling algorithm, which plays all base arms equally and has $O( n \exp ( - \frac{ T }{ R^2  n \Delta^{-2}_{\min} } ) )$ error probability with $\Delta_{\min}=\opt-\max_{M \neq M_*} \wmin(M, \bw)$, Theorem~\ref{thm:bsar} achieves a significantly better correctness guarantee (when $\Delta^B_e>\Delta_{\min}$ for most $e \in [n]$).
In addition, when our CPE-B problem reduces to conventional $K$-armed pure exploration problem~\cite{bubeck2013_SAR}, Theorem~\ref{thm:bsar} matches existing state-of-the-art result in \cite{bubeck2013_SAR}.
To our best knowledge, the lower bound for the fixed-budget setting in the CPE literature~\cite{chen2014cpe,huang_CPE_CS2018,Online_Dense_Subgraph_Kuroki2020,CPE_BL_PL2020} remains open. 

Our error probability analysis falls on taking advantage of the bottleneck property to handle the disturbance from the accepted arm set (which are not pulled sufficiently) and guaranteeing the estimation accuracy of bottleneck rewards. The differences between our analysis and prior analysis for $\mathtt{CSAR}$~\cite{chen2014cpe} are highlighted as follows:
(i) Prior analysis~\cite{chen2014cpe} relies on the linear property to cancel out the common part between two super arms when calculating their reward gap, in order to avoid the disturbance of accepted arms. In contrast, to achieve this goal, we utilize the special acceptance scheme of $\bsar$ to exclude all accepted arms in the calculation of bottleneck rewards, which effectively addresses the perturbation of inaccurate estimation on accepted arms.
(ii) Prior analysis~\cite{chen2014cpe} mainly uses the ``exchange sets'' technique, which only works for the linear reward function and leads to the dependence on the parameter of decision class structures. Instead, our analysis exploits the bottleneck property to establish confidence intervals in the base arm level, and effectively avoids the dependence on the parameter of decision class structures.
\yihan{Added a paragraph to highlight the difference between our error probability analysis and prior analysis.}

\vspace*{-0.5em}
\section{Experiments}
\label{sec:experiments}
\vspace*{-0.5em}

\begin{figure}[t] 
	\centering    
	\subfigure[FC, $s$-$t$ path, large $\delta$] { \label{fig:fc_path_large_delta}
		\includegraphics[width=0.25\columnwidth]{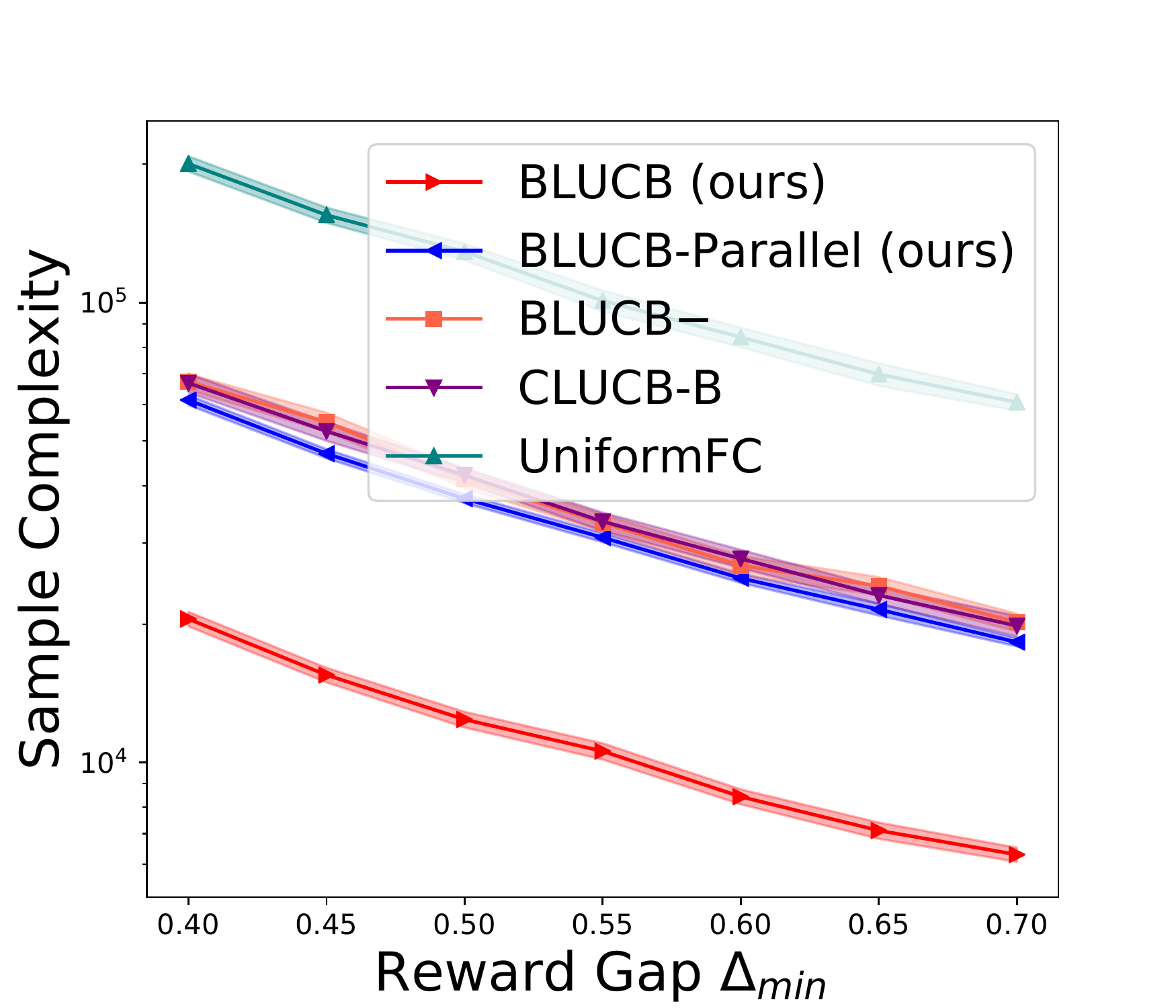} 
	}     
	\subfigure[FC, $s$-$t$ path, small $\delta$] { \label{fig:fc_path_small_delta}
		\includegraphics[width=0.25\columnwidth]{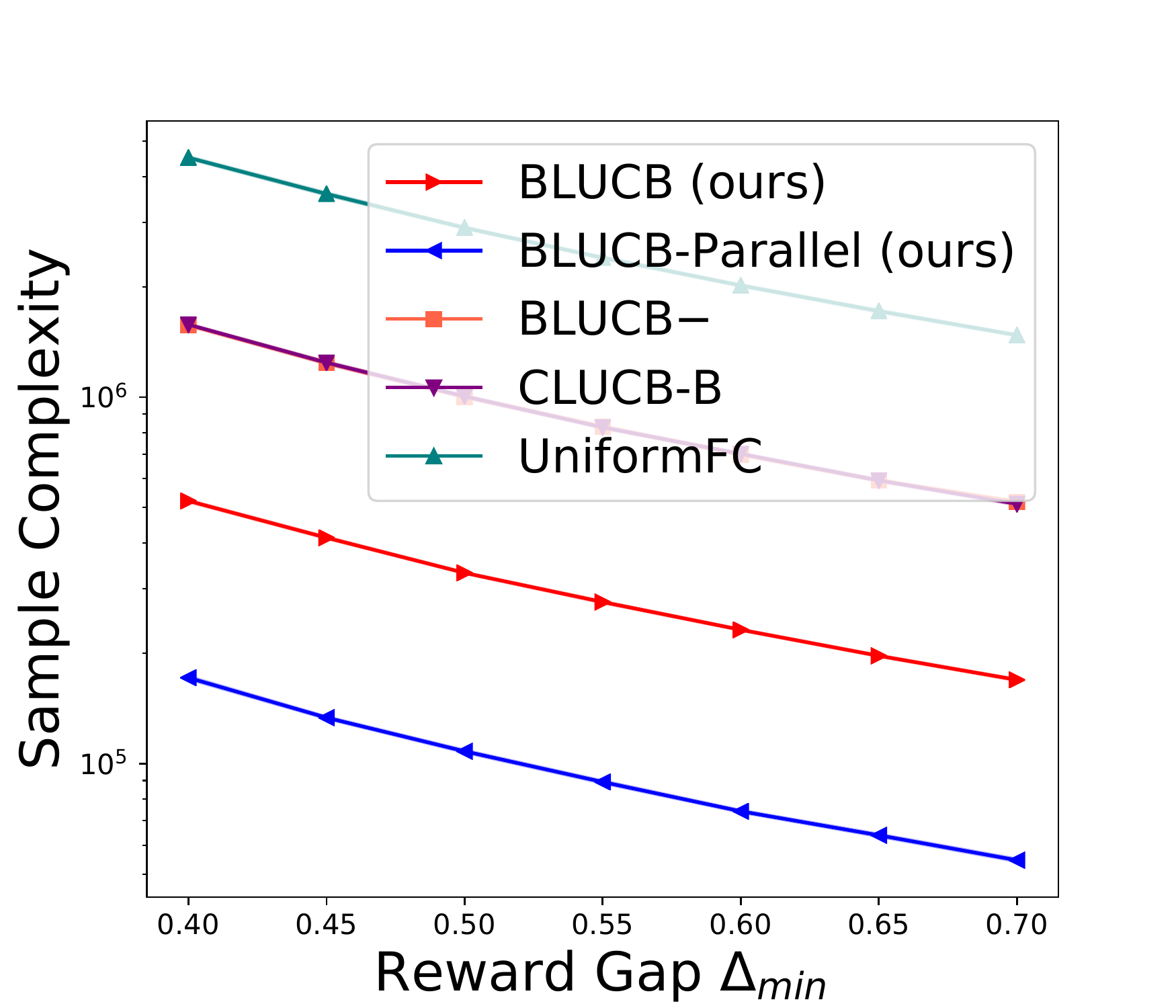}  
	}
	\subfigure[FC, real-world, small $\delta$] { \label{fig:fc_real}
		\includegraphics[width=0.25\columnwidth]{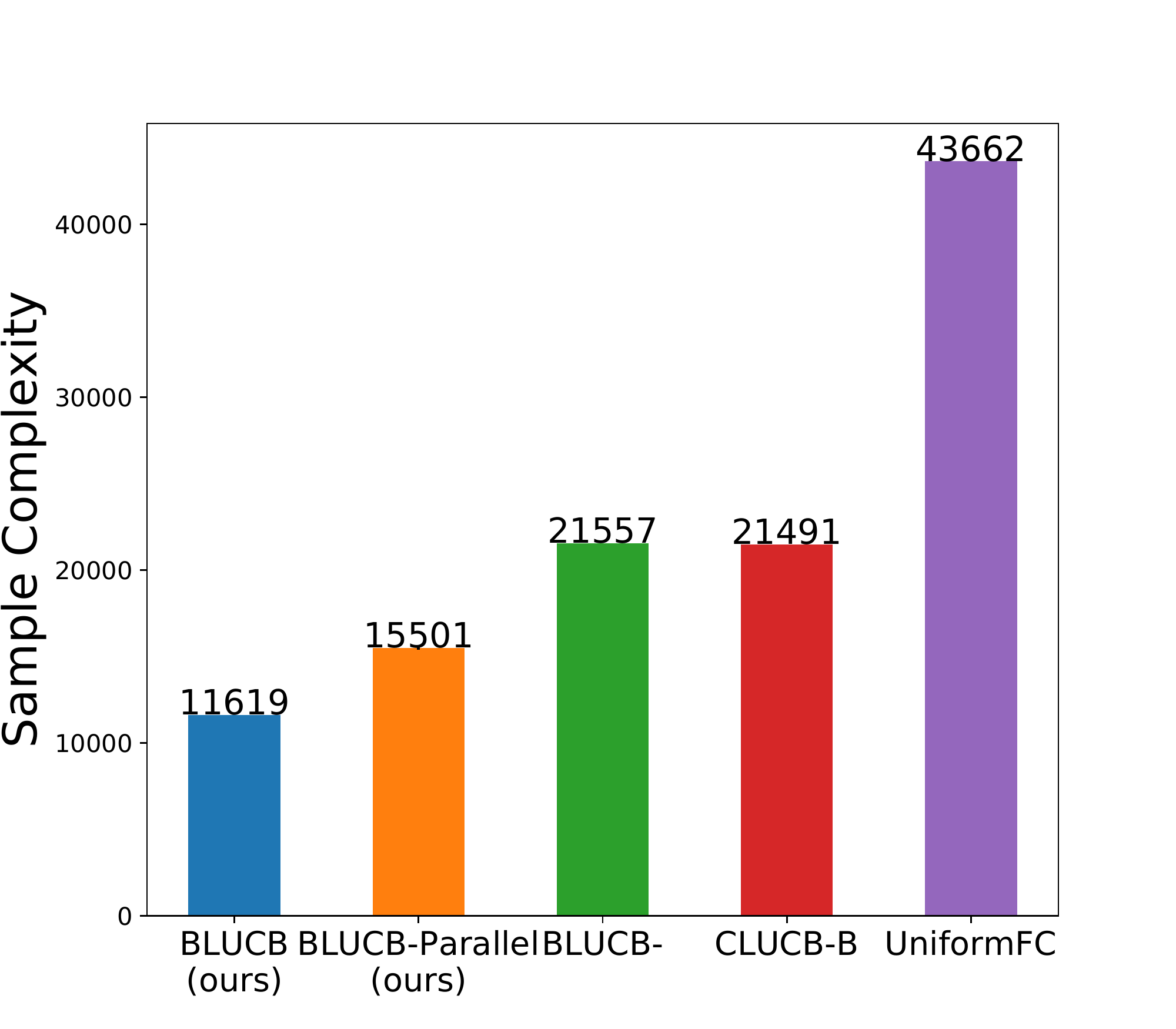}  
	} 
	\\
	\vspace*{-1em}
	\subfigure[FB, matching] { \label{fig:fb_matching_var_budget}
		\includegraphics[width=0.25\columnwidth]{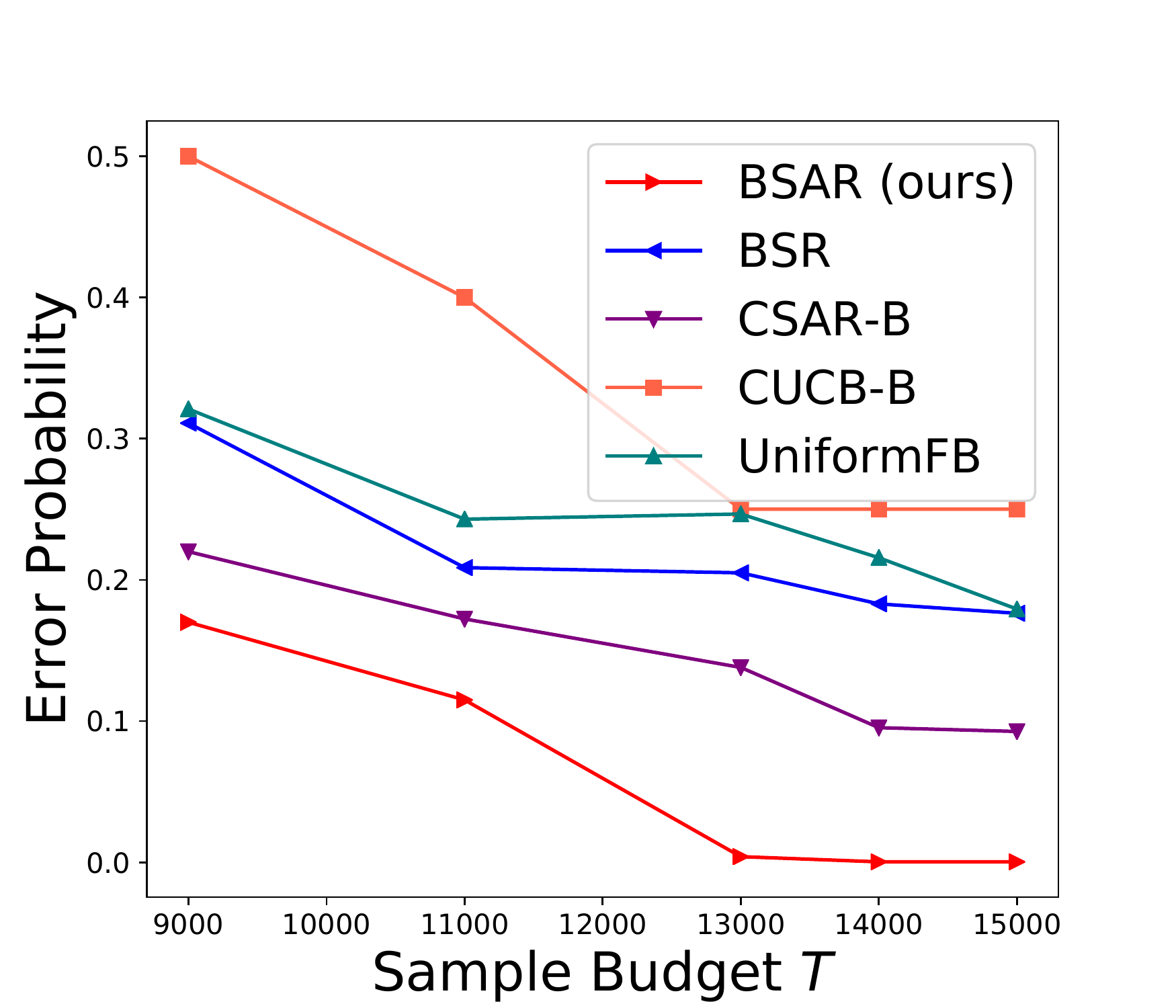}   
	}     
	\subfigure[FB, matching] { \label{fig:fb_matching_var_gap}
		\includegraphics[width=0.25\columnwidth]{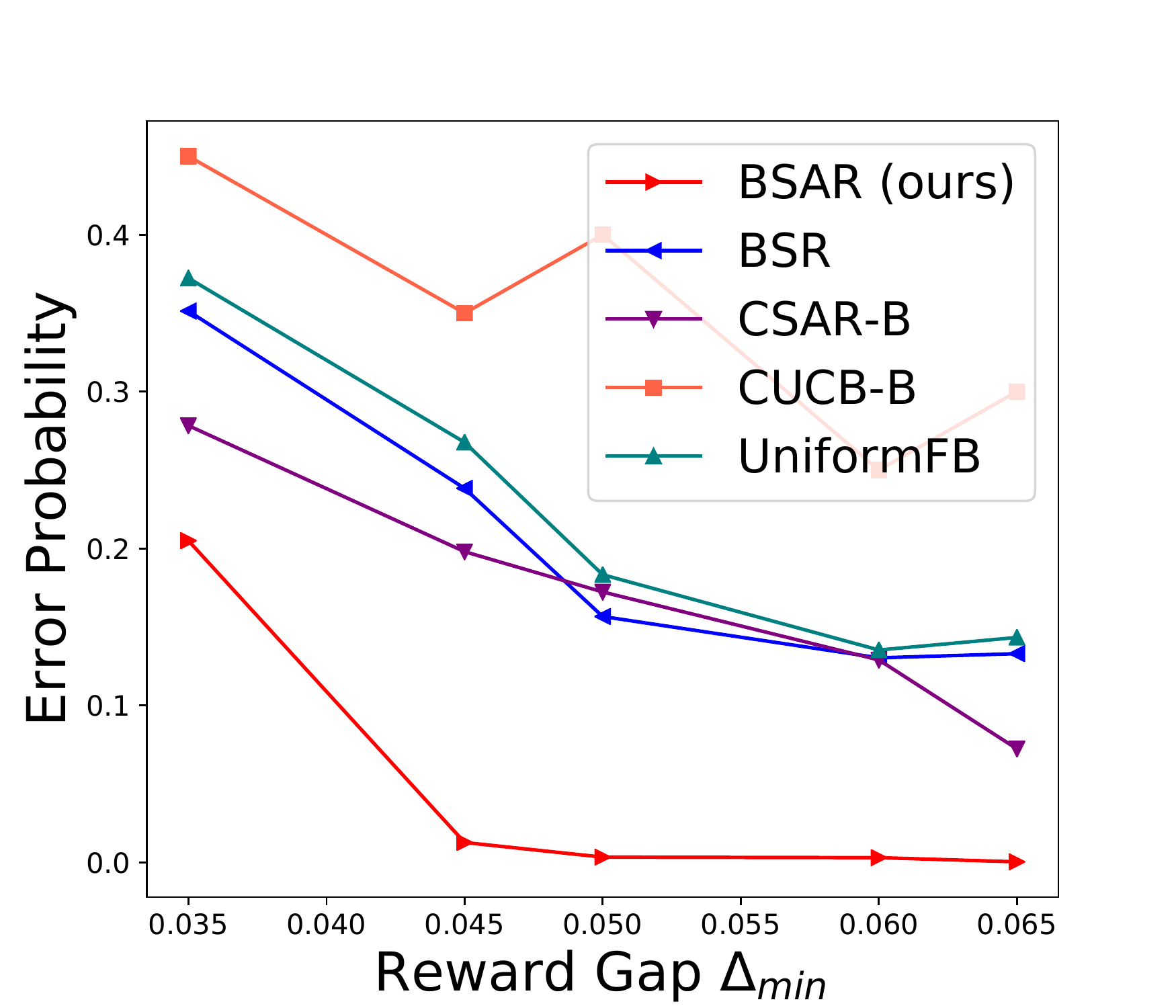}  
	} 
	\subfigure[FB, real-world] { \label{fig:fb_real}
		\includegraphics[width=0.25\columnwidth]{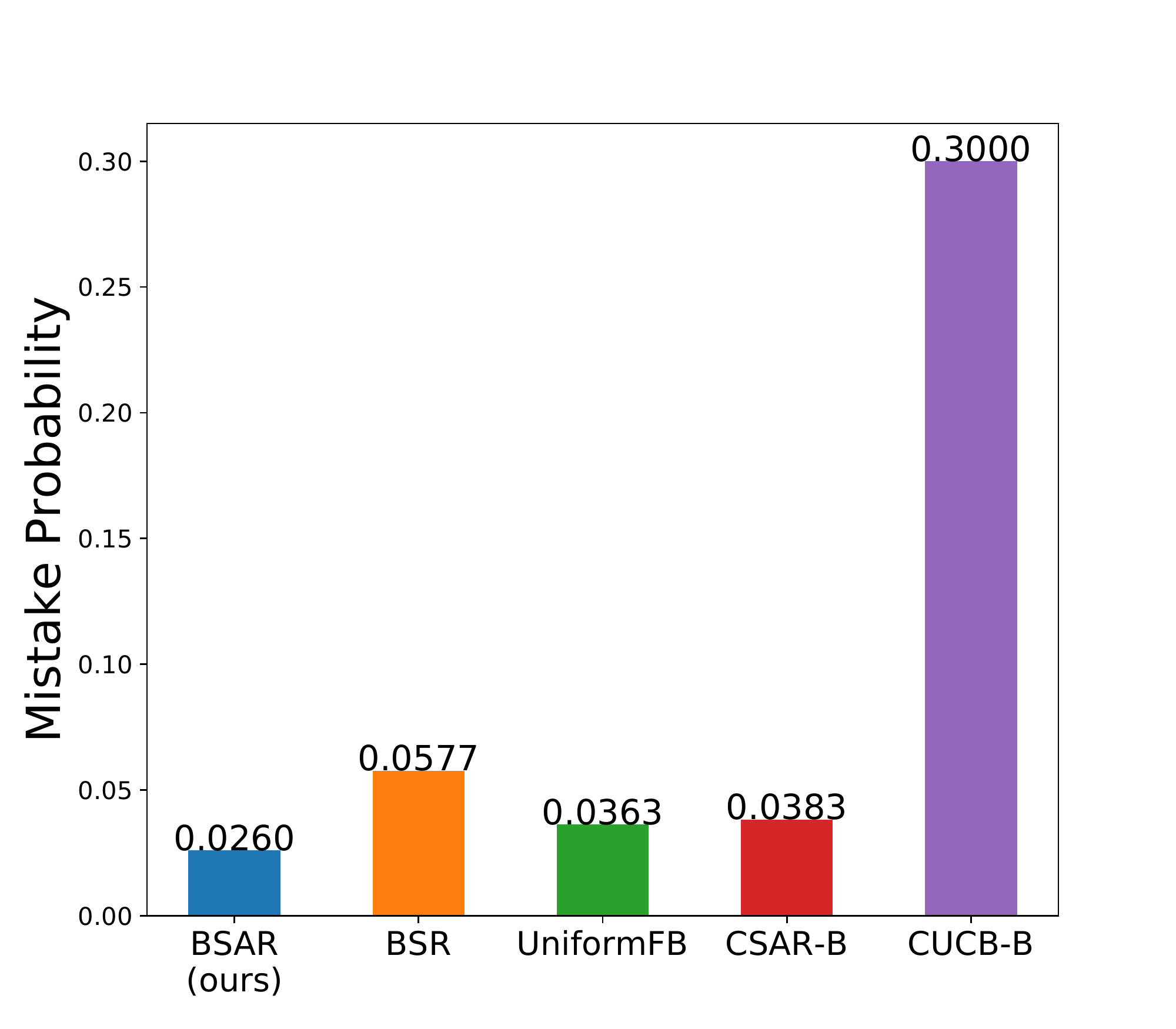}  
	} 
	\caption{Experiments for CPE-B in the FC/FB setting on synthetic and real-world datasets.} \vspace*{-1em}
	\label{fig:experiments_fc}     
\end{figure}

In this section, we conduct experiments for CPE-B in FC/FB settings on synthetic and real-world datasets.
The synthetic dataset consists of the $s$-$t$ path and matching instances. 
For the $s$-$t$ path instance, the number of edges (base arms) $n=85$, and the expected reward of edges $w(e) = [0, 10.5]$ ($e \in [n]$). The minimum reward gap  of any two edges (which is also the minimum gap of bottleneck values between two super arms) is denoted by $\Delta_{\min} \in [0.4,0.7]$. 
For the matching instances, we use a $5 \times 3$ complete bipartite graph, where $n=15$, $w(e) = [0.1, 1.08]$ and $\Delta_{\min}\in[0.03,0.07]$.
We change $\Delta_{\min}$ to generate a series of instances with different hardness (plotted points in Figures~\ref{fig:fc_path_large_delta},\ref{fig:fc_path_small_delta},\ref{fig:fb_matching_var_gap}).
In terms of the real-world dataset, we use the data of American airports and the number of available seats of flights in 2002, provided by the International Air Transportation Association database (\url{www.iata.org})~\cite{architecture2004}. Here we regard an airport as a vertex and a direct flight connecting two airports as an edge (base arm), and also consider the number of available seats of a flight as the expected reward of an edge. Our objective is to find an air route connecting the starting and destination airports which maximizes the minimum number of available seats among its passing flights. In this instance, $n=9$ and $w(e) \in [0.62, 1.84]$.
\OnlyInFull{We present the detailed graphs with specific values of $w(e)$ for the $s$-$t$ path, matching and real-world air route instances in Appendix~\ref{apx:experiment_graph}.}
\OnlyInShort{We present the detailed graphs with specific values of $w(e)$ for the $s$-$t$ path, matching and real-world air route instances in the supplementary material.}

In the FC setting, we set a large $\delta=0.005$ and a small $\delta=\exp(-1000)$, and perform $50$ independent runs to plot average sample complexity with $95\%$ confidence intervals. In the FB setting, we set sample budget $T \in [6000,15000]$, and perform $3000$ independent runs to show the error probability across runs.
For all experiments, the random reward of each edge $e \in [n]$ is i.i.d. drawn from Gaussian distribution $\cN(w(e),1)$.

\yihan{Added more details of experimental setups and presented the graphs in Appendix.}


\textbf{Experiments for the FC setting.}
We compare our $\algbottleneck$/$\algbottleneckparallel$ with three baselines. $\mathtt{BLUCB-}$ is an ablation variant of $\algbottleneck$, which replaces the sample strategy (Lines~\ref{line:blucb_c_t}-\ref{line:blucb_p_t}) with the one that uniformly samples a base arm in critical super arms. 
$\mathtt{CLUCB\mbox{-}B}$~\cite{chen2014cpe} is the state-of-the-art fixed-confidence CPE-L algorithm run with bottleneck reward function.
$\uniformfc$ is a fixed-confidence uniform sampling algorithm.
As shown in Figures~\ref{fig:fc_path_large_delta}-\ref{fig:fc_real}, $\algbottleneck$ and $\algbottleneckparallel$ achieve better performance than the three baselines, which validates the statistical efficiency of our bottleneck-adaptive sample strategy.
Under small $\delta$, $\algbottleneckparallel$ enjoys lower sample complexity than $\algbottleneck$ due to its careful algorithmic design to avoid playing unnecessary base arms, which matches our theoretical results.

\textbf{Experiments for the FB setting.}
Our $\bsar$ is compared with four baselines. 
As an ablation variant of $\bsar$, $\mathtt{BSR}$ removes the special acceptance scheme of $\bsar$. $\mathtt{CSAR\mbox{-}B}$~\cite{chen2014cpe} is the state-of-the-art fixed-budget CPE-L algorithm implemented with bottleneck reward function. 
$\mathtt{CUCB\mbox{-}B}$~\cite{chen2016combinatorial_JMLR} is a regret minimization algorithm allowing nonlinear reward functions, and in pure exploration experiments we let it return the empirical best super arm after $T$ (sample budget) timesteps. 
$\uniformfb$ is a fixed-budget uniform sampling algorithm.
One sees from Figures~\ref{fig:fb_matching_var_budget}-\ref{fig:fb_real} that, $\bsar$ achieves significantly better error probability than all the baselines, which demonstrates that its special acceptance scheme effectively guarantees the correctness for the bottleneck identification task. 

\vspace*{-1em}
\section{Conclusion and Future Work} \label{sec:conclusion_future_work}
\vspace*{-1em}

In this paper, we study the  Combinatorial Pure Exploration with the Bottleneck reward function (CPE-B) problem in FC/FB settings. For the FC setting, we propose two novel algorithms, which achieve the optimal sample complexity for a broad family of instances (within a logarithmic factor), and establish a matching lower bound to demonstrate their optimality.
For the FB setting, we propose an algorithm whose error probability matches the state-of-the-art result, and it is the first to run efficiently on fixed-budget path instances among existing CPE algorithms. 
The empirical evaluation also validates the superior performance of our algorithms.
There are several interesting directions worth further research. One direction is to derive a lower bound for the FB setting, and another direction is to investigate the general nonlinear reward functions.

\begin{ack}
The work of Yihan Du is  supported  in  part  by  the  Technology  and Innovation  Major  Project  of  the  Ministry  of  Science  and Technology  of  China  under  Grant 2020AAA0108400 and 2020AAA0108403. 
Yuko Kuroki is supported by Microsoft Research Asia and JST ACT-X 1124477.
\end{ack}

\bibliographystyle{plain}
\bibliography{neurips_2021_bottleneck_ref}

\begin{thebibliography}{10}

\bibitem{improved_linear_bandit2011}
Yasin Abbasi-Yadkori, D{\'a}vid P{\'a}l, and Csaba Szepesv{\'a}ri.
\newblock Improved algorithms for linear stochastic bandits.
\newblock {\em Advances in Neural Information Processing Systems},
  24:2312--2320, 2011.

\bibitem{agrawal2012analysis}
Shipra Agrawal and Navin Goyal.
\newblock Analysis of thompson sampling for the multi-armed bandit problem.
\newblock In {\em Conference on Learning Theory}, pages 39--1, 2012.

\bibitem{Audibert2010}
Jean-Yves Audibert, S{\'e}bastien Bubeck, and Remi Munos.
\newblock Best arm identification in multi-armed bandits.
\newblock In {\em Conference on Learning Theory}, pages 41--53, 2010.

\bibitem{UCB_auer2002}
Peter Auer, Nicolo Cesa-Bianchi, and Paul Fischer.
\newblock Finite-time analysis of the multiarmed bandit problem.
\newblock {\em Machine Learning}, 47(2-3):235--256, 2002.

\bibitem{communication_networks}
Ron Banner and Ariel Orda.
\newblock Bottleneck routing games in communication networks.
\newblock {\em IEEE Journal on Selected Areas in Communications},
  25(6):1173--1179, 2007.

\bibitem{architecture2004}
Alain Barrat, Marc Barthelemy, Romualdo Pastor-Satorras, and Alessandro
  Vespignani.
\newblock The architecture of complex weighted networks.
\newblock {\em Proceedings of the National Academy of Sciences},
  101(11):3747--3752, 2004.

\bibitem{bubeck2013_SAR}
S{\'e}ebastian Bubeck, Tengyao Wang, and Nitin Viswanathan.
\newblock Multiple identifications in multi-armed bandits.
\newblock In {\em International Conference on Machine Learning}, pages
  258--265, 2013.

\bibitem{bottleneck_spanning_tree1978}
Paolo~M. Camerini.
\newblock The min-max spanning tree problem and some extensions.
\newblock {\em Information Processing Letters}, 7(1):10--14, 1978.

\bibitem{matroid_ChenLJ2016}
Lijie Chen, Anupam Gupta, and Jian Li.
\newblock Pure exploration of multi-armed bandit under matroid constraints.
\newblock In {\em Conference on Learning Theory}, pages 647--669, 2016.

\bibitem{ChenLJ_Nearly_Optimal_Sampling17}
Lijie Chen, Anupam Gupta, Jian Li, Mingda Qiao, and Ruosong Wang.
\newblock Nearly optimal sampling algorithms for combinatorial pure
  exploration.
\newblock In {\em Conference on Learning Theory}, pages 482--534, 2017.

\bibitem{chen2014cpe}
Shouyuan Chen, Tian Lin, Irwin King, Michael~R Lyu, and Wei Chen.
\newblock Combinatorial pure exploration of multi-armed bandits.
\newblock In {\em Advances in Neural Information Processing Systems}, pages
  379--387, 2014.

\bibitem{CPE_DB_2020}
Wei Chen, Yihan Du, Longbo Huang, and Haoyu Zhao.
\newblock Combinatorial pure exploration for dueling bandit.
\newblock In {\em International Conference on Machine Learning}, pages
  1531--1541, 2020.

\bibitem{chen2013combinatorial_ICML}
Wei Chen, Yajun Wang, and Yang Yuan.
\newblock Combinatorial multi-armed bandit: General framework and applications.
\newblock In {\em International Conference on Machine Learning}, pages
  151--159, 2013.

\bibitem{chen2016combinatorial_JMLR}
Wei Chen, Yajun Wang, Yang Yuan, and Qinshi Wang.
\newblock Combinatorial multi-armed bandit and its extension to
  probabilistically triggered arms.
\newblock {\em The Journal of Machine Learning Research}, 17(1):1746--1778,
  2016.

\bibitem{combes2015combinatorial}
Richard Combes, Mohammad~Sadegh Talebi Mazraeh~Shahi, Alexandre Proutiere,
  et~al.
\newblock Combinatorial bandits revisited.
\newblock In {\em Advances in Neural Information Processing Systems}, pages
  2116--2124, 2015.

\bibitem{CPE_BL_PL2020}
Yihan Du, Yuko Kuroki, and Wei Chen.
\newblock Combinatorial pure exploration with full-bandit or partial linear
  feedback.
\newblock {\em Proceedings of the AAAI Conference on Artificial Intelligence},
  2021.

\bibitem{Even2006}
Eyal Even-Dar, Shie Mannor, and Yishay Mansour.
\newblock Action elimination and stopping conditions for the multi-armed bandit
  and reinforcement learning problems.
\newblock {\em Journal of Machine Learning Research}, 7:1079--1105, 2006.

\bibitem{Multi_Bandit}
Victor Gabillon, Mohammad Ghavamzadeh, Alessandro Lazaric, and S{\'e}bastien
  Bubeck.
\newblock Multi-bandit best arm identification.
\newblock In {\em Advances in Neural Information Processing Systems}, 2011.

\bibitem{gabillon2016improved}
Victor Gabillon, Alessandro Lazaric, Mohammad Ghavamzadeh, Ronald Ortner, and
  Peter Bartlett.
\newblock Improved learning complexity in combinatorial pure exploration
  bandits.
\newblock In {\em Artificial Intelligence and Statistics}, pages 1004--1012,
  2016.

\bibitem{two_vertex_connectivity2018}
Loukas Georgiadis, Giuseppe~F Italiano, Luigi Laura, and Nikos Parotsidis.
\newblock 2-vertex connectivity in directed graphs.
\newblock {\em Information and Computation}, 261:248--264, 2018.

\bibitem{Goldberg84}
A.~V. Goldberg.
\newblock Finding a maximum density subgraph.
\newblock Technical report, University of California Berkeley, 1984.

\bibitem{Hamiltonian_path_problem}
Yuri Gurevich and Saharon Shelah.
\newblock Expected computation time for hamiltonian path problem.
\newblock {\em SIAM Journal on Computing}, 16(3):486--502, 1987.

\bibitem{tree_allocation_problem}
Dorit~S Hochbaum and Sung-Pil Hong.
\newblock About strongly polynomial time algorithms for quadratic optimization
  over submodular constraints.
\newblock {\em Mathematical programming}, 69(1):269--309, 1995.

\bibitem{huang_CPE_CS2018}
Weiran Huang, Jungseul Ok, Liang Li, and Wei Chen.
\newblock Combinatorial pure exploration with continuous and separable reward
  functions and its applications.
\newblock In {\em International Joint Conference on Artificial Intelligence},
  pages 2291--2297, 2018.

\bibitem{kalyanakrishnan2012}
Shivaram Kalyanakrishnan, Ambuj Tewari, Peter Auer, and Peter Stone.
\newblock {PAC} subset selection in stochastic multi-armed bandits.
\newblock In {\em International Conference on Machine Learning}, pages
  655--662, 2012.

\bibitem{verification_karnin2016}
Zohar~S Karnin.
\newblock Verification based solution for structured mab problems.
\newblock In {\em Advances in Neural Information Processing Systems}, pages
  145--153, 2016.

\bibitem{kaufmann2016complexity}
Emilie Kaufmann, Olivier Capp{\'e}, and Aur{\'e}lien Garivier.
\newblock On the complexity of best-arm identification in multi-armed bandit
  models.
\newblock {\em The Journal of Machine Learning Research}, 17(1):1--42, 2016.

\bibitem{dense_subgraphs}
Samir Khuller and Barna Saha.
\newblock On finding dense subgraphs.
\newblock In {\em International Colloquium on Automata, Languages, and
  Programming}, pages 597--608. Springer, 2009.

\bibitem{Online_Dense_Subgraph_Kuroki2020}
Yuko Kuroki, Atsushi Miyauchi, Junya Honda, and Masashi Sugiyama.
\newblock Online dense subgraph discovery via blurred-graph feedback.
\newblock In {\em International Conference on Machine Learning}, pages
  5522--5532, 2020.

\bibitem{Multiple-Arm_Identification_Kuroki2020}
Yuko Kuroki, Liyuan Xu, Atsushi Miyauchi, Junya Honda, and Masashi Sugiyama.
\newblock Polynomial-time algorithms for multiple-arm identification with
  full-bandit feedback.
\newblock {\em Neural Computation}, 32(9):1733--1773, 2020.

\bibitem{lai_robbins1985}
Tze~Leung Lai and Herbert Robbins.
\newblock Asymptotically efficient adaptive allocation rules.
\newblock {\em Advances in Applied Mathematics}, 6(1):4--22, 1985.

\bibitem{APT2016}
Andrea Locatelli, Maurilio Gutzeit, and Alexandra Carpentier.
\newblock An optimal algorithm for the thresholding bandit problem.
\newblock In {\em International Conference on Machine Learning}, pages
  1690--1698, 2016.

\bibitem{bottleneck_shortest_path2006}
M~Peinhardt and V~Kaibel.
\newblock On the bottleneck shortest path problem.
\newblock {\em Technical Report}, 2006.

\bibitem{bottleneck_bipartite_matching1994}
Abraham~P Punnen and KPK Nair.
\newblock Improved complexity bound for the maximum cardinality bottleneck
  bipartite matching problem.
\newblock {\em Discrete Applied Mathematics}, 55(1):91--93, 1994.

\bibitem{Tao2018}
Chao Tao, Sa{\'u}l Blanco, and Yuan Zhou.
\newblock Best arm identification in linear bandits with linear dimension
  dependency.
\newblock In {\em International Conference on Machine Learning}, pages
  4877--4886, 2018.

\bibitem{thompson1933}
William~R Thompson.
\newblock On the likelihood that one unknown probability exceeds another in
  view of the evidence of two samples.
\newblock {\em Biometrika}, 25(3/4):285--294, 1933.

\bibitem{Quadratic_network_flow_problems}
Hoang Tuy, Saied Ghannadan, Athanasios Migdalas, and Peter V{\"a}rbrand.
\newblock A strongly polynomial algorithm for a concave
  production-transportation problem with a fixed number of nonlinear variables.
\newblock {\em Mathematical Programming}, 72(3):229--258, 1996.

\bibitem{urban_traffic}
Wenwei Yue, Changle Li, and Guoqiang Mao.
\newblock Urban traffic bottleneck identification based on congestion
  propagation.
\newblock In {\em 2018 IEEE International Conference on Communications}, pages
  1--6, 2018.

\bibitem{network_architecture_search}
Yiyi Zhou, Rongrong Ji, Xiaoshuai Sun, Gen Luo, Xiaopeng Hong, Jinsong Su,
  Xinghao Ding, and Ling Shao.
\newblock K-armed bandit based multi-modal network architecture search for
  visual question answering.
\newblock In {\em Proceedings of the 28th ACM International Conference on
  Multimedia}, pages 1245--1254, 2020.

\end{thebibliography}

\OnlyInFull{
	\clearpage
\appendix	
\section*{Appendix}
\section{More Details of Experimental Setups} \label{apx:experiment_graph}

\begin{figure}[h] 
	\centering    
	\subfigure[$s$-$t$ path] { \label{fig:graph_path}
		\includegraphics[width=0.45\columnwidth]{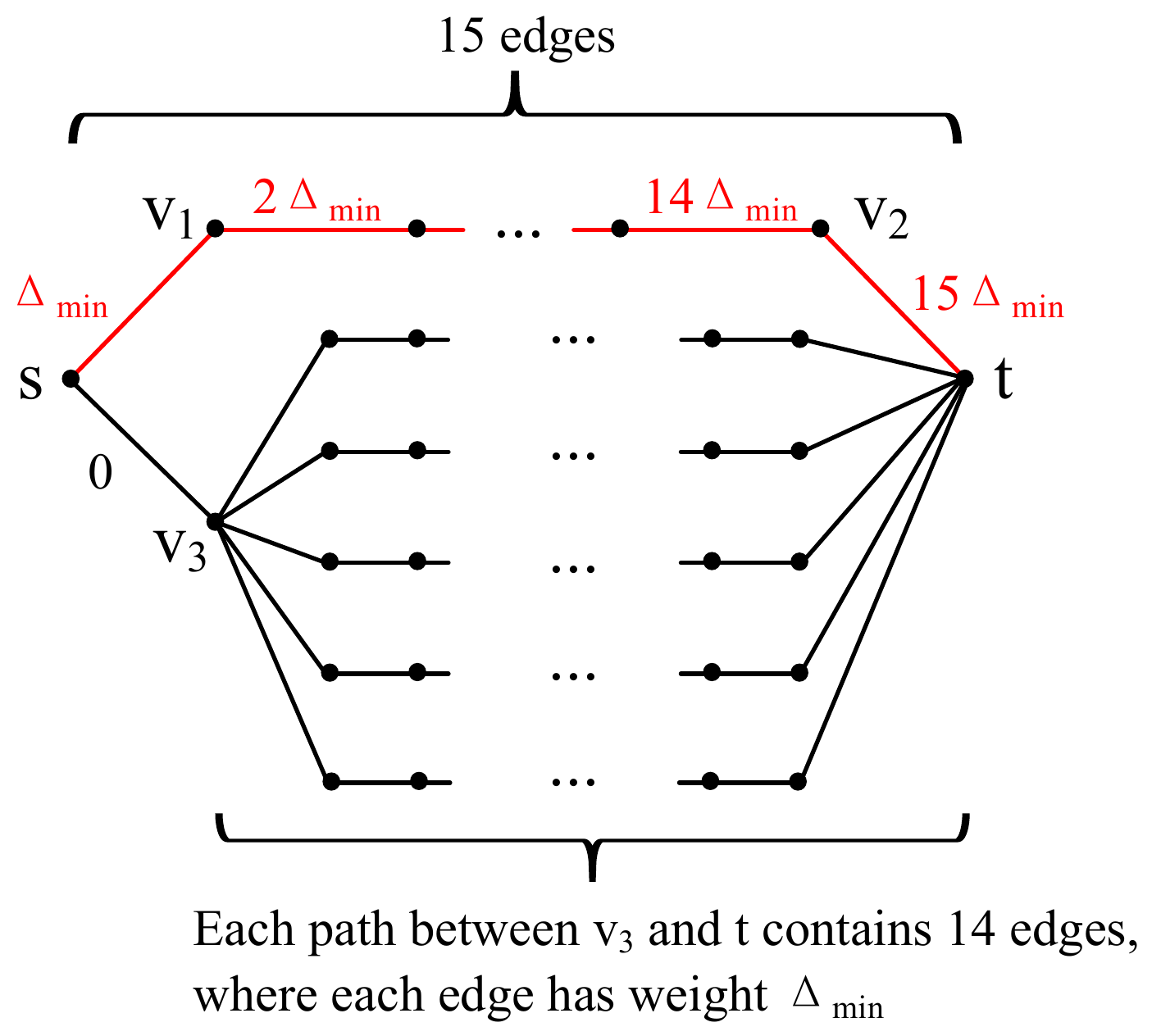}  
	}
	\hfill
	\subfigure[Matching] { \label{fig:graph_matching}
		\includegraphics[width=0.45\columnwidth]{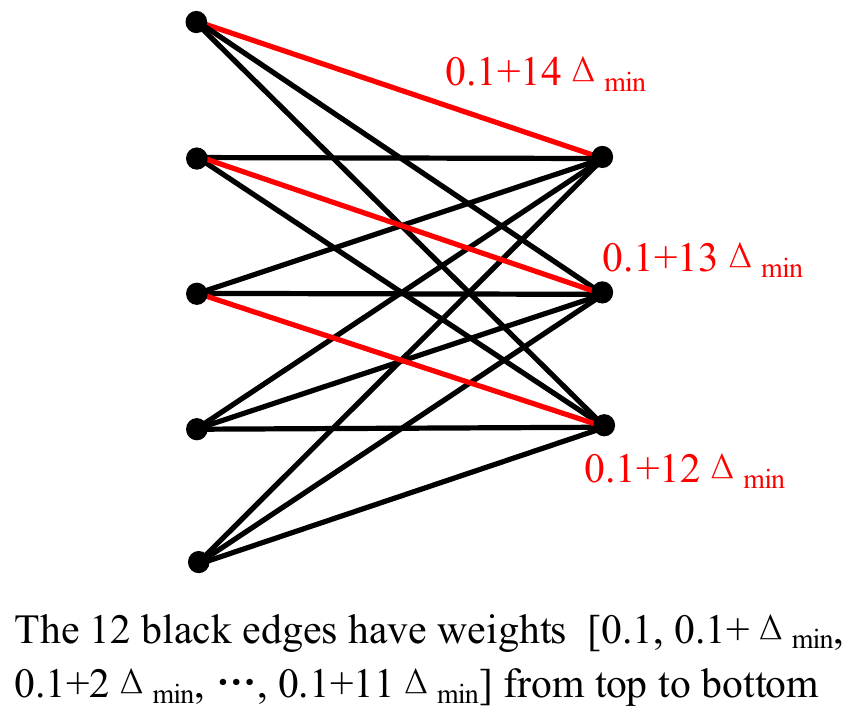}  
	}
	\\
	\subfigure[Real-world air route] { \label{fig:graph_real}
	\includegraphics[width=0.6\columnwidth]{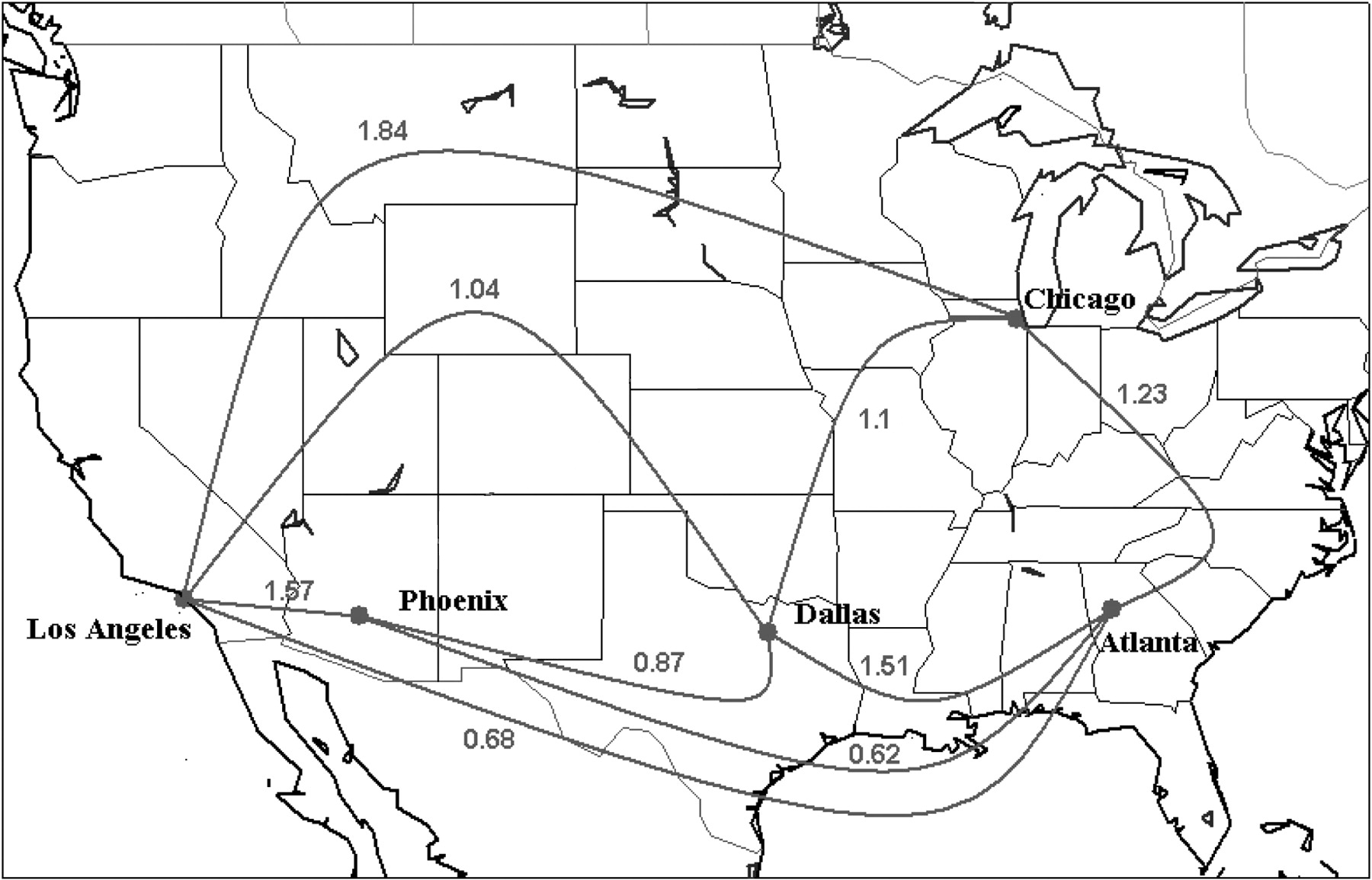}  
	}
	\caption{Graphs of the $s$-$t$ path, matching and real-world air route instances in our experiments.}   
	\label{fig:apx_experiments_fc}  
\end{figure}

In this section, we supplement more details of graphs and the expected rewards of edges (base arms) for the $s$-$t$ path, matching and real-world air route instances in our experiments.

Figure~\ref{fig:graph_path} shows the graph of $s$-$t$ path instance. The red path contains $15$ edges with weights $[\Delta_{\min},2\Delta_{\min},\dots,15\Delta_{\min}]$ and is the optimal $s$-$t$ path with the maximum bottleneck value. There are $5$ paths connecting $v_3$ and $t$, and each of them contains $14$ edges with weights $\Delta_{\min}$. In this instance, we set $\Delta_{\min} \in [0.4,0.7]$.

As shown in Figure~\ref{fig:graph_matching}, the matching instance uses a $5 \times 3$ complete bipartite graph with $n=15$ edges. The red matching is the optimal one, which contains three edges with weights $[0.1+14\Delta_{\min},0.1+13\Delta_{\min},0.1+12\Delta_{\min}]$. The remaining 12 black edges have weights $[0.1,0.1+\Delta_{\min},\dots,0.1+11\Delta_{\min}]$ from top to bottom. In this instance, $\Delta_{\min}\in[0.03,0.07]$.

Figure~\ref{fig:graph_real} illustrates the graph of real-world air route instance, which is originated from \cite{architecture2004}. We regard an airport (e.g., Los Angeles) as a vertex and a direct flight connecting two airports (e.g., Los Angeles $\leftrightarrow$ Chicago) as an edge. The number marked on each edge denotes the number of available seats of this flight, i.e., the expected reward of this edge. Our objective is to find an air route connecting Los Angeles and Atlanta, and the optimal route is [Los Angeles $\leftrightarrow$ Chicago $\leftrightarrow$ Atlanta].

\yihan{Added the graphs and more details of experimental setups.}

\section{CPE-B in the Fixed-Confidence Setting}
\subsection{Proof for Algorithm $\algbottleneck$} \label{apx:blucb}
In this subsection, we prove the sample complexity of Algorithm $\algbottleneck$ (Theorem~\ref{thm:bottleneck_baseline_ub}).

In order to prove Theorem~\ref{thm:bottleneck_baseline_ub}, we first introduce the following Lemmas~\ref{lemma:concentration}-\ref{lemma:blucb_e_notin_M_star_lower}.
For ease of notation, we define a function $\emin(M, \bv)$ to return the base arm with the minimum reward in $M$ with respect to weight vector $\bv$, i.e., $\emin(M, \bv) \in \argmin_{e \in M} v(e)$.

\begin{lemma}[Concentration] \label{lemma:concentration}
	For any $t>0$ and $e \in [n]$, defining the confidence radius $\rad_t(e) = R \sqrt{ \frac{2 \ln (\frac{4nt^3}{\delta})}{T_t(e)} }$ and the events
	$$
	\xi_t=\left\{ \forall e \in [n],\  | w(e)-\hat{w}_t(e) | < \rad_t(e) \right\}
	$$ 
	and 
	$$ 
	\xi=\bigcap \limits_{t=1}^{\infty} \xi_t , 
	$$
	then, we have 
	$$\Pr[\xi] \geq 1-\delta. $$
\end{lemma}
\begin{proof}
	Since for any $e \in [n]$, the reward distribution of base arm $e$ has an R-sub-Gaussian tail and the mean of $w(e)$, according to the Hoeffding's inequality, we have that for any $t>0$ and $e \in [n]$,
	\begin{align*}
		\Pr \mbr{ | w(e)-\hat{w}_t(e) | \geq R \sqrt{ \frac{2 \ln (\frac{4nt^3}{\delta})}{T_t(e)} } } 
		= & \sum_{s=1}^{t-1} \Pr \mbr{ | w(e)-\hat{w}_t(e) | \geq R \sqrt{ \frac{2 \ln (\frac{4nt^3}{\delta})}{T_t(e)} } ,\  T_t(e)=s}
		\\
		\leq & \sum_{s=1}^{t-1} \frac{\delta}{2nt^3}
		\\
		\leq & \frac{\delta}{2nt^2}
	\end{align*}
	Using a union bound over $e \in [n]$, we have 
	\begin{align*}
	\Pr \mbr{ \xi_t }  \leq  \frac{\delta}{2t^2}
	\end{align*}
	and thus
	\begin{align*}
	\Pr \mbr{ \xi } \geq & 1- \sum_{t=1}^{\infty} \Pr[\neg \xi_t]
	\\
	\geq & 1- \sum_{t=1}^{\infty} \frac{\delta}{2t^2}
	\\
	\geq & 1 - \delta
	\end{align*}
\end{proof}

\begin{lemma} \label{lemma:correctness_blucb}
	Assume that event $\xi$ occurs. Then, if algorithm $\algbottleneck$  (Algorithm~\ref{alg:bottleneck}) terminates at round $t$, we have $M_t=M_*$.
\end{lemma}
\begin{proof}
	According to the stop condition (Line \ref{line:blucb_stop} of Algorithm~\ref{alg:bottleneck}), when algorithm $\algbottleneck$  terminates at round $t$, we have that for any $M \in \cM \setminus \superset(M_t)$,
	$$
	\wmin(M_t, \bw) \geq \wmin(M_t, \underline{\bw}_t) \geq \wmin(M, \bar{\bw}_t) \geq \wmin(M, \bw) .
	$$
	For any $M \in \superset(M_t)$, according to the property of the bottleneck reward function,
	we have
	$$
	\wmin(M_t, \bw)  \geq \wmin(M, \bw) .
	$$
	Thus, we have $\wmin(M_t, \bw)  \geq \wmin(M, \bw)$ for any $M \neq M_t$ and according to the unique assumption of $M_*$, we obtain $M_t=M_*$.
\end{proof}

\begin{lemma} \label{lemma:blucb_e_in_M_star}
	Assume that event $\xi$ occurs. For any $e \in M_*$, if $\rad_t(e)<\frac{\Delta^\fc_e}{4}=\frac{1}{4}(w(e)-\max_{M \neq M_*} \wmin(M, \bw))$, then, base arm $e$ will not be pulled at round $t$, i.e., $p_t \neq e$.
\end{lemma}
\begin{proof}
	Suppose that for some $e \in M_*$, $\rad_t(e)<\frac{\Delta^\fc_e}{4}=\frac{1}{4}(w(e)-\max_{M \neq M_*} \wmin(M, \bw))$ and $p_t = e$. 
	According to the selection strategy of $p_t$, we have that $\rad_t(c_t) < \frac{\Delta^\fc_e}{4}$ and $\rad_t(d_t) < \frac{\Delta^\fc_e}{4}$.
	
	Case (i): If $e$ is selected from $M_*$, then one of $M_t$ and $\tilde{M}_t$ is $M_*$ such that $e=\emin(M_*, \underline{\bw}_t)$, and the other is a sub-optimal super arm $M'$. Let $e'=\emin(M', \underline{\bw}_t)$. $\underline{w}(e') \leq \underline{w}(\emin(M', \bw)) \leq w(\emin(M', \bw))= \wmin(M', \bw)$. $\{e, e'\}=\{c_t, d_t\}$. Then, 
	\begin{align*}
	\underline{w}(e)-\bar{w}(e') \geq & w(e)-\underline{w}(e')-2\rad_t(e)-2\rad_t(e')
	\\
	> & w(e)-\wmin(M', \bw)-\Delta^\fc_e 
	\\
	\geq & 0.
	\end{align*}
	Then, we have 
	\begin{align*}
	\wmin(M_*, \underline{\bw}_t) = \underline{w}(e) > \bar {w}(e') \geq \wmin(M', \bar{\bw}_t),
	\end{align*}
	and algorithm $\algbottleneck$ must have stopped, which gives a contradiction.
	
	Case (ii): If $e$ is selected from a sub-optimal super arm $M$, then one of $M_t$ and $\tilde{M}_t$ is $M$ such that $e=\emin(M, \underline{\bw}_t)$. Since $e \in M_*$, we have $w(e) \geq \wmin(M_*, \bw) > \wmin(M, \bw)$ and thus $w(e)-\wmin(M, \bw)=w(e)-w(\emin(M, \bw))>0$. Then, 
	\begin{align*}
	\underline{w}(e)-\underline{w}(\emin(M, \bw)) \geq & w(e)-2\rad_t(e)-\wmin(M, \bw)
	\\
	> & w(e)-\wmin(M, \bw)-\frac{\Delta^\fc_e}{2}
	\\
	>0,
	\end{align*}
	which contradicts $e=\emin(M, \underline{\bw}_t)$.
\end{proof}

\begin{lemma} \label{lemma:blucb_e_notin_M_star_upper}
	Assume that event $\xi$ occurs. For any $e \notin M_*, w(e) \geq \wmin(M_*, \bw)$, if $\rad_t(e)<\frac{\Delta^\fc_e}{2}=\frac{1}{2}(w(e)-\max_{M \in \cM: e\in M} \wmin(M, \bw))$, then, base arm $e$ will not be pulled at round $t$, i.e., $p_t \neq e$.
\end{lemma}
\begin{proof}
	Suppose that for some  $e \notin M_*, w(e) \geq \wmin(M_*, \bw)$, $\rad_t(e)<\frac{\Delta^\fc_e}{2}=\frac{1}{2}(w(e)-\max_{M \in \cM: e\in M} \wmin(M, \bw))$ and $p_t = e$. 
	According to the selection strategy of $p_t$, we have that $\rad_t(c_t) < \frac{\Delta^\fc_e}{2}$ and $\rad_t(d_t) < \frac{\Delta^\fc_e}{2}$.
	
	Since $e \notin M_*$, $e$ is selected from a sub-optimal super arm $M$. One of $M_t$ and $\tilde{M}_t$ is $M$ such that $e=\emin(M, \underline{\bw}_t)$. Since $w(e) \geq \wmin(M_*, \bw) > \wmin(M, \bw)$, we have $w(e)-\wmin(M, \bw)=w(e)-w(\emin(M, \bw))>0$. Then, 
	\begin{align*}
	\underline{w}(e)-\underline{w}(\emin(M, \bw)) \geq & w(e)-2\rad_t(e)-\wmin(M, \bw)
	\\
	> & w(e)-\wmin(M, \bw)-\Delta^\fc_e
	\\
	\geq & 0,
	\end{align*}
	which contradicts $e=\emin(M, \underline{\bw}_t)$.
\end{proof}

\begin{lemma} \label{lemma:blucb_e_notin_M_star_lower}
	Assume that event $\xi$ occurs. For any $e \notin M_*, w(e) < \wmin(M_*, \bw)$, if $\rad_t(e)<\frac{\Delta^\fc_e}{4}=\frac{1}{4}(\wmin(M_*, \bw)-\max_{M \in \cM: e\in M} \wmin(M, \bw))$, then, base arm $e$ will not be pulled at round $t$, i.e., $p_t \neq e$.
\end{lemma}
\begin{proof}
	Suppose that for some  $e \notin M_*, w(e) < \wmin(M_*, \bw)$, $\rad_t(e)<\frac{\Delta^\fc_e}{4}=\frac{1}{4}(\wmin(M_*, \bw)-\max_{M \in \cM: e\in M} \wmin(M, \bw))$ and $p_t = e$. 
	According to the selection strategy of $p_t$, we have that $\rad_t(c_t) < \frac{\Delta^\fc_e}{4}$ and $\rad_t(d_t) < \frac{\Delta^\fc_e}{4}$.
	
	Case (i): If one of $M_t$ and $\tilde{M}_t$ is $M_*$, then the other is a sub-optimal super arm $M$ such that $e=\emin(M, \underline{\bw}_t)$. Let $f=\emin(M_*, \underline{\bw}_t)$. $\{e, f\}=\{c_t, d_t\}$. Then, we have
	\begin{align*}
	\underline{w}(f)-\bar {w}(e) \geq & w(f)-2\rad_t(f)-\underline{w}(e)-2\rad_t(e)
	\\
	> & w(f)-\underline{w}(e)-\Delta^\fc_e
	\\
	\geq & \wmin(M_*, \bw)-\underline{w}(\emin(M, \bw))-\Delta^\fc_e
	\\
	\geq & \wmin(M_*, \bw)-\wmin(M, \bw)-\Delta^\fc_e
	\\
	\geq & 0.
	\end{align*}
	Thus,
	\begin{align*}
	\wmin(M_*, \bw) \geq \underline{w}(\emin(M_*, \bw)) \geq \underline{w}(f)>\bar {w}(e)\geq \wmin(M, \bar{\bw}_t)
	\end{align*}
	and algorithm $\algbottleneck$ must have stopped, which gives a contradiction.
	
	Case (ii): 
	If neither $M_t$ nor $\tilde{M}_t$ is $M_*$ and $e = c_t$, i.e., $e=\emin(M_t, \underline{\bw}_t)$, we have
	\begin{align*}
	\bar {w}(d_t) \geq  \wmin(\tilde{M}_t, \bar{\bw}_t) \geq  \wmin(M_*, \bar{\bw}_t) \geq  \wmin(M_*, \bw)
	\end{align*} 
	and
	\begin{align*}
	\underline{w}(d_t) = \wmin(\tilde{M}_t, \underline{\bw}_t) \leq \wmin(M_t, \underline{\bw}_t) \leq \wmin(M_t, \bw),
	\end{align*} 
	and thus
	\begin{align*}
	2 \rad_t(d_t) = & \bar {w}(d_t)-\underline{w}(d_t)
	\\
	\geq & \wmin(M_*, \bw) -\wmin(M_t, \bw),
	\end{align*} 
	which contradicts $\rad_t(d_t) < \frac{\Delta^\fc_e}{4} <\frac{\Delta^\fc_e}{2} \leq \frac{\wmin(M_*, \bw) -\wmin(M_t, \bw)}{2}$.
	
	Case (iii): If neither $M_t$ nor $\tilde{M}_t$ is $M_*$ and $e = d_t$, i.e., $e=\emin(\tilde{M}_t, \underline{\bw}_t)$. Let $c(M_*, \tilde{M}_t)=\frac{1}{2}(\wmin(M_*, \bw)+\wmin(\tilde{M}_t, \bw))$.
	If $c(M_*, \tilde{M}_t) < w(e) < \wmin(M_*, \bw)$, we have 
	\begin{align*}
	\underline{w}(e) \geq & w(e) - 2 \rad_t(e)
	\\
	> & c(M_*, \tilde{M}_t) - \frac{\Delta^\fc_e}{2}
	\\
	\geq & \frac{1}{2}(\wmin(M_*, \bw)+\wmin(\tilde{M}_t, \bw)) - \frac{1}{2}(\wmin(M_*, \bw)- \wmin(\tilde{M}_t, \bw))
	\\
	= & \wmin(\tilde{M}_t, \bw)
	\\
	\geq & \underline{w}(\emin(\tilde{M}_t, \bw)),
	\end{align*}
	which contradicts $e=\emin(\tilde{M}_t, \underline{\bw}_t)$.
	
	If $\wmin(\tilde{M}_t, \bw) \leq w(e) \leq c(M_*, \tilde{M}_t)$, we have 
	\begin{align*}
	\wmin(\tilde{M}_t, \bar{\bw}_t) \leq & \bar {w}(e)
	\\
	\leq & w(e) + 2 \rad_t(e)
	\\
	< & c(M_*, \tilde{M}_t) + \frac{\Delta^\fc_e}{2}
	\\
	\leq & \frac{1}{2}(\wmin(M_*, \bw)+\wmin(\tilde{M}_t, \bw)) + \frac{1}{2}(\wmin(M_*, \bw)- \wmin(\tilde{M}_t, \bw))
	\\
	= & \wmin(M_*, \bw)
	\\
	\leq & \wmin(M_*, \bar{\bw}_t ).
	\end{align*} 
	In addition, from the uniqueness of $M_*$, we have $M_* \notin \superset(\tilde{M}_t)$.
	Thus, the inequality $\wmin(\tilde{M}_t, \bar{\bw}_t) < \wmin(M_*, \bar{\bw}_t )$ violates the optimality of $\tilde{M}_t$ with respect to $ \bar{\bw}_t$. 
\yihan{Revised the inequality and added the explanation.}
\end{proof}

Next, we prove Theorem~\ref{thm:bottleneck_baseline_ub}.
\begin{proof}
	For any $e \in [n]$, let $T(e)$ denote the number of samples for base arm $e$, and $t_e$ denote the last timestep at which $e$ is pulled. Then, we have $T_{t_e}=T(e)-1$. Let $T$ denote the total number of samples.
	According to Lemmas~\ref{lemma:blucb_e_in_M_star}-\ref{lemma:blucb_e_notin_M_star_lower}, we have
	\begin{align*}
		R \sqrt{ \frac{2 \ln (\frac{4 n t_e^3}{\delta})}{T(e)-1} } \geq \frac{1}{4} \Delta^\fc_e
	\end{align*}
	Thus, we obtain
	\begin{align*}
		T(e) \leq \frac{32 R^2 }{(\Delta^\fc_e)^2}\ln \sbr{\frac{4 n t_e^3}{\delta}} +1 \leq \frac{32 R^2 }{(\Delta^\fc_e)^2}\ln \sbr{\frac{4 n T^3}{\delta}} +1
	\end{align*}
	Summing over $e \in [n]$, we have
	\begin{align*}
	T \leq \sum_{e \in [n]} \frac{32 R^2 }{(\Delta^\fc_e)^2} \ln \sbr{\frac{4 n T^3 }{\delta}   } + n \leq \sum_{e \in [n]} \frac{96 R^2 }{(\Delta^\fc_e)^2} \ln \sbr{\frac{2 n T }{\delta}   } + n,
	\end{align*}
	where $\sum_{e \in [n]} \frac{ R^2 }{(\Delta^\fc_e)^2} \geq n$.
	Then, applying Lemma~\ref{lemma:technical_tool}, we have
	\begin{align*}
		T \leq & \sum_{e \in [n]} \frac{576 R^2 }{(\Delta^\fc_e)^2} \ln \sbr{\frac{2 n^2  }{\delta}  \sum_{e \in [n]} \frac{96 R^2 }{(\Delta^\fc_e)^2} } + n
		\\
		= & O \sbr{ \sum_{e \in [n]} \frac{ R^2 }{(\Delta^\fc_e)^2} \ln \sbr{  \sum_{e \in [n]} \frac{ R^2 n^2  }{(\Delta^\fc_e)^2 \delta} } + n }
		\\
		= & O \sbr{ \sum_{e \in [n]} \frac{ R^2 }{(\Delta^\fc_e)^2} \ln \sbr{  \sum_{e \in [n]} \frac{ R^2 n  }{(\Delta^\fc_e)^2 \delta} }  }
	\end{align*}
\end{proof}

\subsection{Details for the Improved Algorithm $\algbottleneckparallel$}
\label{apx:verification}

In this subsection, we describe algorithm $\algbottleneckparallel$ and its sub-algorithms in details, present the pseudo-code of the offline subroutine $\findewinset$ in Algorithm~\ref{alg:findeminset} and give the proofs of theoretical results.

\subsubsection{Detailed Algorithm Description} \label{apx:verification_algorithm_details}
$\algbottleneckparallel$ simulates multiple sub-algorithms 
$\algbottleneckverify$ (Algorithm~\ref{alg:bottleneck_verify}) with different confidence parameters in parallel.
For $k \in \mathbb{N}$, $\algbottleneckverify_k$ denotes the sub-algorithm $\algbottleneckverify$ with confidence parameter $\delta^V_k=\frac{\delta}{2^{k+1}}$\;
At each timestep $t$, we start or resume sub-algorithms $\algbottleneckverify_k$ such that $t$ is divisible by $2^k$ with only one sample, and then suspend these sub-algorithms.
Such parallel simulation is performed until there exist some sub-algorithm which terminates and returns the answer $M_\textup{out}$, and then we output $M_\textup{out}$ as the answer of $\algbottleneckparallel$.

Algorithm $\algbottleneckverify$ (Algorithm~\ref{alg:bottleneck_verify}) calls Algorithm $\algbottleneckexplore$ (Algorithm~\ref{alg:bottleneck_explore}) as its preparation procedure.
Algorithm $\algbottleneckexplore$ uses a big constant confidence parameter $\kappa$ to guesses an optimal super arm and an advice set $\hat{B}_{\sub}$ which contains the bottleneck base arms for the sub-optimal super arms, and then algorithm $\algbottleneckverify$ verifies the correctness of the answer provided by $\algbottleneckexplore$ with the given  confidence parameter $\delta^V$. 

$\algbottleneckexplore$ first calculates an super arm $M_t$ with the maximum pessimistic bottleneck value, and then uses a subroutine $\findewinset$ (Algorithm~\ref{alg:findeminset}) to find the set of bottleneck base arms $\hat{B}_{\sub,t}$ from all super arms in $\cM \setminus \superset(M_t)$ with respect to the lower reward confidence bound $\underline{\bw}_t$.
Then, we check whether for any base arm $e \in \hat{B}_{\sub,t}$, the optimistic reward of $e$ is lower than a half of its pessimistic reward plus the pessimistic bottleneck value of $M_t$.
If this stopping condition holds, we simply return $M_t$ as the hypothesized optimal super arm and $\hat{B}_{\sub,t}$ as the advice set.
Otherwise, we find the bottleneck $c_t$ from $M_t$ with respect to $\underline{\bw}_t$, and collect the set of base arms $\hat{B}'_{\sub,t}$ which violate the stopping condition from $\hat{B}_{\sub,t}$.
Let $P_t^{E} \overset{\textup{def}}{=} \hat{B}'_{\sub,t} \cup \{ c_t \}$ denote the sampling set of $\algbottleneckexplore$. We plays the base arm with the maximum confidence radius in $P_t^{E}$.

$\algbottleneckverify$ first calculates the super arm $\tilde{M_t}$ with the maximum optimistic bottleneck value in $\cM \setminus \superset(\hat{M}_*)$, and checks whether the pessimistic bottleneck value of $\hat{M}_*$ is higher than the optimistic bottleneck value of $\tilde{M_t}$. If so, we can determine that the answer $\hat{M}_*$ is correct, and simply stop and return $\hat{M}_*$. 
Otherwise, we find the bottleneck $c_t$ from $M_t$ with respect to the lower reward confidence bound $\underline{\bw}_t$, and
collect the set of base arms $F_t$ whose optimistic rewards are higher than the pessimistic bottleneck value of $\hat{M}_*$ from $\hat{B}_{\sub}$. Let $P_t^{V} \overset{\textup{def}}{=} F_t \cup \{ c_t \}$ denote the sampling set of $\algbottleneckverify$. We  samples the base arm with the maximum confidence radius in $P_t^{V}$.

Now, we discuss the skillful subroutine $\findewinset$, which is formally defined as follows.
\begin{definition}[$\findewinset$]\label{def:findewinset}
	We define $\findewinset(\cM, M, \bv)$ as an algorithm that takes decision class $\cM$, super arm $M \in \cM$ and weight vector $\bv \in \mathbb{R}^d$ as inputs and returns a set of base arms $A_{\textup{out}}$ that satisfies (i) for any $M' \in \cM \setminus \superset(M)$, there exists a base arm $e \in A_{\textup{out}} \cap M'$ such that $\bv(e)=\wmin(M',\bv)$, and (ii) for any $e \in A_{\textup{out}}$, there exists a super arm $M' \in \cM \setminus \superset(M)$ such that $e \in M'$ and $\bv(e)=\wmin(M',\bv)$.
\end{definition}

\begin{algorithm}[t]
	\caption{$\findewinset$, offline subroutine of $\algbottleneckexplore$} \label{alg:findeminset}
	\begin{multicols}{2}
		\begin{algorithmic}[1]
			\STATE \textbf{Input:} $\cM$, $M_{\exclude}$, $\bv$ and existence oracle $\existoracle(\cF, e)$: check if there exists a feasible super arm $M \in \cF$ such that $e \in M$, and return $M$ if there exists and $\perp$ otherwise.
			\STATE $A_{\textup{out}} \leftarrow \varnothing$\;
			\FOR{$e \in [n] $}
			\STATE Remove all base arms with  rewards lower than $v(e)$ from $\cM$, and obtain $\cM_{\geq v(e)}$\;
			\label{line:findeminset_cM_geq_v_e}
			\IF{$e \!\in\! M_{\exclude}$ \textup{and} $v(e) \!=\! \wmin(M_{\exclude}, \bv)$}
			\label{line:findeminset_special_case_if}
			\FOR{\textup{each} $e_0 \in M_{\exclude} \setminus \{e\}$} 
			\label{line:findeminset_special_case_start}
			\STATE Remove $e_0$ from $\cM_{\geq v(e)}$, and obtain $\cM_{\geq v(e), -e_0}$
			\STATE \textbf{if} $\existoracle(\cM_{\geq v(e), -e_0}, e ) \neq \perp$, then
			$A_{\textup{out}} \leftarrow A_{\textup{out}} \cup \{e\}$ and \textbf{break}\;
			\ENDFOR
			\label{line:findeminset_special_case_end}
			\ELSIF{$\existoracle(\cM_{\geq v(e)}, e ) \neq \perp$}
			\label{line:findeminset_elseif_exist}
			\STATE $A_{\textup{out}} \leftarrow A_{\textup{out}} \cup \{e\}$\;
			\label{line:findeminset_include_e}
			\ENDIF
			\ENDFOR
			\STATE \textbf{return} $A_{\textup{out}}$\;
		\end{algorithmic}
	\end{multicols}
	\vspace*{-1em}
\end{algorithm}

Algorithm~\ref{alg:findeminset} illustrates the implementation procedure of $\findewinset$.
We access the subroutine $\findewinset$ an efficient existence oracle $\existoracle(\cF, e)$, which returns a feasible super arm that contains base arm $e$ from decision class $\cF$ if there exists, and otherwise returns  $\perp$, i.e., $\existoracle(\cF, e) \in \{M \in \cF : e \in M\}$.
Such efficient oracles exist for a wide family of decision classes. For example, for $s$-$t$ paths, this problem can be reduced to the well-studied $2$-vertex connectivity problem~\cite{two_vertex_connectivity2018}, which is polynomially tractable (see Section~\ref{sec:reduce_2_vertex} for the proof of reduction). For maximum cardinality matchings, we just need to remove $e$ and its two end vertices, and then find a feasible maximum cardinality matching in the remaining graph. For spanning trees, we can just merge the vertices of $e$ and find a feasible spanning tree in the remaining graph, which can also be solved efficiently.

In the subroutine $\findewinset$, we enumerate all the base arms to collect the bottlenecks.
For each enumerated base arm $e$, we first remove all base arms with the rewards lower than $w(e)$ from $\cM$ and obtain a new decision class $\cM_{\geq w(e)}$, in which the super arms only contain the base arms with the rewards at least $w(e)$. Then, we call $\existoracle$ to check whether there exists a feasible super arm $M_e$ that contains $e$ in $\cM_{\geq w(e)}$. If there exist, then we obtain that $e$ is the bottleneck of $M_e$ with respect to $\bv$, i.e., $e \in M_e$ and $w(e)=\wmin(M_e, \bv)$, and add $e$ to the output set $A_{\textup{out}}$. Otherwise, we can determine that $e$ is not the bottleneck for any super arm in $\cM \setminus \superset(M_t)$ with respect to $\bv$. However, for the particular $e$ such that $e \in M_t$ and $w(e)=\wmin(M_t, \bv)$, directly calling $\existoracle$ can obtain the output $M_t$, which should be excluded. We solve this problem by repeatedly removing each base arm in $M_t \setminus \{e\}$ and then calling $\existoracle$ on the new decision classes, which can check whether $e$ is some super arm's bottleneck apart from $\cM \setminus \superset(M_t)$.

\subsubsection{Proof for Algorithm $\algbottleneckparallel$}
Below we give the theoretical results for the proposed algorithms.

For algorithm $\algbottleneckexplore$, define the events 
$
\xi_{0,t}=\left\{ \forall e \in [n],\  | w(e)-\hat{w}_t(e) | < \rad_t(e) \right\}
$
and 
$ 
\xi_0=\bigcap \limits_{t=1}^{\infty} \xi_{0,t}  
$.
Then, similar to Lemma~\ref{lemma:concentration}, we have 
$\Pr[\xi_0] \geq 1-\kappa$.
For algorithm $\algbottleneckverify$, define the events 
$
\xi_{t}=\left\{ \forall e \in [n],\  | w(e)-\hat{w}_t(e) | < \rad_t(e) \right\}
$
and 
$ 
\xi=\bigcap \limits_{t=1}^{\infty} \xi_{t}  
$. Then, applying Lemma~\ref{lemma:concentration}, we have 
$\Pr[\xi] \geq 1-\delta^V$.
Let $M_{\textup{second}}=\argmax_{M \in \cM \setminus \{M_*\}} \wmin(M,w)$ denote the second best super arm.

\begin{lemma}[Correctness of $\algbottleneckexplore$] \label{lemma:bottleneck_explore_correctness}
	For algorithm $\algbottleneckexplore$, assume that event $\xi_0$ occurs. Then, if algorithm $\algbottleneckexplore$  (Algorithm~\ref{alg:bottleneck_explore}) terminates at round $t$, we have that (i) $M_t=M_*$, (ii) for any $M \in  \cM \setminus \superset(M_*)$, there exists a base arm $e \in \hat{B}_{\sub,t} \cap M$ satisfying $w(e) \leq \frac{1}{2}(\wmin(M_*, \bw)+\wmin(M, \bw))$, i.e., $\Delta^\fc_{\wmin(M_*, \bw),e} \geq \frac{1}{2}\Delta^\fc_{M_*,M}$ and (iii) for any $e \in \hat{B}_{\sub,t}$, there exists a sub-optimal super arm $M \in \cM \setminus \superset(M_*)$ such that $e \in M$ and $w(e) \leq \frac{1}{2} ( \wmin(M_*, \bw) + \wmin (M) ) \leq \frac{1}{2} ( \wmin(M_*, \bw) + \wmin (M_{\textup{second}}, \bw) ) $.
\end{lemma}
\begin{proof}
	According to the stop condition (Lines \ref{line:blucb_explore_stop_condition} of Algorithm~\ref{alg:bottleneck_explore}) and the definition of $\findewinset$ (Definition~\ref{def:findewinset}), when algorithm $\algbottleneckexplore$  (Algorithm~\ref{alg:bottleneck_explore}) terminates at round $t$, we have that for any $\cM \setminus \superset(M_t)$, there exists a base arm $e \in \hat{B}_{\sub,t} \cap M$ satisfying 
	$$
	\bar{w}_t(e) \leq  \frac{1}{2} (\wmin(M_t,\underline{\bw}_t)+\underline{w}_t(e)) = \frac{1}{2} ( \wmin(M_t,\underline{\bw}_t)+\wmin(M,\underline{\bw}_t) ) 
	$$
	and thus,
	$$
	w(e) \leq \frac{1}{2} (\wmin(M_t, \bw)+\wmin(M, \bw)).
	$$
	Then, we can obtain $M_t=M_*$. Otherwise, we cannot find any base arm $e \in M_*$ satisfying $w(e) \leq \frac{1}{2} (\wmin(M_t, \bw)+\wmin(M_*, \bw)) < \wmin(M_*, \bw)$, where $M_t$ is a sub-optimal super arm.
	Thus, we have (i) and (ii). 
	
	Now we prove (iii).
	According to the stop condition (Lines \ref{line:blucb_explore_stop_condition} of Algorithm~\ref{alg:bottleneck_explore}), the definition of $\findewinset$ (Definition~\ref{def:findewinset}) and $M_t=M_*$, we have that for any $e \in \hat{B}_{\sub,t}$, there exists a sub-optimal super arm $\cM \setminus \superset(M_*)$ such that $e \in M$ and $\underline{w}_t(e)=\wmin(M,\underline{\bw}_t)$, and thus
	$$
	\bar{w}_t(e) \leq  \frac{1}{2} (\wmin(M_*,\underline{\bw}_t)+\underline{w}_t(e)) = \frac{1}{2} ( \wmin(M_*,\underline{\bw}_t)+\wmin(M,\underline{\bw}_t) ),
	$$
	Then, for any $e \in \hat{B}_{\sub,t}$,
	$$
	w(e) \leq  \frac{1}{2} ( \wmin(M_*, \bw) + \wmin (M) )  \leq  \frac{1}{2} ( \wmin(M_*, \bw) + \wmin (M_{\textup{second}}, \bw) ) , 
	$$
	which completes the proof.
\end{proof}

\begin{lemma} \label{lemma:half_separate_two_base_arms}
	For algorithm $\algbottleneckexplore$, assume that event $\xi_0$ occurs. For any two base arms $e_1, e_2 \in [n]$ s.t. $w(e_1)<w(e_2)$, if $\rad_t(e_1)< \frac{1}{6}\Delta^\fc_{e_2, e_1}$ and $\rad_t(e_2)< \frac{1}{6}\Delta^\fc_{e_2, e_1}$, we have $\bar{w}_t(e_1) < \frac{1}{2}(\underline{w}_t(e_1)+\underline{w}_t(e_2))$.
\end{lemma}
\begin{proof}
	if $\rad_t(e_1)< \frac{1}{6}\Delta^\fc_{e_2, e_1}$ and $\frac{1}{6}\Delta^\fc_{e_2, e_1}$, we have
	\begin{align*}
	\bar{w}_t(e_1) - \frac{1}{2}(\underline{w}_t(e_1)+\underline{w}_t(e_2)) & = \underline{w}_t(e_1) + 2 \rad_t(e_1) - \frac{1}{2}(\underline{w}_t(e_1)+\underline{w}_t(e_2))
	\\
	& = \frac{1}{2} \underline{w}_t(e_1)  - \frac{1}{2} \underline{w}_t(e_2) + 2 \rad_t(e_1)
	\\
	& \leq \frac{1}{2} w_t(e_1) - \frac{1}{2} (w_t(e_2)-2 \rad_t(e_2)) + 2 \rad_t(e_1)
	\\
	& = \frac{1}{2} w_t(e_1) - \frac{1}{2} w_t(e_2) + \rad_t(e_2) + 2 \rad_t(e_1)
	\\
	& < \frac{1}{2} w_t(e_1) - \frac{1}{2} w_t(e_2) + \frac{1}{2}\Delta^\fc_{e_2, e_1}
	\\
	& = 0 .
	\end{align*}
\end{proof}

\begin{lemma}
	\label{lemma:blucb_explore_e_in_M_star}
	For algorithm $\algbottleneckexplore$, assume that event $\xi_0$ occurs. For any $e \in M_*$, if $\rad_t(e)<\frac{\Delta^\fc_e}{12}=\frac{1}{12}(w(e)-\max_{M \neq M_*} \wmin(M, \bw))$, then, base arm $e$ will not be pulled at round $t$, i.e., $p_t \neq e$.
\end{lemma}
\begin{proof}
	Suppose that for some $e \in M_*$, $\rad_t(e)<\frac{\Delta^\fc_e}{12}=\frac{1}{12}(w(e)-\max_{M \neq M_*} \wmin(M, \bw))$ and $p_t = e$. 
	
	According to the selection strategy of $p_t$, we have that for any $i \in P_t^{E}$, $\rad_t(i) \leq \rad_t(e) < \frac{\Delta^\fc_e}{12}$.
	From the definition of $\findewinset$ and $c_t$, we have that for any $i \in P_t^{E}$, there exists a super arm $M \in \cM$ such that $i \in M$ and $\underline{w}_t(i)=\wmin(M, \underline{\bw}_t)$.
	
	
	Case (i): Suppose that $e$ is selected from a sub-optimal super arm $M$, i.e., $e \in M$ and  $\underline{w}_t(e)=\wmin(M, \underline{\bw}_t)$. Then, using $\rad_t(e) < \frac{\Delta^\fc_e}{12} < \frac{\Delta^\fc_e}{2}$, we have
	\begin{align*}
	\underline{w}_t(e) \geq & w(e) - 2\rad_t(e)
	\\
	> & w(e) - ( w(e)- \wmin(M_{\textup{second}}, \bw) )
	\\
	= & \wmin(M_{\textup{second}}, \bw)
	\\ 
	\geq & \wmin(M,\bw)
	\\
	\geq & \wmin(M, \underline{\bw}_t)
	\end{align*}
	which gives a contradiction.

	Case (ii): Suppose that $e$ is selected from $M_*$, i.e., $\underline{w}_t(e)=\wmin(M_*, \underline{\bw}_t)$. 
	We can obtain $M_*=M_t$. Otherwise,
	\begin{align*}
	\wmin(M_*, \underline{\bw}_t) = & \underline{w}_t(e) 
	\\
	\geq & w(e) - 2\rad_t(e)
	\\
	> & w(e) - ( w(e)- \wmin(M_{\textup{second}}, \bw) )
	\\
	= & \wmin(M_{\textup{second}}, \bw)
	\\ 
	\geq & \wmin(M_t,\bw)
	\\
	\geq & \wmin(M_t, \underline{\bw}_t)
	\end{align*}
	Thus, We have $M_*=M_t$.
	From the definition of $\findewinset$, for any $e' \in \hat{B}'_{\sub,t}$, there exists a super arm $M' \in \cM \setminus \{M_*\}$ such that $e' \in M'$ and $\underline{w}_t(e')=\wmin(M', \underline{\bw}_t)$. If $w(e') \geq \frac{1}{2}(w(e)+\wmin(M',\bw)) $, using $\rad_t(e') < \frac{\Delta^\fc_e}{12} < \frac{\Delta^\fc_e}{4}$, we have
	\begin{align*}
	\underline{w}_t(e') \geq & w(e') - 2\rad_t(e')
	\\
	> & \frac{1}{2}(w(e)+\wmin(M',\bw)) - \frac{1}{2}( w(e)-M_{\textup{second}} )
	\\
	\geq & \frac{1}{2}(w(e)+\wmin(M',\bw)) - \frac{1}{2}( w(e)-\wmin(M',\bw) )
	\\ 
	= & \wmin(M',\bw),
	\end{align*}
	which gives a contradiction.
	
	If $w(e') < \frac{1}{2}(w(e)+\wmin(M',\bw)) $,
	we have 
	\begin{align*}
	\Delta^\fc_{e, e'} = & w(e)-w(e')
	\\
	> & w(e) -  \frac{1}{2}( w(e)+\wmin(M',\bw) )
	\\
	= & \frac{1}{2} ( w(e) - \wmin(M',\bw) )
	\end{align*} 
	Since $\rad_t(e) < \frac{\Delta^\fc_e}{12}=\frac{1}{12}(w(e)- \wmin (M_{\textup{second}}) ) \leq \frac{1}{12}(w(e)-  \wmin(M',\bw) ) \leq \frac{1}{6} \Delta^\fc_{e, e'}$ and $\rad_t(e') \leq \rad_t(e) < \frac{1}{6} \Delta^\fc_{e, e'}$, according to Lemma~\ref{lemma:half_separate_two_base_arms}, we have 
	\begin{align*}
	\bar{w}_t(e') < & \frac{1}{2}(\underline{w}_t(e)+\underline{w}_t(e'))
	\\
	= & \frac{1}{2}(\wmin(M_t, \underline{\bw}_t)+\underline{w}_t(e')) ,
	\end{align*}
	which contradicts the definition of $\hat{B}'_{\sub,t}$.
\end{proof}

\begin{lemma}
	\label{lemma:blucb_explore_e_notin_M_star_upper}
	For algorithm $\algbottleneckexplore$, assume that event $\xi_0$ occurs. For any $e \notin M_*, w(e) \geq \wmin(M_*, \bw)$, if $\rad_t(e)<\frac{\Delta^\fc_e}{2}=\frac{1}{2}(w(e)-\max_{M \in \cM: e\in M} \wmin(M, \bw))$, then, base arm $e$ will not be pulled at round $t$, i.e., $p_t \neq e$.
\end{lemma}
\begin{proof}
	Suppose that for some  $e \notin M_*, w(e) \geq \wmin(M_*, \bw)$, $\rad_t(e)<\frac{\Delta^\fc_e}{2}=\frac{1}{2}(w(e)-\max_{M \in \cM: e\in M} \wmin(M, \bw))$ and $p_t = e$. 
	
	According to the selection strategy of $p_t$, we have that for any $i \in P_t^{E}$, $\rad_t(i) \leq \rad_t(e) < \frac{\Delta^\fc_e}{12}$.
	From the definition of $\findewinset$ and $c_t$, we have that for any $i \in P_t^{E}$, there exists a super arm $M \in \cM$ such that $i \in M$ and $\underline{w}_t(i)=\wmin(M, \underline{\bw}_t)$.
	
	Since $e \notin M_*$, $e$ is selected from a sub-optimal super arm $M'$, i.e., $e \in M'$ and $\underline{w}_t(e)=\wmin(M', \underline{\bw}_t)$. Then, 
	\begin{align*}
	\underline{w}_t(e) \geq & w(e) - 2\rad_t(e)
	\\
	> & w(e) - (w(e)-\max_{M \in \cM: e\in M} \wmin(M, \bw))
	\\
	= & \max_{M \in \cM: e\in M} \wmin(M, \bw)
	\\ 
	\geq & \wmin(M',\bw)
	\\
	\geq & \wmin(M', \underline{\bw}_t)
	\end{align*}
	which contradicts $\underline{w}_t(e)=\wmin(M', \underline{\bw}_t)$.
\end{proof}

\begin{lemma} 
	\label{lemma:blucb_explore_e_notin_M_star_lower}
	For algorithm $\algbottleneckexplore$, assume that event $\xi_0$ occurs. For any $e \notin M_*, w(e) < \wmin(M_*, \bw)$, if $\rad_t(e)<\frac{\Delta^\fc_e}{12}=\frac{1}{12}(\wmin(M_*, \bw)-\max_{M \in \cM: e\in M} \wmin(M, \bw))$, then, base arm $e$ will not be pulled at round $t$, i.e., $p_t \neq e$.
\end{lemma}
\begin{proof}
	Suppose that for some  $e \notin M_*, w(e) < \wmin(M_*, \bw)$, $\rad_t(e)<\frac{\Delta^\fc_e}{12}=\frac{1}{12}(\wmin(M_*, \bw)-\max_{M \in \cM: e\in M} \wmin(M, \bw))$ and $p_t = e$. 
	
	According to the selection strategy of $p_t$, we have that for any $i \in P_t^{E}$, $\rad_t(i) \leq \rad_t(e) < \frac{\Delta^\fc_e}{12}$.
	From the definition of $\findewinset$ and $c_t$, we have that for any $i \in P_t^{E}$, there exists a super arm $M \in \cM$ such that $i \in M$ and $\underline{w}_t(i)=\wmin(M, \underline{\bw}_t)$.
	Thus, there exists a sub-optimal super arm $M' \in \cM$ such that $e \in M'$ and $\underline{w}_t(e)=\wmin(M', \underline{\bw}_t)$.
	
	Case (i) Suppose that $w(e) \geq \frac{1}{2}(\wmin(M_*,\bw)+\wmin(M',\bw)) $. Using $\rad_t(e) < \frac{\Delta^\fc_e}{12}< \frac{\Delta^\fc_e}{4}$, we have
	\begin{align*}
	\underline{w}_t(e) \geq & w(e) - 2\rad_t(e)
	\\
	> & \frac{1}{2}(\wmin(M_*,\bw)+\wmin(M',\bw)) - \frac{1}{2}( \wmin(M_*, \bw)-\max_{M \in \cM: e\in M} \wmin(M, \bw) )
	\\
	= & \frac{1}{2}\wmin(M',\bw) + \frac{1}{2} \max_{M \in \cM: e\in M} \wmin(M, \bw)
	\\
	\geq & \wmin(M',\bw) ,
	\end{align*}
	which contradicts $\underline{w}_t(e)=\wmin(M', \underline{\bw}_t)$.
	
	Case (ii) Suppose that $w(e) < \frac{1}{2}(\wmin(M_*,\bw)+\wmin(M',\bw))$. 
	According to the definition of $\findewinset$ (Definition~\ref{def:findewinset}) and $c_t$, there exists a base arm $\tilde{e} \in \hat{B}_{\sub,t} \cap \{c_t\}$ satisfying $\tilde{e} \in M_*$ and  $\underline{w}_t(\tilde{e})=\wmin(M_*, \underline{\bw}_t)$. 
	
	First, we prove $\tilde{e} \in P_t^{E}$ and thus $\rad_t(\tilde{e}) \leq \rad_t(e) < \frac{\Delta^\fc_e}{12}$. 
	If $M_*=M_t$, then $\tilde{e}=c_t$ and the claim holds.
	If $M_* \neq M_t$, then $\tilde{e} \in \hat{B}_{\sub,t}$. We can obtain that $\tilde{e}$ will be put into $\hat{B}'_{\sub,t} \subseteq P_t^{E}$. Otherwise, we have
	$$
	w(\tilde{e}) \leq \bar{w}_t(\tilde{e}) \leq  \frac{1}{2} (\wmin(M_t,\underline{\bw}_t)+\underline{w}_t(\tilde{e})) \leq \frac{1}{2} (\wmin(M_t, \bw)+w(\tilde{e})).
	$$
	Since $w(\tilde{e}) \geq \wmin(M_*, \bw)$ and $\wmin(M_t, \bw) < \wmin(M_*, \bw)$, the above inequality cannot hold.
	Thus, we obtain that $\tilde{e}$ will be put into $\hat{B}'_{\sub,t} \subseteq P_t^{E}$ and thus $\rad_t(\tilde{e}) \leq \rad_t(e) < \frac{\Delta^\fc_e}{12}$. 
	
	Next, we discuss the following two cases: (a) $e=c_t$ and (b) $e \neq c_t$.
	
	(a) if $e=c_t$, then $M_t \neq M_*$. Using $\rad_t(\tilde{e}) \leq \rad_t(e) < \frac{\Delta^\fc_e}{12}< \frac{\Delta^\fc_e}{2}$, we have
	\begin{align*}
	\wmin(M_*, \underline{\bw}_t) = & \underline{w}_t(\tilde{e}) 
	\\
	\geq & w(e) - 2\rad_t(e)
	\\
	> & w(e) - (\wmin(M_*, \bw)-\max_{M \in \cM: e\in M} \wmin(M, \bw))
	\\
	\geq & \wmin(M_*,\bw) - (\wmin(M_*,\bw)- \wmin(M_t,\bw))
	\\ 
	\geq & \wmin(M_t,\bw)
	\\
	\geq & \wmin(M_t, \underline{\bw}_t),
	\end{align*}
	which contradicts the definition of $M_t$.
	
	(b) if $e \neq c_t$, i.e., $e \in \hat{B}'_{\sub,t}$ we have
	\begin{align*}
	\Delta^\fc_{\tilde{e}, e} = & w(\tilde{e})-w(e)
	\\
	> & \wmin(M_*, \bw) -  \frac{1}{2}( \wmin(M_*,\bw)+\wmin(M',\bw) )
	\\
	= & \frac{1}{2} ( \wmin(M_*, \bw) - \wmin(M',\bw) )
	\end{align*} 
	Since $\rad_t(e) < \frac{\Delta^\fc_e}{12}=\frac{1}{12}(\wmin(M_*, \bw)-\max_{M \in \cM: e\in M} \wmin(M, \bw) ) \leq \frac{1}{12}(\wmin(M_*, \bw)- \wmin(M',\bw) ) < \frac{1}{6} \Delta^\fc_{\tilde{e}, e}$ and $\rad_t(\tilde{e}) \leq \rad_t(e) < \frac{1}{6} \Delta^\fc_{\tilde{e}, e}$, according to Lemma~\ref{lemma:half_separate_two_base_arms}, we have 
	\begin{align*}
	\bar{w}_t(e) < & \frac{1}{2}(\underline{w}_t(\tilde{e})+\underline{w}_t(e))
	\\
	= & \frac{1}{2}(\wmin(M_*, \underline{\bw}_t)+\underline{w}_t(e)) 
	\\
	\leq & \frac{1}{2}(\wmin(M_t, \underline{\bw}_t)+\underline{w}_t(e')) 
	\end{align*}    
	which contradicts the definition of $\hat{B}'_{\sub,t}$.
\end{proof}

\begin{theorem}[Sample Complexity of $\algbottleneckexplore$]
	\label{thm:blucb_explore_ub}
	With probability at least $1-\kappa$, the $\algbottleneckexplore$ algorithm (Algorithm~\ref{alg:bottleneck_explore}) will return $M_*$ with sample complexity 
	$$
	O \sbr{ \sum_{e \in [n]} \frac{ R^2 }{(\Delta^\fc_e)^2} \ln \sbr{  \sum_{e \in [n]} \frac{ R^2 n  }{(\Delta^\fc_e)^2 \kappa} }  } .
	$$
\end{theorem}
\begin{proof}
	For any $e \in [n]$, let $T(e)$ denote the number of samples for base arm $e$, and $t_e$ denote the last timestep at which $e$ is pulled. Then, we have $T_{t_e}=T(e)-1$. Let $T$ denote the total number of samples.
	According to Lemmas~\ref{lemma:blucb_explore_e_in_M_star}-\ref{lemma:blucb_explore_e_notin_M_star_lower}, we have
	\begin{align*}
	R \sqrt{ \frac{2 \ln (\frac{4 n t_e^3}{\kappa})}{T(e)-1} } \geq \frac{1}{12} \Delta^\fc_e
	\end{align*}
	Thus, we obtain
	\begin{align*}
	T(e) \leq \frac{288 R^2 }{(\Delta^\fc_e)^2}\ln \sbr{\frac{4 n t_e^3}{\kappa}} +1 \leq \frac{288 R^2 }{(\Delta^\fc_e)^2}\ln \sbr{\frac{4 n T^3}{\kappa}} +1
	\end{align*}
	Summing over $e \in [n]$, we have
	\begin{align*}
	T \leq \sum_{e \in [n]} \frac{288 R^2 }{(\Delta^\fc_e)^2} \ln \sbr{\frac{4 n T^3 }{\kappa}   } + n \leq \sum_{e \in [n]} \frac{864 R^2 }{(\Delta^\fc_e)^2} \ln \sbr{\frac{2 n T }{\kappa}   } + n,
	\end{align*}
	where $\sum_{e \in [n]} \frac{ R^2 }{(\Delta^\fc_e)^2} \geq n$.
	Then, applying Lemma~\ref{lemma:technical_tool}, we have
	\begin{align*}
	T \leq & \sum_{e \in [n]} \frac{5184 R^2 }{(\Delta^\fc_e)^2} \ln \sbr{\frac{2 n^2  }{\kappa}  \sum_{e \in [n]} \frac{864 R^2 }{(\Delta^\fc_e)^2} } + n
	\\
	= & O \sbr{ \sum_{e \in [n]} \frac{ R^2 }{(\Delta^\fc_e)^2} \ln \sbr{  \sum_{e \in [n]} \frac{ R^2 n^2  }{(\Delta^\fc_e)^2 \kappa} } + n }
	\\
	= & O \sbr{ \sum_{e \in [n]} \frac{ R^2 }{(\Delta^\fc_e)^2} \ln \sbr{  \sum_{e \in [n]} \frac{ R^2 n  }{(\Delta^\fc_e)^2 \kappa} }  }.
	\end{align*}
\end{proof}

\begin{lemma} \label{lemma:bottleneck_verify_e_in_M_star}
	For algorithm $\algbottleneckverify$, assume that event $\xi_0 \cap \xi$ occurs. For any $e \in M_*$, if $\rad_t(e)<\frac{\Delta^\fc_e}{8}=\frac{1}{8}(w(e)-\max_{M \neq M_*} \wmin(M, \bw))$, then, base arm $e$ will not be pulled at round $t$, i.e., $p_t \neq e$.
\end{lemma}
\begin{proof}
	Suppose that for some $e \in M_*$, $\rad_t(e)<\frac{\Delta^\fc_e}{8}=\frac{1}{8}(w(e)-\max_{M \neq M_*} \wmin(M, \bw))$ and $p_t = e$. 
	
	Since event $\xi_0$ occurs, according to Lemma~\ref{lemma:bottleneck_explore_correctness}, we have that (i) $\hat{M}_*=M_*$, (ii) for any $M \neq M_*$, there exists a base arm $e \in \hat{B}_{\sub} \cap M$ satisfying $w(e) \leq \frac{1}{2}(\wmin(M_*, \bw)+\wmin(M, \bw))$, i.e., $\Delta^\fc_{\wmin(M_*, \bw),e} \geq \frac{1}{2}\Delta^\fc_{M_*,M}$ and (iii) for any $i \in \hat{B}_{\sub}$, there exists a sub-optimal super arm $M \neq M_*$ such that $i \in M$ and $w(i) \leq \frac{1}{2} ( \wmin(M_*, \bw) + \wmin (M) ) \leq \frac{1}{2} ( \wmin(M_*, \bw) + \wmin (M_{\textup{second}}, \bw) ) $.
	According to the selection strategy of $p_t$, we have that for any $i \in P_t^{V}$, $\rad_t(i) \leq \rad_t(e) < \frac{\Delta^\fc_e}{8}$.
	
	First, since for any $i \in \hat{B}_{\sub}$, $w(i) \leq \frac{1}{2} ( \wmin(M_*, \bw) + \wmin(M_{\textup{second}}, \bw) ) < \wmin(M_*, \bw)$ and $w(e) \geq \wmin(M_*, \bw)$, we can obtain that $e=c_t$. 
	
	Then, for any $i \in F_t$, we have
	\begin{align*}
	\Delta^\fc_{e,i} = & w(e)-w(i)
	\\
	\geq & w(e)-\frac{1}{2} ( \wmin(M_*, \bw) + \wmin(M_{\textup{second}}, \bw) )
	\\
	\geq & w(e)-\frac{1}{2} ( w(e) + \wmin(M_{\textup{second}}, \bw) )
	\\
	= & \frac{1}{2} ( w(e) - \wmin(M_{\textup{second}}, \bw) )
	\\
	= & \frac{1}{2} \Delta^\fc_e
	\end{align*}
	and
	\begin{align*}
	\underline{w}_t(c_t)-\bar{w}_t(i) = & \underline{w}_t(e)-\bar{w}_t(i) 
	\\
	\geq & w(e)-2\rad_t(e)-(w(i)+2\rad_t(i))
	\\
	= & \Delta^\fc_{e,i}-2\rad_t(e)-2\rad_t(i)
	\\
	> & \frac{1}{2} \Delta^\fc_e - \frac{\Delta^\fc_e}{4} - \frac{\Delta^\fc_e}{4}
	\\
	= & 0,
	\end{align*}
	which implies $F_t=\varnothing$. Thus, for any $i \in \hat{B}_{\sub}$, $\underline{w}_t(c_t) \geq \bar{w}_t(i)$.
	
	According to Lemma~\ref{lemma:bottleneck_explore_correctness}(ii), for any $M' \neq \hat{M}_*$, there exists a base arm $e' \in \hat{B}_{\sub} \cap M'$, we have
	\begin{align*}
	\wmin(\hat{M}_*, \underline{\bw}_t) = & \underline{w}_t(c_t)
	\\
	\geq & \bar{w}_t(e')
	\\
	\geq & \wmin(M', \bar{\bw}_t)
	\end{align*}
	which implies that algorithm~\ref{alg:bottleneck_verify} has already stopped.
\end{proof}

\begin{lemma} \label{lemma:bottleneck_verify_e_notin_M_star}
	For algorithm $\algbottleneckverify$, assume that event $\xi_0 \cap \xi$ occurs. For any $e \notin M_*, w(e) < \wmin(M_*, \bw)$, if $\rad_t(e)<\frac{\Delta^\fc_e}{8}=\frac{1}{8}(\wmin(M_*, \bw)-\max_{M \in \cM: e\in M} \wmin(M, \bw))$, then, base arm $e$ will not be pulled at round $t$, i.e., $p_t \neq e$.
\end{lemma}
\begin{proof}
	Suppose that for some $e \notin M_*, w(e) < \wmin(M_*, \bw)$, $\rad_t(e)<\frac{\Delta^\fc_e}{8}=\frac{1}{8}(\wmin(M_*, \bw)-\max_{M \in \cM: e\in M} \wmin(M, \bw))$ and $p_t = e$. 
	
	Since event $\xi_0$ occurs, according to Lemma~\ref{lemma:bottleneck_explore_correctness}, we have that (i) $\hat{M}_*=M_*$, (ii) for any $M \neq M_*$, there exists a base arm $e \in \hat{B}_{\sub} \cap M$ satisfying $w(e) \leq \frac{1}{2}(\wmin(M_*, \bw)+\wmin(M, \bw))$, i.e., $\Delta^\fc_{\wmin(M_*, \bw),e} \geq \frac{1}{2}\Delta^\fc_{M_*,M}$ and (iii) for any $i \in \hat{B}_{\sub}$, there exists a sub-optimal super arm $M \neq M_*$ such that $i \in M$ and $w(i) \leq \frac{1}{2} ( \wmin(M_*, \bw) + \wmin (M) ) \leq \frac{1}{2} ( \wmin(M_*, \bw) + \wmin (M_{\textup{second}}, \bw) ) $.
	According to the selection strategy of $p_t$, we have that for any $i \in P_t^{V}$, $\rad_t(i) \leq \rad_t(e) < \frac{\Delta^\fc_e}{8}$.
	
	Since $e \notin M_*=\hat{M}_*$, we have that $e \in F_t$ and there exists a sub-optimal super arm $M'$ such that $e \in M'$ and $w(e) \leq \frac{1}{2} ( \wmin(M_*, \bw) + \wmin (M') )$.  Then, we have
	\begin{align*}
	\underline{w}_t(c_t)-\bar{w}_t(e) \geq & w(c_t)-2\rad_t(c_t)-(w(e)+2\rad_t(e))
	\\
	\geq & \wmin(M_*, \bw)-\frac{1}{2}(\wmin(M_*, \bw)+\wmin(M', \bw))-2\rad_t(c_t)-2\rad_t(e)
	\\
	> & \frac{1}{2} \Delta^\fc_{M_*, M'} - \frac{\Delta^\fc_e}{4} - \frac{\Delta^\fc_e}{4}
	\\
	= & \frac{1}{2} \Delta^\fc_{M_*, M'} - \frac{1}{2} (\wmin(M_*, \bw)-\max_{M \in \cM: e\in M} \wmin(M, \bw))
	\\
	\geq & \frac{1}{2} \Delta^\fc_{M_*, M'} - \frac{1}{2} (\wmin(M_*, \bw)-\wmin(M', \bw))
	\\
	= & 0,
	\end{align*}
	which contradicts $e \in F_t$.
\end{proof}

Recall that $B = \{e \mid e \notin M_*, w(e) < \wmin(M_*, \bw) \}$ and $B^c = \{e \mid e \notin M_*, w(e) \geq \wmin(M_*, \bw) \}$.

\begin{theorem}[Sample Complexity of $\algbottleneckverify$]
	\label{thm:blucb_verify_ub}
	With probability at least $1-\kappa-\delta^V$, the $\algbottleneckverify$ algorithm (Algorithm~\ref{alg:bottleneck_verify}) will return $M_*$ with sample complexity 
	$$
	O \sbr{ \sum_{e \in B^c} \frac{ R^2 }{(\Delta^\fc_e)^2} \ln \sbr{  \sum_{e \in B^c} \frac{ R^2 n  }{(\Delta^\fc_e)^2 } } + \sum_{e \in M_* \cup B} \frac{ R^2 }{(\Delta^\fc_e)^2} \ln \sbr{  \sum_{e \in M_* \cup B} \frac{ R^2 n  }{(\Delta^\fc_e)^2 \delta^V } } } .
	$$
\end{theorem}
\begin{proof}
	Assume that event $\xi_0 \cap \xi$ occurs. $\Pr[\xi_0 \cap \xi] \geq 1-\kappa-\delta^V$.
	
	First, we prove the correctness.
	According to Lemma~\ref{lemma:bottleneck_explore_correctness}, the hypothesized $\hat{M}_*$ outputted by the preparation procedure $\algbottleneckexplore$ is exactly the optimal super arm, and thus if $\algbottleneckverify$ terminates, it returns the correct answer.
	
	Next, we prove the sample complexity upper bound.
	According to Theorem~\ref{thm:blucb_explore_ub},  the preparation procedure $\algbottleneckexplore$ costs sample complexity  
	$$
	O \sbr{ \sum_{e \in [n]} \frac{ R^2 }{(\Delta^\fc_e)^2} \ln \sbr{  \sum_{e \in [n]} \frac{ R^2 n  }{(\Delta^\fc_e)^2 \kappa} }  } .
	$$
	Then, we bound the sample complexity of the following verification part. Following the analysis procedure of Theorem~\ref{thm:blucb_explore_ub} with Lemmas~\ref{lemma:bottleneck_verify_e_in_M_star},\ref{lemma:bottleneck_verify_e_notin_M_star}, we can obtain that the verification part cost sample complexity  
	$$
	\sum_{e \in M_* \cup B} \frac{ R^2 }{(\Delta^\fc_e)^2} \ln \sbr{  \sum_{e \in M_* \cup B} \frac{ R^2 n  }{(\Delta^\fc_e)^2 \delta^V } } .
	$$
	Combining both parts, we obtain that the sample complexity is bounded by 
	\begin{align*}
	& O \sbr{ \sum_{e \in B^c} \frac{ R^2 }{(\Delta^\fc_e)^2} \ln \sbr{  \sum_{e \in B^c} \frac{ R^2 n  }{(\Delta^\fc_e)^2 \kappa} } + \sum_{e \in M_* \cup B} \frac{ R^2 }{(\Delta^\fc_e)^2} \ln \sbr{  \sum_{e \in M_* \cup B} \frac{ R^2 n  }{(\Delta^\fc_e)^2 \delta^V } } } 
	\\
	= & O \sbr{ \sum_{e \in B^c} \frac{ R^2 }{(\Delta^\fc_e)^2} \ln \sbr{  \sum_{e \in B^c} \frac{ R^2 n  }{(\Delta^\fc_e)^2 } } + \sum_{e \in M_* \cup B} \frac{ R^2 }{(\Delta^\fc_e)^2} \ln \sbr{  \sum_{e \in M_* \cup B} \frac{ R^2 n  }{(\Delta^\fc_e)^2 \delta^V } } }.
	\end{align*}
\end{proof}

\begin{lemma}[Correctness of $\algbottleneckverify$]
	\label{lemma:blucb_verify_correctness}
	With probability at least $1-\delta$, if algorithm $\algbottleneckverify$  (Algorithm~\ref{alg:bottleneck_verify}) terminates, it returns the optimal super arm $M_*$.
\end{lemma}
\begin{proof}
	Assume that event $\xi$ occurs, where $\Pr[\xi] \geq 1-\delta^V$.
	If algorithm $\algbottleneckverify$ terminates by returning $\hat{M}_*$, we have that for any $M \in \cM \setminus \superset(\hat{M}_*)$,
	$$
	\wmin(\hat{M}_*, \bw) \geq \wmin(\hat{M}_*, \underline{\bw}_t) \geq \wmin(M, \bar{\bw}_t) \geq \wmin(M, \bw) .
	$$
	For any $M \in \superset(\hat{M}_*)$, according to the property of the bottleneck reward function,
	we have
	$$
	\wmin(\hat{M}_*, \bw)  \geq \wmin(M, \bw) .
	$$
	Thus, we have $\wmin(\hat{M}_*, \bw)  \geq \wmin(M, \bw)$ for any $M \neq \hat{M}_*$ and according to the unique assumption of $M_*$, we obtain $\hat{M}_*=M_*$. In other words, algorithm $\algbottleneckverify$  (Algorithm~\ref{alg:bottleneck_verify}) will never return a wrong answer.
\end{proof}

Now, we prove Theorem~\ref{thm:bottleneck_verify_ub}.
\begin{proof}
	Using Theorem~\ref{thm:blucb_verify_ub}, Lemma~\ref{lemma:blucb_verify_correctness} and Lemma 4.8 in \cite{ChenLJ_Nearly_Optimal_Sampling17}, we can obtain this theorem.
\end{proof}

\subsection{PAC Learning}
\label{apx:PAC}

In this subsection, we further study the fixed-confidence CPE-B problem in the PAC learning setting, where the learner's objective is to identify a super arm $M_{\textup{pac}}$ such that  $\wmin(M_{\textup{pac}}, \bw) \geq \opt-\varepsilon$, and the uniqueness assumption of the optimal super arm is dropped. 
We propose two algorithms $\algbottleneckpac$ and $\algbottleneckparallelpac$ for the PAC learning setting, based on  $\algbottleneck$ and $\algbottleneckparallel$, respectively.

The PAC algorithms and their theoretical guarantees do not require the uniqueness assumption of the optimal super arm. Compared to the PAC lower bound, both proposed PAC algorithms achieve the optimal sample complexity for some family of instances. 
Similar to the exact case, when $\delta$ is small enough, the dominant term of sample complexity for algorithm $\algbottleneckparallelpac$ does not depend on the reward gaps of unnecessary base arms, and thus $\algbottleneckparallelpac$ achieves better theoretical guarantee than $\algbottleneckpac$ and matches the lower bound for a broader family of instances.

\subsubsection{Algorithm $\algbottleneckpac$}

$\algbottleneckpac$ simply replaces the stopping condition of $\algbottleneck$ (Line~\ref{line:blucb_stop} in Algorithm~\ref{alg:bottleneck}) with $  \wmin(\tilde{M}_t, \bar{\bw_t})-\wmin(M_t, \underline{\bw}_t) \leq \varepsilon $ to allow an $\varepsilon$ deviation between the returned answer and the optimal one.  The sample complexity of algorithm $\algbottleneckpac$ is given as follows.

\begin{theorem}[Fixed-confidence Upper Bound for PAC]
	\label{thm:bottleneck_pac_ub}
	With probability at least $1-\delta$, the $\algbottleneckpac$ algorithm will return $M_{\textup{out}}$ such that  $\wmin( M_{\textup{out}} , \bw ) \geq \wmin( M_* , \bw ) - \varepsilon $, with sample complexity 
	$$
	O  \sbr{  \sum_{e \in [n]}  \frac{ R^2 }{\max\{(\Delta^\fc_e)^2, \varepsilon^2 \}} \ln  \sbr{  \sum_{e \in [n]}  \frac{ R^2 n  }{\max\{(\Delta^\fc_e)^2, \varepsilon^2 \} \delta}  }  }  .
	$$
\end{theorem}


Now we prove the sample complexity of algorithm $\algbottleneckpac$ (Theorem~\ref{thm:bottleneck_pac_ub}). 
\begin{proof}
	First, we prove the correctness.
	When the stop condition of $\algbottleneckpac$ is satisfied, we have that for any $M \neq M_t$,
	$$
	\wmin(M, \bw)-\wmin(M_t, \bw) \leq \wmin(M, \bar{\bw}_t)-\wmin(M_t, \underline{\bw}_t) \leq \wmin(\tilde{M}_t, \bar{\bw}_t)-\wmin(M_t, \underline{\bw}_t) \leq \varepsilon .
	$$
	If $M_t=M_*$, then the correctness holds.
	If $M_t \neq M_*$, the returned super arm $M_t$ satisfies
	$$
	\wmin(M_*, \bw)-\wmin(M_t, \bw) \leq \varepsilon ,
	$$
	which guarantees the correctness.
	
	Next, we prove the sample complexity.
	When inheriting the proof of Theorem~\ref{thm:bottleneck_baseline_ub} for the baseline algorithm $\algbottleneck$, to prove Theorem~\ref{thm:bottleneck_pac_ub} for PAC leaning, it suffices to prove that for any $e \in [n]$, if $\rad_t(e)<\frac{\varepsilon}{2}$, base arm $e$ will not be pulled at round $t$, i.e., $p_t \neq e$.
	
	Suppose that $\rad_t(e)<\frac{\varepsilon}{2}$ and $p_t = e$. According to the selection strategy of $p_t$, we have $\rad_t(c_t)<\frac{\varepsilon}{2}$ and $\rad_t(d_t)<\frac{\varepsilon}{2}$.
	According to the definition of $d_t$, we have 
	\begin{align*}
	\wmin(\tilde{M}_t, \bar{\bw}_t )-\wmin(M_t, \underline{\bw}_t) \leq & \bar{w}(d_t)-\wmin(\tilde{M}_t, \underline{\bw}_t)
	\\
	= & \bar{w}(d_t)- \underline{w}(d_t)
	\\
	= & 2 \rad_t(d_t)
	\\
	< & \varepsilon ,
	\end{align*}
	which contradicts the stop condition.
\end{proof}

\subsubsection{Algorithm $\algbottleneckparallelpac$}
$\algbottleneckparallelpac$ is obtained by simply replacing the stopping condition of $\algbottleneckverify$ (Line~\ref{line:bottleneck_verify_stop} in Algorithm~\ref{alg:bottleneck_verify}) with $  \wmin(\tilde{M}_t, \bar{\bw}_t)-\wmin(M_t, \underline{\bw}_t) \leq \varepsilon $.

Theorem~\ref{thm:bottleneck_verify_pac_ub} presents the sample complexity of algorithm $\algbottleneckparallelpac$.
\begin{theorem}[Improved Fixed-confidence Upper Bound for PAC]
	\label{thm:bottleneck_verify_pac_ub}
	For any $\delta < 0.01$, with probability at least $1-\delta$, algorithm $\algbottleneckparallelpac$  returns $M_*$ and takes the expected sample complexity 
	\begin{align*}
		O \sbr{   \sum_{e \in M_* \cup N} \frac{ R^2 }{ \max\{(\Delta^\fc_e)^2, \varepsilon^2 \} } \ln \sbr{  \sum_{e \in M_* \cup N} \frac{ R^2 n  }{\max\{(\Delta^\fc_e)^2, \varepsilon^2 \} \delta } } + \sum_{e \in \tilde{N}} \frac{ R^2 }{(\Delta^\fc_e)^2} \ln \sbr{  \sum_{e \in \tilde{N}} \frac{ R^2 n  }{(\Delta^\fc_e)^2 } } }  .
	\end{align*}
\end{theorem}

\begin{proof}
	First, we prove the correctness.
	When the stop condition of $\algbottleneckverify$ is satisfied, we have that for any $M \neq \hat{M}_*$,
	$$
	\wmin(M, \bw)-\wmin(\hat{M}_*, \bw) \leq \wmin(M, \bar{\bw}_t)-\wmin(\hat{M}_*, \underline{\bw}_t) \leq \wmin(\tilde{M}_t, \bar{\bw}_t)-\wmin(\hat{M}_*, \underline{\bw}_t) \leq \varepsilon .
	$$
	If $\hat{M}_*=M_*$, then the correctness holds.
	If $\hat{M}_* \neq M_*$, the returned super arm $\hat{M}_*$ satisfies
	$$
	\wmin(M_*, \bw)-\wmin(\hat{M}_*, \bw) \leq \varepsilon ,
	$$
	which guarantees the correctness.
	
	Next, we prove the sample complexity.
	We inherit the proofs of Theorems~\ref{thm:bottleneck_verify_ub},\ref{thm:blucb_verify_ub}. Then, to prove Theorem~\ref{thm:bottleneck_verify_pac_ub} for PAC leaning, it suffices to prove that conditioning on $\xi_0 \cap \xi$, for any $e \in [n]$, if $\rad_t(e)<\frac{\varepsilon}{4}$, base arm $e$ will not be pulled at round $t$, i.e., $p_t \neq e$.
	
	Suppose that $\rad_t(e)<\frac{\varepsilon}{4}$ and $p_t = e$. According to the selection strategy of $p_t$, we have $\rad_t(c_t)<\frac{\varepsilon}{4}$ and for any $e \in F_t$ $\rad_t(e)<\frac{\varepsilon}{4}$.
	Using $F_t \subseteq \hat{B}_{\sub}$ and the definition of $\hat{B}_{\sub}$, we have that for any $e \in F_t$
	\begin{align*}
	\bar{w}(e)-\wmin(\hat{M}_*, \underline{\bw}_t) 
	\leq & w(e) + 2 \rad_t(e) -( w(c_t) - 2 \rad_t(c_t))
	\\
	< & \wmin(\hat{M}_*, \bw) + 2 \rad_t(e) -( \wmin(\hat{M}_*, \bw) - 2 \rad_t(c_t))
	\\
	< & 4 \cdot \frac{\varepsilon}{4}
	\\
	= & \varepsilon 
	\end{align*}
	and thus 
	\begin{align*}
	\wmin(\tilde{M}_t, \bar{\bw}_t )-\wmin(\hat{M}_*, \underline{\bw}_t) \leq \varepsilon ,
	\end{align*}
	which contradicts the stop condition.
\end{proof}

\section{Lower Bounds for the Fixed-Confidence Setting}
\label{apx:lower_bound}

In this section, we present the proof of lower bound for the exact fixed-confidence CPE-B problem. Then, we also provide a lower bound for the PAC fixed-confidence CPE-B problem and give its proof.

First, we prove the lower bound for the exact fixed-confidence CPE-B problem (Theorem~\ref{thm:fixed_confidence_lb}).
Notice that, the sample complexity of algorithm $\algbottleneck$ also matches the lower bound within a logarithmic factor if we replace condition (iii) 
	below with that each sub-optimal super arm only has a single base arm.

\begin{proof}
	Consider an instance $\cI$ of the fixed-confidence CPE-B problem such that: (i) the reward distribution of each base arm $e \in [n]$ is $\cN(w(e), R)$; (ii) both $M_*$ and the second best super arms are unique, and the second best super arm has no overlapped base arm with $M_*$; (iii) in each sub-optimal super arm, there is a single base arm with reward below $\wmin(M_*, \bw)$. 
	
	Fix an arbitrary $\delta$-correct algorithm $\mathbb{A}$.
	For an arbitrary base arm $e \in M_*$, we construct an instance $\cI'$ by changing its reward distribution to $\cN(w'(e), R)$ where $w'(e)=w(e)-2\Delta^\fc_e$.
	Recall that  $M_{\textup{second}}=\argmax_{M \neq M_*} \wmin(M, \bw)$. 
	For instance $\cI'$, from the definition of $\Delta^\fc_e$ (Definition~\ref{def:fc_gap}),
	\begin{align*}
	w'(e) = & w(e)-2\Delta^\fc_e 
	\\
	= & w(e) - ( w(e)-\wmin(M_{\textup{second}}, \bw) ) - \Delta^\fc_e
	\\
	= & \wmin(M_{\textup{second}}, \bw) - \Delta^\fc_e
	\\
	< & \wmin(M_{\textup{second}}, \bw)
	\end{align*}
	and $\wmin(M_*, \bw')=w'(e)<\wmin(M_{\textup{second}}, \bw)$. Thus, $M_{\textup{second}}$ becomes the optimal super arm.

	Let $T_e$ denote the number of samples drawn from base arm $e$ when algorithm $\mathbb{A}$ runs on instance $\cI$. 
	Let $d(x,y)=x\ln(x/y)+(1-x)\ln[(1-x)/(1-y)]$ denote the binary relative entropy function.
	Define $\cH$ as the event that algorithm $\mathbb{A}$ returns $M_*$.
	Since $\mathbb{A}$ is $\delta$-correct, we have $\Pr \limits_{\mathbb{A},\cI}[\cH] \geq 1-\delta$ and $\Pr \limits_{\mathbb{A},\cI'}[\cH] \leq \delta$.
	Thus, $d(\Pr \limits_{\mathbb{A},\cI}[\cH], \Pr \limits_{\mathbb{A},\cI'}[\cH]) \geq d(1-\delta,\delta)$.
	Using Lemma 1 in~\cite{kaufmann2016complexity}, we can obtain
	\begin{align*}
	\ex[T_e] \kl(\cN(w(e), R), \cN(w'(e), R)) \geq d(1-\delta,\delta),
	\end{align*}
	\wei{What is $d(1-\delta,\delta)$? Also I cannot directly make a connection between $\delta$ the high-probability measure, with the rest of the formula.
	The only way that I can build a connection is through $T_e$, and that means $T_e$ is the number of samples need to for the algorithm to be $\delta$-correct. Is it so?
	If so, it is fine. Just want to get a better understanding of the above formula.}
	\yihan{Added the definition of $d(x,y)$ and explanation before the formula.}
	Since the reward distribution of each base arm is Gaussian distribution, we have $\kl(\cN(w(e), R), \cN(w(e'), R))=\frac{1}{2 R^2}(w(e)-w'(e))^2=\frac{2}{ R^2} (\Delta^\fc_e)^2$. Since $\delta \in (0,0.1)$, $d(1-\delta,\delta) \geq 0.4 \ln(1/ \delta)$. Thus, we have
	\begin{align*}
	\frac{2}{ R^2} (\Delta^\fc_e)^2 \cdot \ex[T_e] \geq 0.4 \ln( \frac{1}{\delta} ).
	\end{align*}
	Then, 
	\begin{align*}
	\ex[T_e]  \geq 0.2 \frac{ R^2 }{(\Delta^\fc_e)^2} \ln( \frac{1}{\delta} ).
	\end{align*}
	
	For an arbitrary base arm $e \notin M_*, w(e)<\wmin(M_*, \bw)$, we can construct another instance $\cI'$ by changing its reward distribution to $\cN(w'(e), R)$ where $w'(e)=w(e)+2\Delta^\fc_e$. Let $M_e$ denote the sub-optimal super arm that contains $e$.

	For instance $\cI'$, from the definition of $\Delta^\fc_e$ (Definition~\ref{def:fc_gap}),
	\begin{align*}
	w'(e) = & w(e) + 2\Delta^\fc_e 
	\\
	= & w(e) + ( \wmin(M_*, \bw)-\wmin(M_e) ) + \Delta^\fc_e
	\\
	= & w(e) + ( \wmin(M_*, \bw)-w(e) ) + \Delta^\fc_e
	\\
	= & \wmin(M_*, \bw) + \Delta^\fc_e
	\\
	> & \wmin(M_*, \bw).
	\end{align*}
	Thus, $M_e$ become the optimal super arm. Similarly, using Lemma 1 in~\cite{kaufmann2016complexity} we can obtain
	\begin{align*}
	\frac{2}{ R^2} (\Delta^\fc_e)^2 \cdot \ex[T_e] \geq 0.4 \ln( \frac{1}{\delta} ).
	\end{align*}
	Then, 
	\begin{align*}
	\ex[T_e]  \geq 0.2 \frac{ R^2 }{(\Delta^\fc_e)^2} \ln( \frac{1}{\delta} ).
	\end{align*}
	
	Summing over all $e \in M_*$ and $e \notin M_*, w(e)<\wmin(M_*, \bw)$, we can obtain that any $\delta$-correct algorithm has sample complexity 
	$$
	\Omega \left(  \sum_{e \in M_* \cup B} \frac{R^2}{(\Delta^\fc_e)^2}  \ln \left(\frac{1}{\delta}\right)  \right).
	$$
\end{proof}

Next, we present the lower bound for the PAC fixed-confidence CPE-B problem, where we can relax condition (ii) in the proof of the exact lower bound (Theorem~\ref{thm:fixed_confidence_lb}).
To formally state our result, we first introduce 
the notion of  \emph{$(\delta,\varepsilon)$-correct algorithm} as follows.
For any confidence parameter $\delta \in (0,1)$ and accuracy parameter $\varepsilon>0$, we call an algorithm $\cA$ a $(\delta,\varepsilon)$-correct algorithm if for the fixed-confidence CPE-B in PAC learning, $\cA$ returns a super arm $M_{\textup{pac}}$ such that  $\wmin(M_{\textup{pac}}, \bw) \geq \wmin(M_*, \bw)-\varepsilon$ with probability at least $1-\delta$. 

\begin{theorem}[Fixed-confidence Lower Bound for PAC]
	\label{thm:fixed_confidence_pac_lb}
	There exists a family of instances for the fixed-confidence CPE-B problem, where for any $\delta \in (0,0.1)$,  any $(\delta,\varepsilon)$-correct algorithm has the expected sample complexity 
	$$
	\Omega \Bigg( \sum_{e \in M_* \cup B} \frac{R^2}{ \max\{(\Delta^\fc_e)^2, \varepsilon^2 \}}  \ln \left(\frac{1}{\delta}\right)  \Bigg).
	$$
\end{theorem}
\begin{proof}
	Consider the instance $\cI$ for the PAC  fixed-confidence CPE-B problem, where $\varepsilon<\wmin(M_*, \bw)-\wmin(M_\textup{second}, \bw)$ (to guarantee that $M_*$ is unique) and  (i) the reward distribution of each base arm $e \in [n]$ is $\cN(w(e), 1)$; (ii) the PAC solution $M_{\textup{pac}}$ is unique and the second best super arm has no overlapped base arm with $M_{\textup{pac}}$; (iii) in each sub-optimal super arm, there is a single base arm with reward below $\wmin(M_*, \bw)$. 
	Then, following the proof procedure of Theorem~\ref{thm:fixed_confidence_lb}, we can obtain Theorem~\ref{thm:fixed_confidence_pac_lb}.
\end{proof}

\section{CPE-B in the Fixed-Budget Setting}
\label{apx:fixed_budget}

\begin{algorithm}[t!]
	\caption{$\aroracle$} \label{alg:aroracle}
	\begin{algorithmic}[1]
		\STATE \textbf{Input:} decision class $\cM$, accepted base arm $a$,  set of the rejected base arms $R$ and weight vector $\bv$.
		\STATE Remove the base arms in $R$ from $\cM$ and obtain a new decision class $\cM_{-R}$.
		\STATE Sort the remaining base arms by descending rewards and denote them by $e_{(1)}, \dots, e_{(n-|R|)}$\;
		\FOR{$e= e_{(1)}, \dots, e_{(n-|R|)} $}
		\STATE Remove all base arms with the rewards lower than $w(e)$ from $\cM$ and obtain a new decision class $\cM_{-R, \geq w(e)}$\;
		\STATE $M_{\textup{out}} \leftarrow \existoracle(\cM_{-R, \geq w(e)}, a )$\;
		\IF{$M_{\textup{out}} \neq \perp$}
		\STATE \textbf{return} $M_{\textup{out}}$\;
		\ENDIF
		\ENDFOR
	\end{algorithmic}
\end{algorithm}

In this section, we present the implementation details of $\aroracle$ and error probability proof for algorithm $\bsar$.

\subsection{Implementation Details of $\aroracle$}
First, we discuss $\aroracle$.
Recall that $\aroracle \in \argmax_{M \in \cM(e, R)} \wmin(M, \bw)$, where $\cM(e, R)=\{ M \in \cM: e \in M, R \cap M = \varnothing \}$. If $\cM(e, R)=\varnothing$, $\aroracle=\perp$. 
Algorithm~\ref{alg:aroracle} gives the algorithm pseudo-code of $\aroracle$. As $\findewinset$, $\aroracle$ also uses the existence oracle $\existoracle$ to find a feasible super arm that contains some base arm from the given decision class if there exists, and otherwise return $\perp$.
We explain the procedure of $\aroracle$ as follows: we first remove all base arms in $R$ from the decision class $\cM$ to disable the super arms that contain the rejected base arms. Then, we enumerate the remaining base arms by descending rewards. For each enumerated base arm $e$, we remove the base arms with rewards lower than $w(e)$ and obtain a new decision class $\cM_{-R, \geq w(e)}$, and then use $\existoracle$ to find  a feasible super arm that contains the accepted base arm $a$ from  $\cM_{-R, \geq w(e)}$. Once such a feasible super arm is found, the procedure terminates and returns this super arm.
%
%
Since the enumeration of base arms is performed according to descending rewards and the computed decision class only contains base arms no worse than the enumerated one, $\aroracle$ guarantees to return an optimal super arm from $\cM(e, R)$. 

As for the computational efficiency, the time complexity of $\aroracle$ mainly depends on the step of finding a feasible super arm containing some base arm $e$. 
Fortunately, this existence problem can be solved in polynomial time for a wide family of decision classes. For example, for $s$-$t$ paths, this problem can be reduced to the well-studied $2$-vertex connectivity problem~\cite{two_vertex_connectivity2018}, which is polynomially tractable (see Section~\ref{sec:reduce_2_vertex} for the proof of reduction). 
For maximum cardinality matchings, we just need to remove $e$ and its two end vertices, and then find a feasible maximum cardinality matching in the remaining graph; and for spanning trees, we can just merge the vertices of $e$ and find a feasible spanning tree in the remaining graph.
All of the above cases can be solved efficiently.

\subsection{Proof for Algorithm $\bsar$}
Below we present the proof of error probability for algorithm $\bsar$. 
To prove Theorem~\ref{thm:bsar}, we first introduce the lowing Lemmas~\ref{lemma:bsar_mistake_event}-\ref{lemma:bsar_mistake_p_t}.

\begin{lemma} \label{lemma:bsar_mistake_event}
	For phase $t=1, \dots, n$, define events
	$$
	\cE_t=\lbr{ \forall i \in [n]\setminus(A_t \cup R_t), \    |\hat{w}_t(i)-w(i)| < \frac{ \Delta^\fb_{(n+1-t)} }{8}  } .
	$$
	and $\cE \triangleq \bigcap_{t=1}^{n} \cE_t$.
	Then, we have 
	$$
	\Pr\mbr{ \cE } \geq 1- 2 n^2  \exp \sbr{ - \frac{(T-n)  }{128 \tilde{\log}(n)  R^2 H^B } } .
	$$
\end{lemma}

\begin{proof}
	For any $t \in [n]$ and $e \in  [n]\setminus(A_t \cup R_t)$, according to the Hoeffding's inequality,
	$$
	\lbr{ |\hat{w}_t(i)-w(i)| \geq \frac{ \Delta^\fb_{(n+1-t)} }{8}  } \leq 2 \exp \sbr{ - \frac{\tilde{T} (\Delta^\fb_{n-t+1})^2}{128 R^2} } .
	$$
	From the definition of $\tilde{T}$ and $H^B$, we have
	\begin{align*}
		\Pr \mbr{ |\hat{w}_t(i)-w(i)| \geq \frac{ \Delta^\fb_{(n+1-t)} }{8}  } \leq & 2 \exp \sbr{ - \frac{\tilde{T} (\Delta^\fb_{n-t+1})^2}{128 R^2} }
		\\
		\leq & 2 \exp \sbr{ - \frac{ \frac{T-n}{\tilde{\log}(n)(n-t+1)} (\Delta^\fb_{n-t+1})^2}{128 R^2} }
		\\
		= & 2 \exp \sbr{ - \frac{(T-n)  }{128 \tilde{\log}(n)  R^2 \frac{n-t+1}{(\Delta^\fb_{n-t+1})^2}} }
		\\
		\leq & 2 \exp \sbr{ - \frac{(T-n)  }{128 \tilde{\log}(n)  R^2 H^B } }
	\end{align*}
	By a union bound over $t \in [n]$ and $e \in  [n]\setminus(A_t \cup R_t)$, we have
	\begin{align*}
		\Pr[\cE] \geq & 1-n^2 \Pr \mbr{ |\hat{w}_t(i)-w(i)| \geq \frac{ \Delta^\fb_{(n+1-t)} }{8}  }
		\\
		\geq & 1- 2 n^2 \exp \sbr{ - \frac{(T-n)  }{128 \tilde{\log}(n)  R^2 H^B } }.
	\end{align*}
\end{proof}

\begin{lemma}\label{lemma:bsar_e_in_M_t_cup_M_*}
	Fix any phase $t>0$. Assume that event $\cE_t$ occurs and algorithm $\bsar$ does not make any mistake before phase $t$, i.e., $A_t \subseteq M_*$ and $R_t \cap M_* = \varnothing$. Then, for any $e \in [n]\setminus(A_t \cup R_t)$ s.t. $\Delta^\fb_e \geq \Delta^\fb_{(n+1-t)} $, we have $e \in (M_* \cap M_t)\cup(\neg M_* \cap \neg M_t)$.
\end{lemma}
\begin{proof}
	Suppose that $e \in (M_* \cap \neg M_t)\cup(\neg M_* \cap M_t)$.
	
	Case (I). If $e \in M_*, e \notin M_t$, then $M_t$ is a sub-optimal super arm and $\Delta^\fb_{M_*,M_t} \geq  \Delta^\fb_e \geq \Delta^\fb_{(n+1-t)}$.
	Then, we have 
	\begin{align*}
		\wmin(M_*,\hat{\bw}_t)-\wmin(M_t,\hat{\bw}_t)  > & \wmin(M_*,\bw)-\frac{1}{8} \Delta^\fb_{(n+1-t)} - (\wmin(M_t,\bw)+\frac{1}{8} \Delta^\fb_{(n+1-t)} )
		\\
		= & \wmin(M_*,\bw) - \wmin(M_t,\bw) - \frac{1}{4} \Delta^\fb_{(n+1-t)}
		\\
		\geq & \Delta^\fb_e - \frac{1}{4} \Delta^\fb_{(n+1-t)}
		\\
		\geq & \frac{3}{4} \Delta^\fb_{(n+1-t)}
		\\
		> & 0,
	\end{align*}
	which contradicts the definition of $M_t$.
	
	Case (II). If $e \in M_t, e \notin M_*$, then $M_t$ is a sub-optimal super arm and $\Delta^\fb_{M_*,M_t} \geq  \Delta^\fb_e \geq \Delta^\fb_{(n+1-t)}$.
	Then, we have 
	\begin{align*}
		\wmin(M_*,\hat{\bw}_t)-\wmin(M_t,\hat{\bw}_t)  > & \wmin(M_*,\bw)-\frac{1}{8} \Delta^\fb_{(n+1-t)} - (\wmin(M_t,\bw)+\frac{1}{8} \Delta^\fb_{(n+1-t)} )
		\\
		= & \wmin(M_*,\bw) - \wmin(M_t,\bw) - \frac{1}{4} \Delta^\fb_{(n+1-t)}
		\\
		\geq & \Delta^\fb_e - \frac{1}{4} \Delta^\fb_{(n+1-t)}
		\\
		\geq & \frac{3}{4} \Delta^\fb_{(n+1-t)}
		\\
		> & 0,
	\end{align*}
	which contradicts the definition of $M_t$.
	
	Thus, the supposition does not hold and we obtain $e \in (M_* \cap M_t)\cup(\neg M_* \cap \neg M_t)$.
\end{proof}

\begin{lemma} \label{lemma:bsar_exist_e}
	Fix any phase $t>0$. Assume that event $\cE_t$ occurs and algorithm $\bsar$ does not make any mistake before phase $t$, i.e., $A_t \subseteq M_*$ and $R_t \cap M_* = \varnothing$. Then, there exists some base arm 
	$e \in [n]\setminus(A_t \cup R_t)$ s.t. $\Delta^\fb_e \geq \Delta^\fb_{(n+1-t)} $ and this base arm $e$ satisfies 
	$$
	\wmin(M_t,\hat{\bw}_t)-\wmin(\tilde{M}_{t,e},\hat{\bw}_t) > \frac{3}{4} \Delta^\fb_{(n+1-t)}.
	$$
\end{lemma}
\begin{proof}
	Since $e \in [n]\setminus(A_t \cup R_t)$ and $\Delta^\fb_e \geq \Delta^\fb_{(n+1-t)} $, according to Lemma~\ref{lemma:bsar_e_in_M_t_cup_M_*}, we have $e \in (M_* \cap M_t)\cup(\neg M_* \cap \neg M_t)$.
	
	Case (I). If $e \in (M_* \cap M_t)$, then $e \notin \tilde{M}_{t,e}$ (if $\tilde{M}_{t,e}=\perp$ then the lemma trivially holds) and $\Delta^\fb_{M_*, \tilde{M}_{t,e}} \geq \Delta^\fb_e$. We have
	\begin{align*}
		\wmin(M_*,\hat{\bw}_t)-\wmin(\tilde{M}_{t,e},\hat{\bw}_t)  > & \wmin(M_*,\bw)-\frac{1}{8} \Delta^\fb_{(n+1-t)} - (\wmin(\tilde{M}_{t,e},\bw)+\frac{1}{8} \Delta^\fb_{(n+1-t)} )
		\\
		= & \wmin(M_*,\bw)-\wmin(\tilde{M}_{t,e},\bw)-\frac{1}{4} \Delta^\fb_{(n+1-t)} 
		\\
		\geq & \Delta^\fc_e-\frac{1}{4} \Delta^\fb_{(n+1-t)} 
		\\
		\geq & \frac{3}{4} \Delta^\fb_{(n+1-t)} .
	\end{align*}
	
	Case (II). If $e \in (\neg M_* \cap \neg M_t)$, then $e \in \tilde{M}_{t,e}$ (if $\tilde{M}_{t,e}=\perp$ then the lemma trivially holds) and $\Delta^\fb_{M_*, \tilde{M}_{t,e}} \geq \Delta^\fb_e$. We have
	\begin{align*}
		\wmin(M_*,\hat{\bw}_t)-\wmin(\tilde{M}_{t,e},\hat{\bw}_t)  > & \wmin(M_*,\bw)-\frac{1}{8} \Delta^\fb_{(n+1-t)} - (\wmin(\tilde{M}_{t,e},\bw)+\frac{1}{8} \Delta^\fb_{(n+1-t)} )
		\\
		= & \wmin(M_*,\bw)-\wmin(\tilde{M}_{t,e},\bw)-\frac{1}{4} \Delta^\fb_{(n+1-t)} 
		\\
		\geq & \Delta^\fc_e-\frac{1}{4} \Delta^\fb_{(n+1-t)} 
		\\
		\geq & \frac{3}{4} \Delta^\fb_{(n+1-t)} .
	\end{align*}
	
	Combining cases (I) and (II), we obtain the lemma.
\end{proof}

\begin{lemma} \label{lemma:bsar_mistake_p_t}
	Fix any phase $t>0$. Assume that event $\cE_t$ occurs and algorithm $\bsar$ does not make any mistake before phase $t$, i.e., $A_t \subseteq M_*$ and $R_t \cap M_* = \varnothing$.
	Then, for any $p \in [n]\setminus(A_t \cup R_t)$ s.t. $p \in (M_* \cap \neg M_t)\cup(\neg M_* \cap M_t)$, we have 
	$$
	\wmin(M_t,\hat{\bw}_t)-\wmin(\tilde{M}_{t,p},\hat{\bw}_t) < \frac{1}{4} \Delta^\fb_{(n+1-t)}
	$$
\end{lemma}
\begin{proof}
	Since $p \in (M_* \cap \neg M_t)\cup(\neg M_* \cap M_t)$, then $M_t$ is a sub-optimal super arm and $\Delta^\fb_{M_t, M_*} < 0$. We have 
	\begin{align*}
		\wmin(M_t,\hat{\bw}_t)-\wmin(M_*,\hat{\bw}_t)  < &   \wmin(M_t,\bw)+\frac{1}{8} \Delta^\fb_{(n+1-t)} - (\wmin(M_*,\bw)-\frac{1}{8} \Delta^\fb_{(n+1-t)})
		\\
		= & \wmin(M_t,\bw) - \wmin(M_*,\bw) + \frac{1}{4} \Delta^\fb_{(n+1-t)}
		\\
		< & \frac{1}{4} \Delta^\fb_{(n+1-t)}.
	\end{align*}
	Since $p \in (M_* \cap \neg M_t)\cup(\neg M_* \cap M_t)$, according to the definition of $\tilde{M}_{t,p}$, we have $\wmin(\tilde{M}_{t,p},\hat{\bw}_t) \geq \wmin(M_*,\hat{\bw}_t)$. Then, we have
	\begin{align*}
		\wmin(M_t,\hat{\bw}_t)-\wmin(\tilde{M}_{t,p},\hat{\bw}_t) \leq & \wmin(M_t,\hat{\bw}_t)-\wmin(M_*,\hat{\bw}_t) 
		\\
		< & \frac{1}{4} \Delta^\fb_{(n+1-t)}.
	\end{align*}
\end{proof}

Now, we prove Theorem~\ref{thm:bsar}.
\begin{proof}
	First, we prove that the number of samples for algorithm $\bsar$ is bounded by $T$. Summing the number of samples for each phase, we have
	\begin{align*}
		\sum_{t=1}^{n} \tilde{T}_t \leq & \sum_{t=1}^{n} \sbr{ \frac{T-n}{\tilde{\log}(n)(n-t+1)} +1}
		\\
		= & \frac{T-n}{\tilde{\log}(n)} \tilde{\log}(n)+n
		\\
		= & T.
	\end{align*}
	
	Next, we prove the mistake probability. According to Lemma~\ref{lemma:bsar_mistake_event}, in order to prove Theorem~\ref{thm:bsar}, it suffices to prove that conditioning on $\cE$, algorithm $\bsar$ returns $M_*$.
	
	Assuming that $\cE$ occurs, we prove by induction. Fix a phase $t \in [n]$. Suppose that algorithm $\bsar$ does not make any mistake before phase $t$, i.e., $A_t \subseteq M_*$ and $R_t \cap M_* = \varnothing$. We show that algorithm $\bsar$ does not make any mistake in phase $t$ either.
	
	According to Lemma~\ref{lemma:bsar_exist_e}, there exists some base arm 
	$e \in [n]\setminus(A_t \cup R_t)$ s.t. $\Delta^\fb_e \geq \Delta^\fb_{(n+1-t)} $ and this base arm $e$ satisfies $ \wmin(M_t,\hat{\bw}_t)-\wmin(\tilde{M}_{t,e},\hat{\bw}_t) > \frac{3}{4} \Delta^\fb_{(n+1-t)} $.
	Suppose that algorithm $\bsar$ makes a mistake in phase $t$, i.e., $p_t \in (M_* \cap \neg M_t)\cup(\neg M_* \cap M_t)$. According to Lemma~\ref{lemma:bsar_mistake_p_t}, we have $ \wmin(M_t,\hat{\bw}_t)-\wmin(\tilde{M}_{t,p_t},\hat{\bw}_t) < \frac{1}{4} \Delta^\fb_{(n+1-t)} $.
	Then,
	\begin{align*}
		\wmin(M_t,\bar{\bw}_t)-\wmin(\tilde{M}_{t,e},\bar{\bw}_t) > & \frac{3}{4} \Delta^\fb_{(n+1-t)}
		\\
		> & \frac{1}{4} \Delta^\fb_{(n+1-t)}
		\\
		> & \wmin(M_t,\bar{\bw}_t)-\wmin(\tilde{M}_{t,p_t},\bar{\bw}_t), 
	\end{align*}
	which contradicts the selection strategy of $p_t$.
	Thus, $p_t \in (M_* \cap M_t)\cup(\neg M_* \cap \neg M_t)$, i.e., algorithm $\bsar$ does not make any mistake in phase $t$, which completes the proof.
\end{proof}

\subsection{Exponential-time Complexity of the Accept-reject Oracle used in Prior Work \cite{chen2014cpe}}
\label{apx:np_hard}
The accept-reject oracle used in prior CPE-L work~\cite{chen2014cpe}, which returns the optimal super arm with a given base arm set $A_t$ contained in it, costs exponential running time on $s$-$t$ path instances.
This is because the problem $\cP$ of finding an $s$-$t$ path which contains a given edge set is NP-hard. 
In the following, we prove the NP-hardness of problem $\cP$ by building a reduction from the Hamiltonian Path Problem~\cite{Hamiltonian_path_problem} to $\cP$.

\begin{proof}
	Given any Hamiltonian Path instance $G=(V,E)$ with start and end nodes $s, t \in V$, we need to find an $s$-$t$ path that passes through every vertex in $V$ once ($s$-$t$ Hamiltonian path).
	We construct a new graph $G'$ as follows: for each vertex $u \in V \setminus \{s, t\}$, we split $u$ into two vertices $u_1, u_2$ and add an ``internal'' edge $(u_1, u_2)$. 
	For each edge $(u,v) \in E$ such that $u,v \in V \setminus \{s, t\}$, we change the original $(u,v)$ to two edges $(u_1,v_2)$ and $(u_2,v_1)$.
	For each edge $(s,u) \in E$ such that $u \in V \setminus \{s, t\}$, we change the original $(s,u)$ to edge $(s,u_1)$.
	For each edge $(u,t) \in E$ such that $u \in V \setminus \{s, t\}$, we change the original $(u,t)$ to edge $(u_2,t)$.
	
	Then, the following two statements are equivalent: (i) there exists an $s$-$t$ Hamiltonian path in $G$, and (ii) there exists an $s$-$t$ path in $G'$, which contains all internal edges $(u_1, u_2)$ for $u \in V \setminus \{s, t\}$.
	If there is a polynomial-time oracle to find an $s$-$t$ path which contains a given edge set, then this oracle can solve the given Hamiltonian path instance in polynomial time. 
	However, the Hamiltonian Path Problem is NP-hard, and thus the problem of finding an $s$-$t$ path which contains a given edge set is also NP-hard.
\end{proof}

\section{Time Complexity}
\label{apx:time_complexity}
In this paper, all our algorithms run in polynomial time.
Since the running time of our algorithms mainly depends on their used offline procedures, here we present the time complexity of used offline procedures on three common decision classes, e.g., $s$-$t$ paths, maximum cardinality matchings and spanning trees.
Let $E$ and $V$ denote the numbers of edges and vertices in the graph, respectively.

\begin{table}[h]
\centering
\begin{tabular}{ccccc}
	\toprule  
	& $\oracle$ & $\existoracle$ & $\findewinset$ & $\aroracle$ \\
	\midrule  
	$s$-$t$ paths& $O(E)$ & $O(E+V)$ & $O(E^2(E+V))$ & $O(E(E+V))$\\
	matchings& $O(V \sqrt{V E})$ & $O(E \sqrt{V})$ & $O(E^3 \sqrt{V})$ & $O(E^2 \sqrt{V})$\\
	spanning trees& $O(E)$ & $O(E)$ & $O(E^3)$ & $O(E^2)$\\
	\bottomrule 
\end{tabular}
\vspace*{0.8em}
\caption{Time complexity of the offline procedures used in our algorithms.}
\end{table}
 
\subsection{Reduction of $\existoracle$ to $2$-vertex Connectivity} \label{sec:reduce_2_vertex}
In this subsection, we show how to reduce \emph{the problem of finding a $s$-$t$ path that contains a given edge $(u,v)$ ($\existoracle$)} to \emph{the $2$-vertex connectivity problem~\cite{two_vertex_connectivity2018}} as follows.

First, we formally define these two problems.

\textbf{Problem A ($\existoracle$).} Given a graph $G$ with vertices $s, t, u, v$, check if there exists a $s$-$t$ simple path that contains $(u,v)$, and output such a path if it exists.

\textbf{Problem B ($2$-vertex connectivity).} Given a graph $G$ with vertices $w, z$, check if there exist two vertex-disjoint paths connecting $w$ and $z$, and output such two vertex-disjoint paths if they exist.

Now we present the proof of reduction from Problem A to Problem B.

\begin{proof}
The reduction starts from a given instance of Problem A. Given a graph $G$ with vertices $s, t, u, v$, we divide edge $(u,v)$ into two edges $(u,w), (w,v)$ with an added virtual vertex $w$. Similarly, we also divide edge $(s,t)$ into two edges $(s,z),(z,t)$ with an added virtual vertex $z$. 
Now, we show that finding a $s$-$t$ simple path that contains $(u,v)$ is equivalent to finding two vertex-disjoint paths connecting $w$ and $z$.

(i) If we have a $s$-$t$ simple path $p$ that contains $(u,v)$, then $p$ has two subpaths $p_1, p_2$ connecting $s,w$ and $w,t$, respectively, where $p_1, p_2$ do not have overlapped vertices. We concatenate $p_1$ and $(s,z)$, and concatenate $p_2$ and $(t,z)$. Then, we obtain two vertex-disjoint paths connecting $w$ and $z$.

(ii) If we have two vertex-disjoint paths connecting $w$ and $z$, then using the facts that $w$ is only connected to vertices $u,v$ and $z$ is only connected to vertices $s,t$, we can obtain two vertex-disjoint paths $q_1, q_2$ connecting $s,u$ and $t,v$, respectively (or connecting $s,v$ and $t,u$, respectively). We concatenate $q_1, q_2$ and $(u,v)$. Then, we obtain a $s$-$t$ simple path that contains $(u,v)$.

Therefore, we showed that for any given instance of Problem A, we can transform it to an instance of Problem B (by the above construction), and then use an existing oracle of Problem B \cite{two_vertex_connectivity2018} to solve the given instance of Problem A. 
\end{proof}
\yihan{Added the proof of reduction.}

\section{Extension to General Reward Functions}
\label{apx:cpe_g}

\subsection{Problem Setting}
In this section, we study the extension of CPE-B to general reward functions (CPE-G) in the fixed-confidence setting.
Let $f(M, \bw)$ denote the expected reward function of super arm $M$, which only depends on $\{w(e)\}_{e \in M}$.
Different from previous CPE works~\cite{chen2014cpe,ChenLJ_Nearly_Optimal_Sampling17,huang_CPE_CS2018} which either study the linear reward function or impose strong assumptions (continuous and separable~\cite{huang_CPE_CS2018}) on nonlinear reward functions, we only make the following two standard assumptions:

\begin{assumption}[Monotonicity] \label{assumption:monotonicity_basic}
	For any $M \in \cM$ and $\boldsymbol{v}, \boldsymbol{v}' \in \R^n$ such that $\forall e \in [n],\  v'(e) \leq v(e)$, we have $f(M, \boldsymbol{v}') \leq f(M, \boldsymbol{v})$.
\end{assumption}


\begin{assumption}[Lipschitz continuity with $\infty$-norm] \label{assumption:lipschitz_basic_infinity_norm}
	For any $M \in \cM$ and $\bv, \bv' \in \R^n$, there exists a universal constant $U>0$ such that $| f(M, \bv)-f(M, \bv') | \leq U \max_{e \in M} | v(e)-v'(e) | $.
\end{assumption}

A wide family of reward functions satisfy these two mild assumptions, with the linear reward function (CPE-L)~\cite{chen2014cpe,ChenLJ_Nearly_Optimal_Sampling17},  bottleneck reward function (CPE-B) and continuous and separable reward functions (CPE-CS)~\cite{huang_CPE_CS2018} as its special cases. In addition, many other interesting problems, such as the quadratic network flow~\cite{Quadratic_network_flow_problems}, quadratic network allocation~\cite{tree_allocation_problem} and the densest subgraph~\cite{Goldberg84}, are encompassed by CPE-G.


\subsection{Algorithm for CPE-G}

\begin{algorithm}[t!]
	\caption{$\nonlucb$} \label{alg:nonlinear_lucb}
	\begin{algorithmic}[1]
		\STATE \textbf{Input:} decision class $\cM$, confidence $\delta \in (0,1)$, reward function $f$ and maximization oracle $\oracle$ for $f$.
		\STATE Initialization: play each base arm $e \!\in\! [n]$ once. Initialize empirical means $\hat{\bw}_{n+1}$ and set $T_{n+1}(e) \!\leftarrow\! 1,  \forall e \!\in\! [n]$.\\
		\FOR{$t=n+1, n+2, \dots$}
		\STATE $\rad_t(e) \leftarrow R \sqrt{2 \ln (\frac{4nt^3}{\delta})/ T_t(e)}, \  \forall e \in [n]$\;
		\STATE $\underline{w}_t(e)\leftarrow \hat{w}_t(e)-\rad_t(e),\  \forall e \in [n]$\;
		\STATE $\bar{w}_t(e)\leftarrow \hat{w}_t(e)+\rad_t(e),\  \forall e \in [n]$\;
		\STATE $M_t \leftarrow \oracle(\cM, \underline{\bw}_t)$\; \label{line:GLUCB_M_t}
		\STATE $\tilde{M}_t \leftarrow \oracle(\cM \setminus \{M_t\}, \bar{\bw}_t)$ or $\tilde{M}_t \leftarrow \oracle(\cM \setminus \superset(M_t), \bar{\bw}_t)$\;
		\label{line:GLUCB_tilde_M_t}
		\IF{ $f(M_t, \underline{\bw}_t) \geq f(\tilde{M}_t, \bar{\bw}_t)$ }  \label{line:general_stop_condition}
		\STATE \textbf{return} $M_t$\;
		\ENDIF
		\STATE $p_t \leftarrow \argmax_{M_t \cup \tilde{M}_t} \rad_t(e)$\; \label{general_select_p_t}
		\STATE Play base arm $p_t$ and observe the reward\;
		\STATE Update empirical means $\hat{\bw}_{t+1}$\;
		\STATE Update the number of pulls: $T_{t+1}(p_t) \leftarrow T_{t}(p_t)+1$ and $T_{t+1}(e) \leftarrow T_{t}(e)$ for all $e \neq p_t$.
		\ENDFOR
	\end{algorithmic}
\end{algorithm}

For CPE-G, we propose a novel algorithm $\nonlucb$ as in 
Algorithm~\ref{alg:nonlinear_lucb}.
We allow $\nonlucb$ to access an efficient maximization oracle $\oracle$ for reward function $f$ to find an optimal super arm from the given decision class and weight vector. 
Formally, $\oracle(\cF, \bv) \in \argmax_{M \in \cF} f(M, \bv)$.
We describe the procedure of $\nonlucb$ as follows: at each timestep, we compute the lower and upper confidence bounds of the base arm rewards, and use the maximization oracle $\oracle$ to find the super arm $M_t$ with the maximum pessimistic reward from $\cM$ and super arm $\tilde{M}_t$ with the maximum optimistic reward from $\cM \setminus \{M_t\}$. Then, we play the base arm $p_t$ with the maximum confidence radius from $M_t \cup \tilde{M}_t$. When we see that the pessimistic reward of $M_t$ is higher than the optimistic reward of $\tilde{M}_t$, which implies that  $M_t$ has a higher reward than any other super arm with high confidence, we stop the algorithm and return $M_t$. 

Different from CPE-B or previous CPE works~\cite{chen2014cpe,CPE_DB_2020} which only focus on the bottleneck base arms or those in symmetric difference, in CPE-G we select the base arm among the \emph{entire} union set of two critical super arms, since for these two super arms, any base arm in their union can affect the reward difference and should be estimated.

\subsection{Implementation of the Oracle in $\nonlucb$}
Now we discuss the implementation of $\oracle$ in $\nonlucb$.
For $\cF=\cM$, we simply calculate an optimal super arm from $\cM$ with respect to $\bv$.
Such a maximization oracle can be implemented efficiently for a rich class of decision classes and nonlinear reward functions, such as the densest subgraph~\cite{dense_subgraphs}, quadratic network flow problems~\cite{Quadratic_network_flow_problems} and quadratic network allocation  problems~\cite{tree_allocation_problem}.
For $\cF=\cM \setminus \{M_t\}$, it is more challenging to implement in polynomial time. 
We first discuss three common cases, where the step $\tilde{M}_t \leftarrow \oracle(\cM \setminus \{M_t\}, \bar{\bw}_t)$ (labeled as (a)) can be replaced with a more practical statement $\tilde{M}_t \leftarrow \oracle(\cM \setminus \superset(M_t), \bar{\bw}_t)$ (labeled as (b)). Then, we can implement it as follows: repeatedly remove each base arm in $M_t$ and compute the best super arm from the remaining decision class, and then return the one with the maximum reward.

Below we formally state the three cases:

\emph{Case (i). Any two super arms $M,M' \in \cM$ satisfies $M\setminus M' \neq \varnothing$.}

In this case, $\superset(M_t)=M_t$ and the statements (a),(b) are equivalent.
Many decision classes such as top $k$, maximum cardinality matchings, spanning trees fall in this case.

\emph{Case (ii). $f$ is set monotonically decreasing.} 

As CPE-B, $f(M_t,\bw) \geq f(M',\bw)$ for all $M' \in \superset(M_t)$, and we only need to compare $M_t$ against super arms in $\cM \setminus \superset(M_t)$.

\emph{Case (iii). $f$ is strictly set monotonically increasing.} 

According to Line~\ref{line:GLUCB_M_t} of Algorithm~\ref{alg:nonlinear_lucb}, we have that $\superset(M_t)=M_t$ and the statements (a),(b) are equivalent.
Linear (CPE-L), quadratic, and continuous and separable (CPE-CS) reward functions satisfy this property when the expected rewards of base arms are non-negative.


If neither of the above cases holds, algorithm $\nonlucb$ executes $\tilde{M}_t \leftarrow \oracle(\cM \setminus \{M_t\}, \bar{\bw}_t)$ by disabling $M_t$ in some way and finding the best super arm from the remaining decision class with the basic maximization oracle.
For the densest subgraph problem, for example, we can construct $\oracle(\cM \setminus \{M_t\}, \bar{\bw}_t)$ efficiently by the following procedure.
Given $M_t \subseteq E$, we consider the corresponding a set of vertices $S_t \subseteq V$.
First, for each vertex $i \in S_t$, we remove $i \in S_t$ from $V$, and obtain the best solution $S^*_i$ in the remaining graph by using any exact algorithms.
Second, for each $j \notin S_t$, we force $\{j\} \cup S$ to be included, and obtain the best solution $S^*_j$. Then we output the best solution among them. Note that the second step can be efficiently done by an exact flow-based algorithm with a min-cut procedure~\cite{Goldberg84}.

\subsection{Sample Complexity of $\nonlucb$}
Now we show the sample complexity of $\nonlucb$ for CPE-G.
For any $e \notin M_*$,
let $\Delta^\fg_e=f(M_*, \bw)-\max \limits_{M \in \cM: e\in M} f(M, \bw)$, and for any $e \in M_*$, let $\Delta^\fg_e=f(M_*, \bw)-\max_{M \neq M_*} f(M, \bw) =\Delta_{\min}$. We formally state the sample complexity of $\nonlucb$ in Theorem~\ref{thm:general_cpe_ub}.

\begin{theorem} \label{thm:general_cpe_ub}
	With probability at least $1-\delta$, the $\nonlucb$ algorithm for CPE-G will return the optimal super arm with sample complexity 
	$$
	O \sbr{ \sum_{e \in [n]} \frac{ R^2 U^2 }{(\Delta^G_e)^2} \ln \sbr{  \sum_{e \in [n]} \frac{ R^2 U^2 n  }{(\Delta^G_e)^2 \delta} }  } .
	$$
\end{theorem}

Compared to the uniform sampling algorithm 
\OnlyInFull{(presented in Appendix~\ref{apx:uniform_algorithm})}
\OnlyInShort{(presented in the supplementary material)}
which has the $O (  \frac{ R^2 U^2 n }{\Delta_{\min}^2} \ln (  \frac{ R^2 U^2 n  }{\Delta_{\min}^2 \delta} )  )$ sample complexity, $\nonlucb$ achieves a much tighter result owing to the adaptive sample strategy, which validates its effectiveness.
Moreover, to our best knowledge, $\nonlucb$ is the first algorithm with  non-trivial sample complexity for CPE with general reward functions, 
which encompass a rich class of nonlinear combinatorial problems, such as the densest subgraph problem~\cite{Goldberg84}, quadratic network flow problem~\cite{Quadratic_network_flow_problems} and quadratic network allocation  problem~\cite{tree_allocation_problem}.

To prove the sample complexity of algorithm $\nonlucb$ (Theorem~\ref{thm:general_cpe_ub}), we first introduce the following Lemmas~\ref{lemma:nonlucb_correctness},\ref{lemma:nonlucb_pull_radius}.

\begin{lemma}[Correctness of $\nonlucb$]
	\label{lemma:nonlucb_correctness}
	Assume that event $\xi$ occurs. Then, if algorithm $\nonlucb$  (Algorithm~\ref{alg:nonlinear_lucb}) terminates at round $t$, we have $M_t=M_*$.
\end{lemma}
\begin{proof}
	According to the stop condition (Line \ref{line:general_stop_condition} of Algorithm~\ref{alg:nonlinear_lucb}), when algorithm $\nonlucb$  (Algorithm~\ref{alg:nonlinear_lucb}) terminates at round $t$, we have that for any $M \neq M_t$,
	$$
	f(M_t, \bw) \geq f(M_t, \underline{\bw}_t) \geq f(M, \bar{\bw}_t) \geq f(M, \bw) .
	$$
	Thus, we have $M_t=M_*$.
\end{proof}

\begin{lemma} \label{lemma:nonlucb_pull_radius}
	Assume that event $\xi$ occurs. For any $e \in [n]$, if $\rad_t(e)<\frac{\Delta^G_e}{4U}$, then, base arm $e$ will not be pulled at round $t$, i.e., $p_t \neq e$.
\end{lemma}
\begin{proof}
	(i) Suppose that for some $e \notin M_*$, $\rad_t(e)< \frac{\Delta^G_e}{4U} = \frac{1}{4U} (f(M_*, \boldsymbol{w}) - \max_{M \in \cM: e \in M} f(M, \boldsymbol{w}))$ and $p_t = e$. 
	According to the selection strategy of $p_t$, we have that for any $e' \in M_t \cup \tilde{M}_t$, $\rad_t(e')\leq \rad_t(e)< \frac{\Delta^G_e}{4U} $.
	
	First, we can prove that $M_t \neq M_*$ and $\tilde{M}_t \neq M_*$. Otherwise, one of $M_t, \tilde{M}_t$ is $M_*$ and the other is a sub-optimal super arm containing $e$, which is denoted by $M'$. Then,
	\begin{align*}
	f(M_*, \underline{\bw}_t) - f(M', \bar{\bw}_t) \geq & (f(M_*, \bw) - 2 U \max_{i \in M_*} \rad_i) - (f(M', \bw) + 2 U \max_{j \in M'} \rad_j) 
	\\
	> & \Delta^G_{M_*,M'} -   \frac{\Delta^G_e}{2} -  \frac{\Delta^G_e}{2}
	\\
	= & 0,
	\end{align*}
	which gives a contradiction.
	
	Then, if $e \in \tilde{M}_t$, we have
	\begin{align*}
	f(\tilde{M}_t, \bar{\bw}_t) \leq & f(\tilde{M}_t, \bw) + 2U \max_{i \in \tilde{M}_t} \rad_i
	\\
	< & f(\tilde{M}_t, \bw) + \frac{\Delta^G_e}{2} 
	\\
	< & f(M_*, \boldsymbol{w})
	\\
	\leq & f(M_*, \bar{\bw}_t),
	\end{align*}
	which contradicts the definition of $\tilde{M}_t$.
	
	If $e \in M_t$, we have
	\begin{align*}
	f(\tilde{M}_t, \bar{\bw}_t) - f(\tilde{M}_t, \underline{\bw}_t) \geq & f(M_*, \bar{\bw}_t) - f(M_t, \underline{\bw}_t)
	\\
	\geq & f(M_*, \bw) - f(M_t, \bw)
	\\
	= & \Delta^G_{M_*,M_t}.
	\end{align*}
	On the other hand, we have
	\begin{align*}
	f(\tilde{M}_t, \bar{\bw}_t) - f(\tilde{M}_t, \underline{\bw}_t) \leq & (f(\tilde{M}_t, \hat{\bw})+ U\max_{i \in \tilde{M}_t} \rad_i) - (f(\tilde{M}_t, \hat{\bw})- U\max_{i \in \tilde{M}_t} \rad_i)
	\\
	= & 2 U\max_{i \in \tilde{M}_t} \rad_i
	\\
	< & \frac{\Delta^G_e}{2}
	\\
	\leq & \frac{\Delta^G_{M_*,M_t}}{2}
	\\
	< & \Delta^G_{M_*,M_t},
	\end{align*}
	which gives a contradiction.
	
	(ii) Suppose that for some $e \in M_*$, $\rad_t(e)<\frac{\Delta^G_e}{4U} =  \frac{\Delta_{\min}}{4U}$ and $p_t = e$. 
	According to the selection strategy of $p_t$, we have that for any $e' \in M_t \cup \tilde{M}_t$, $\rad_t(e')\leq \rad_t(e)< \frac{\Delta_{\min}}{4U}$.
	
	First, we can prove that $M_t \neq M_*$ and $\tilde{M}_t \neq M_*$. Otherwise, one of $M_t, \tilde{M}_t$ is $M_*$ and the other is a sub-optimal super arm, which is denoted by $M'$. Then,
	\begin{align*}
	f(M_*, \underline{\bw}_t) - f(M', \bar{\bw}_t) \geq & (f(M_*, \bw) - 2U \max_{i \in M_*} \rad_i) - (f(M', \bw) + 2U \max_{j \in M'} \rad_j) \\
	\\
	> & \Delta^G_{M_*,M'} -   \frac{\Delta_{\min}}{2} -  \frac{\Delta_{\min}}{2}
	\\
	= & 0,
	\end{align*}
	which gives a contradiction.
	
	Thus, both $M_t$ and $\tilde{M}_t$ are sub-optimal super arms.
	
	However, on the other hand, we have
	\begin{align*}
	f(\tilde{M}_t, \bar{\bw}_t) \leq & f(\tilde{M}_t, \bw) + 2U \max_{i \in \tilde{M}_t} \rad_i
	\\
	< & f(\tilde{M}_t, \bw) + \frac{\Delta_{\min}}{2} 
	\\
	< & f(M_*, \boldsymbol{w})
	\\
	\leq & f(M_*, \bar{\bw}_t),
	\end{align*}
	which contradicts the definition of $\tilde{M}_t$.
\end{proof}

Now, we prove Theorem~\ref{thm:general_cpe_ub}.
\begin{proof}
	For any $e \in [n]$, let $T(e)$ denote the number of samples for base arm $e$, and $t_e$ denote the last timestep at which $e$ is pulled. Then, we have $T_{t_e}=T(e)-1$. Let $T$ denote the total number of samples.
	According to Lemma~\ref{lemma:nonlucb_pull_radius}, we have
	\begin{align*}
	R \sqrt{ \frac{2 \ln (\frac{4 n t_e^3}{\delta})}{T(e)-1} } \geq \frac{\Delta^G_e}{4U} 
	\end{align*}
	Thus, we obtain
	\begin{align*}
	T(e) \leq \frac{32 R^2 U^2}{(\Delta^G_e)^2}\ln \sbr{\frac{4 n t_e^3}{\delta}} +1 \leq \frac{32 R^2 U^2 }{(\Delta^G_e)^2}\ln \sbr{\frac{4 n T^3}{\delta}} +1
	\end{align*}
	Summing over $e \in [n]$, we have
	\begin{align*}
	T \leq \sum_{e \in [n]} \frac{32 R^2 U^2 }{(\Delta^G_e)^2} \ln \sbr{\frac{4 n T^3 }{\delta}   } + n \leq \sum_{e \in [n]} \frac{96 R^2 U^2 }{(\Delta^G_e)^2} \ln \sbr{\frac{2 n T }{\delta}   } + n,
	\end{align*}
	where $\sum_{e \in [n]} \frac{ R^2 U^2 }{(\Delta^G_e)^2} \geq n$.
	Then, applying Lemma~\ref{lemma:technical_tool}, we have
	\begin{align*}
	T \leq & \sum_{e \in [n]} \frac{576 R^2 U^2 }{(\Delta^G_e)^2} \ln \sbr{\frac{2 n^2  }{\delta}  \sum_{e \in [n]} \frac{96 R^2 U^2 }{(\Delta^G_e)^2} } + n
	\\
	= & O \sbr{ \sum_{e \in [n]} \frac{ R^2 U^2 }{(\Delta^G_e)^2} \ln \sbr{  \sum_{e \in [n]} \frac{ R^2 U^2 n^2  }{(\Delta^G_e)^2 \delta} } + n }
	\\
	= & O \sbr{ \sum_{e \in [n]} \frac{ R^2 U^2 }{(\Delta^G_e)^2} \ln \sbr{  \sum_{e \in [n]} \frac{ R^2 U^2 n  }{(\Delta^G_e)^2 \delta} }  }.
	\end{align*}
\end{proof}

\section{Uniform Sampling Algorithms}
\label{apx:uniform_algorithm}

\begin{algorithm}[t!]
	\caption{$\uniformfc$} \label{alg:uniform_fc}
	\begin{algorithmic}[1]
		\STATE \textbf{Input:} decision class $\cM$, confidence $\delta \in (0,1)$, reward function $f$ and maximization oracle $\oracle$ for $f$.
		\FOR{$t=1, 2, \dots$}
		\STATE For each base arm $e \in [n]$, pull $e$ once and then update its empirical mean $\hat{w}_t(e)$ and number of samples $T_t(e)$\;
		\STATE $\rad_t \leftarrow R \sqrt{2 \ln (\frac{4nt^3}{\delta})/ t}$\;
		\STATE $\underline{w}_t(e)\leftarrow \hat{w}_t(e)-\rad_t,\  \forall e \in [n]$\;
		\STATE $\bar{w}_t(e)\leftarrow \hat{w}_t(e)+\rad_t,\  \forall e \in [n]$\;
		\STATE $M_t \leftarrow \oracle(\cM, \underline{\bw}_t)$\;
		\STATE $\tilde{M}_t \leftarrow \oracle(\cM \setminus \superset(M_t), \bar{\bw}_t)$\;
		\IF{ $f(M_t, \underline{\bw}_t) \geq f(\tilde{M}_t, \bar{\bw}_t)$ }  
		\STATE \textbf{return} $M_t$\;
		\ENDIF
		\ENDFOR
	\end{algorithmic}
\end{algorithm}

In this section, we present the uniform sampling algorithms for the fixed-confidence and fixed-budget CPE problems.  

Algorithm~\ref{alg:uniform_fc} illustrates the uniform sampling algorithm $\uniformfc$ for the fixed-confidence CPE problem.
Below we state the sample complexity of algorithm $\uniformfc$.

\begin{theorem} \label{thm:uniform_fc}
	With probability at least $1-\delta$, the $\uniformfc$ algorithm (Algorithm~\ref{alg:uniform_fc}) will return the optimal super arm with sample complexity 
	$$
	O \sbr{  \frac{ R^2 U^2 n }{\Delta_{\min}^2} \ln \sbr{   \frac{ R^2 U^2 n  }{\Delta_{\min}^2 \delta} }  } .
	$$
\end{theorem}
\begin{proof}
	Let $\Delta_{\min}=\min_{e \in [n]} \Delta^G_{e}=f(M_*, \bw)-f(M_\textup{second}, \bw) $. 
	Assume that event $\xi$ occurs. Then, if $\rad_t<\frac{\Delta_{\min}}{4U}$, algorithm $\uniformfc$ will stop.
	Otherwise, 
	\begin{align*}
		f(M_t, \underline{\bw}_t) - f(\tilde{M}_t, \bar{\bw}_t) \geq & f(M_t, \bw)-2U \max_{i \in M_t} \rad_t - (f(\tilde{M}_t, \bw)+2U \max_{i \in \tilde{M}_t} \rad_t)
		\\
		= & f(M_t, \bw)- f(\tilde{M}_t, \bw) - 4U \rad_t 
		\\
		> & \Delta_{\min} - \Delta_{\min}
		\\
		= & 0,
	\end{align*}
	which contradicts the stop condition.
	
	Let $T_n$ denote the number of rounds and $T$ denote the total number of samples. Then, we have
	\begin{align*}
	R \sqrt{ \frac{2 \ln (\frac{4 n T_n^3}{\delta})}{T_n-1} } \geq \frac{\Delta_{\min}}{4U} 
	\end{align*}
	Thus, we obtain
	\begin{align*}
	T_n \leq \frac{32 R^2 U^2 }{\Delta_{\min}^2}\ln \sbr{\frac{4 n T_n^3}{\delta}} +1 \leq \frac{96 R^2 U^2 }{\Delta_{\min}^2}\ln \sbr{\frac{2 n T_n}{\delta}} +1 .
	\end{align*}
	Then, applying Lemma~\ref{lemma:technical_tool}, we have
	\begin{align*}
	T_n \leq & \frac{576 R^2 U^2 }{\Delta_{\min}^2} \ln \sbr{\frac{2 n  }{\delta}   \frac{96 R^2 U^2 }{\Delta_{\min}^2} } + 1
	\\
	= & O \sbr{  \frac{ R^2 U^2 }{\Delta_{\min}^2} \ln \sbr{   \frac{ R^2 U^2 n  }{\Delta_{\min}^2 \delta} }  }.
	\end{align*}
	Summing over the number of samples for all the base arms, we obtain
	$$
	T= O \sbr{  \frac{ R^2 U^2 n }{\Delta_{\min}^2} \ln \sbr{   \frac{ R^2 U^2 n  }{\Delta_{\min}^2 \delta} }  } .
	$$
\end{proof}

Algorithm~\ref{alg:uniform_fb} illustrates the uniform sampling algorithm $\uniformfb$ for the fixed-budget CPE problem.
Below we state the error probability of algorithm $\uniformfb$.

\begin{algorithm}[t!]
	\caption{$\uniformfb$} \label{alg:uniform_fb}
	\begin{algorithmic}[1]
		\STATE \textbf{Input:} $\cM$, budget $T$, reward function $f$ and maximization oracle $\oracle$ for $f$.
		\STATE Pull each base arm $e \in [n]$ for $\left \lfloor T/n \right \rfloor$ times\;
		\STATE Update the empirical means $\hat{\bw}_t$
		\STATE $M_{\textup{out}} \leftarrow \oracle(\cM, \hat{\bw}_t)$\;
		\STATE \textbf{return} $M_{\textup{out}}$\;
	\end{algorithmic}
\end{algorithm}

Let $H^U=n (\Delta_{\min})^{-2}$, where $\Delta_{\min}=f(M_*, \bw)-f(M_\textup{second}, \bw) $.
\begin{theorem} \label{thm:uniform_fb}
	For any $T>n$, algorithm $\uniformfb$ uses at most $T$ samples and returns the optimal super arm with the error probability bounded by
	$$
	O\sbr{ n \exp \sbr{ - \frac{ T }{ R^2 U^2 H^U} } }.
	$$
\end{theorem}
\begin{proof}
    Since algorithm $\uniformfb$ allocates $\left \lfloor T/n \right \rfloor$ samples to each base arm $e \in [n]$, the total number of samples is at most $T$.

    Now we prove the error probablity.
	Define event 
	$$
	\cG=\lbr{ \forall i \in [n], \    |\hat{w}_t(i)-w(i)| < \frac{ \Delta_{\min} }{4U}  } .
	$$

	According to the Hoeffding's inequality,
	$$
	\lbr{ |\hat{w}_t(i)-w(i)| \geq \frac{ \Delta_{\min} }{4U}  } \leq 2 \exp \sbr{ - \frac{ T \Delta_{\min}^2 }{32 R^2 U^2 n} } .
	$$
	By a union bound over $e \in  [n]$, we have
	\begin{align*}
	\Pr[\cG] \geq & 1- 2 n \exp \sbr{ - \frac{ T \Delta_{\min}^2 }{32 R^2 U^2 n} }
	\\
	= & 1- 2 n \exp \sbr{ - \frac{ T }{32 R^2 U^2 H^U} }
	\end{align*}
	
	Below we prove that conditioning on event $\cG$, $M_{\textup{out}}=M_*$. 
	Suppose that $M_{\textup{out}} \neq M_*$,
	\begin{align*}
	f(M_{\textup{out}}, \hat{\bw}_t) - f(M_*, \hat{\bw}_t)  \leq & f(M_{\textup{out}}, \bw) + \frac{\Delta_{\min} }{4} - (f(M_*, \bw) - \frac{\Delta_{\min} }{4})
	\\
	\leq & - \Delta_{\min}  + \frac{\Delta_{\min} }{2}
	\\
	= & - \frac{\Delta_{\min} }{2}
	\\
	< & 0,
	\end{align*}
	which contradicts the selection strategy of $M_{\textup{out}}$.
	Thus, conditioning on event $\cG$, algorithm $\uniformfb$ returns $M_*$. Then, we obtain Theorem~\ref{thm:uniform_fb}.
\end{proof}

\section{Technical Tool}
In this section, we present a technical tool used in the proofs of our results.
\begin{lemma} \label{lemma:technical_tool}
	If $T \leq c_1 \ln(c_2 T)+c_3$ holds for some constants $c_1,c_2,c_3 \geq 1$ such that $\ln(c_1 c_2 c_3) \geq 1$, we have $T \leq 6 c_1 \ln(c_1 c_2 c_3) + c_3$.
\end{lemma}
\begin{proof}
	In the inequality $T \leq c_1 \ln(c_2 T)+c_3$, the LHS is linear with respect to $T$ and the RHS is logarithmic with respect to $T$. Thus, we have $T > c_1 \ln(c_2 T)+c_3$ for a big enough $T$.
	Then, to prove $T \leq T_0 \triangleq 6 c_1 \ln(c_1 c_2 c_3) + c_3$, it  suffices prove that  $T_0 > c_1 \ln(c_2 T_0)+c_3$.
	Since
	\begin{align*}
	c_1 \ln(c_2 T_0)+c_3 = & c_1 \ln(c_2 (6 c_1 \ln(c_1 c_2 c_3) + c_3) )+c_3
	\\
	= & c_1 \ln( 6 c_1 c_2 \ln(c_1 c_2 c_3) + c_2 c_3 )+c_3
	\\
	\leq & c_1 \ln( 6 c_1^2 c_2^2 c_3  + c_2 c_3 )+c_3
	\\
	\leq & c_1 \ln( 7 c_1^2 c_2^2 c_3  )+c_3
	\\
	\leq & 2 c_1 \ln( 7 c_1 c_2 c_3  )+c_3
	\\
	= & 2 c_1 \ln(  c_1 c_2 c_3  )+ 2 c_1 \ln(7) + c_3
	\\
	\leq & 2 c_1 \ln(  c_1 c_2 c_3  )+ 2 \ln(7) c_1  \ln(  c_1 c_2 c_3  ) + c_3
	\\
	= & (2 + 2 \ln(7)) c_1 \ln(  c_1 c_2 c_3  ) + c_3
	\\
	\leq & 6 c_1 \ln(  c_1 c_2 c_3  ) + c_3
	\\
	= & T_0 ,
	\end{align*}
	we obtain the lemma.
\end{proof}

}

\end{document}